\def\usearxivstyle{1}  
\ifnum\usearxivstyle=1
\newcommand{\arxivstyle}[2]{#1}
\else
\newcommand{\arxivstyle}[2]{#2}
\fi

\documentclass[10.5pt]{article}
\usepackage{fullpage}
\makeatletter
\long\def\@makecaption#1#2{
  \vskip 0.8ex
  \setbox\@tempboxa\hbox{\small {\bf #1:} #2}
  \parindent 1.5em  
  \dimen0=\hsize
  \advance\dimen0 by -3em
  \ifdim \wd\@tempboxa >\dimen0
  \hbox to \hsize{
    \parindent 0em
    \hfil 
    \parbox{\dimen0}{\def\baselinestretch{0.96}\small
      {\bf #1.} #2
    } 
    \hfil}
  \else \hbox to \hsize{\hfil \box\@tempboxa \hfil}
  \fi
}
\makeatother

\usepackage{times,enumitem,booktabs,multicol,blindtext,url,parskip}
\usepackage{amsfonts}
\usepackage{amsmath}
\usepackage{amsthm}
\usepackage{amssymb}
\usepackage{caption}
\usepackage{microtype}
\usepackage{graphbox}
\usepackage{subfigure}
\usepackage{booktabs} 
\usepackage{graphicx}
\usepackage{natbib}
\usepackage{fancyhdr}
\usepackage{color}
\usepackage{algorithm}
\usepackage{algorithmic}
\usepackage{forloop}
\usepackage{xcolor}
\usepackage{ulem}

\usepackage{placeins}

\usepackage{amsmath}
\usepackage{amssymb}
\usepackage{mathtools}
\usepackage{amsthm}
\usepackage{thmtools}
\usepackage{thm-restate}

\newtheorem{theorem}{Theorem}

\newtheorem{prop}[theorem]{Proposition}

\newcommand{\set}[1]{\left\{#1\right\}}

\newcommand{\strongaugs}{strong-augs}

\newcommand{\sourcedist}{P_S}
\newcommand{\targetdist}{P_T}
\newcommand{\unlabeldist}{P_U}

\newcommand{\inputx}{x}
\newcommand{\labelvar}{y}
\newcommand{\labelx}{y_x}
\newcommand{\inputxp}{{x'}}
\newcommand{\labelxp}{y_{x'}}
\newcommand{\domainvar}{d}
\newcommand{\domainx}{d_x}
\newcommand{\domainxp}{d_{x'}}
\newcommand{\numcls}{r}

\newcommand{\embeddim}{k}

\newcommand{\cdsize}{n}
\newcommand{\posx}{\inputx^+}
\newcommand{\pospairdist}{S_+}

\newcommand{\inputspace}{\sX}
\newcommand{\labelspace}{\sY}

\newcommand{\aug}{\mathcal{A}}

\newcommand{\augproba}{{\alpha'}}
\newcommand{\augprobb}{{\beta'}}
\newcommand{\augprobg}{{\gamma'}}
\newcommand{\augprobr}{{\rho'}}

\newcommand{\pairproba}{{\alpha}}
\newcommand{\pairprobb}{{\beta}}
\newcommand{\pairprobg}{{\gamma}}
\newcommand{\pairprobr}{{\rho}}
\newcommand{\pairnorm}{{C}}

\newcommand{\posclass}{1}
\newcommand{\negclass}{2}

\newcommand{\sbmg}{G}

\newcommand{\sbme}{E}

\newcommand{\encoder}{\phi}
\newcommand{\empencoder}{\widehat{\encoder}}
\newcommand{\empencoderMatrix}{\widehat{\Phi}}
\newcommand{\empencoderdann}{\widehat{\encoder}_\text{dann}}
\newcommand{\head}{h}
\newcommand{\dhead}{{\head}_{\text{dom}}}
\newcommand{\emphead}{\widehat{\head}}
\newcommand{\empheaddann}{\widehat{\head}_\text{dann}}
\newcommand{\clf}{f}
\newcommand{\empclf}{\widehat{\clf}}
\newcommand{\linmat}{B}
\newcommand{\emplinmat}{\widehat{\linmat}}
\newcommand{\empclferm}{\widehat{\clf}_\text{erm}}
\newcommand{\empclfdann}{\widehat{\clf}_\text{dann}}
\newcommand{\domainclf}{\dhead}

\newcommand{\lpretrain}{\sL_\text{pretrain}}

\newcommand{\lfinetune}{\sL}

\newcommand{\Lzeroone}{\sL_{0-1}}
\newcommand{\Lerm}{\sL_\text{ERM}}
\newcommand{\Ldann}{\sL_\text{DANN}}
\newcommand{\LDzeroone}{\sL^{D}_{0-1}}

\newcommand{\hdh}{\mathcal{H}\Delta \mathcal{H}}
\DeclareMathSymbol{\shortminus}{\mathbin}{AMSa}{"39}

\newcommand{\embedding}{\widehat{F}}

\newcommand{\sourcedata}{S}

\newcommand{\targetdata}{T}

\newcommand{\acrossdomain}{\alpha}
\newcommand{\acrossclass}{\beta}
\newcommand{\acrossboth}{\gamma}
\newcommand{\withinboth}{\rho}

\newcommand{\indicator}{\mathbf{1}}

\DeclareMathOperator*{\argmin}{arg\,min}
\DeclareMathOperator*{\argmax}{arg\,max}

\newlength{\widebarargwidth}
\newlength{\widebarargheight}
\newlength{\widebarargdepth}

\newcommand\sL{\ensuremath{\mathcal{L}}}

\newcommand\sR{\ensuremath{\mathcal{R}}}
\newcommand\sS{\ensuremath{\mathcal{S}}}
\newcommand\sT{\ensuremath{\mathcal{T}}}

\newcommand\sX{\ensuremath{\mathcal{X}}}
\newcommand\sY{\ensuremath{\mathcal{Y}}}

\newcommand\R{\ensuremath{\mathbb{R}}} 
\newcommand\Z{\ensuremath{\mathbb{Z}}} 

\ifthenelse{\isundefined{\definition}}{\newtheorem{definition}{Definition}}{}
\ifthenelse{\isundefined{\assumption}}{\newtheorem{assumption}{Assumption}}{}
\ifthenelse{\isundefined{\hypothesis}}{\newtheorem{hypothesis}{Hypothesis}}{}
\ifthenelse{\isundefined{\proposition}}{\newtheorem{proposition}{Proposition}}{}
\ifthenelse{\isundefined{\theorem}}{\newtheorem{theorem}{Theorem}}{}
\ifthenelse{\isundefined{\lemma}}{\newtheorem{lemma}{Lemma}}{}
\ifthenelse{\isundefined{\corollary}}{\newtheorem{corollary}{Corollary}}{}
\ifthenelse{\isundefined{\alg}}{\newtheorem{alg}{Algorithm}}{}
\ifthenelse{\isundefined{\example}}{\newtheorem{example}{Example}}{}
\newcommand{\E}{\ensuremath{\mathbb{E}}} 

\def\shownotes{0}  
\ifnum\shownotes=1
\newcommand{\authnote}[2]{[#1: #2]}
\else
\newcommand{\authnote}[2]{}
\fi
\newcommand{\pl}[1]{{\color{red}\authnote{PL}{#1}}}

\newcommand{\ak}[1]{{\color{blue}\authnote{AK}{#1}}}

\DeclareMathOperator{\Exp}{\mathrm{\mathbb{E}}}

\DeclareMathOperator{\Real}{\mathbb{R}}

\usepackage{hyperref}
\hypersetup{
    colorlinks=true,
    allcolors=blue
}

\begin{document}


\renewcommand*{\thefootnote}{\fnsymbol{footnote}}
\begin{center}
  {\LARGE Connect, Not Collapse: Explaining Contrastive Learning for \\
  \vspace{0.1cm}
   Unsupervised Domain Adaptation} \\
  \vspace{.4cm}
  {\large Kendrick Shen\footnotemark[\value{footnote}]\footnote{Equal contribution.} ~~~~ Robbie Jones\footnotemark[\value{footnote}] ~~~~ Ananya Kumar\footnotemark[\value{footnote}] ~~~~ Sang Michael Xie\footnotemark[\value{footnote}] \\
  \vspace{0.1cm}
  Jeff Z. HaoChen ~~~~ Tengyu Ma ~~~~ Percy Liang} \\
  \vspace{.4cm}
  {\large Stanford University} \\
  \vspace{.05cm}
  Department of Computer Science \\
  \vspace{.4cm}
  \texttt{\{kshen6,\,rmjones,\,ananya,\,xie,\,jhaochen,\,tengyuma,\,pliang\}@cs.stanford.edu}
  \vspace{.2cm}
\end{center}

\begin{abstract}
We consider unsupervised domain adaptation (UDA), where labeled data from a source domain (e.g., photos) and unlabeled data from a target domain (e.g., sketches) are used to learn a classifier for the target domain.
Conventional UDA methods (e.g., domain adversarial training) learn domain-invariant features to generalize from the source domain to the target domain.
In this paper, we show that contrastive pre-training, which learns features on unlabeled source and target data and then fine-tunes on labeled source data, is competitive with strong UDA methods.
However, we find that contrastive pre-training does not learn domain-invariant features, diverging from conventional UDA intuitions.
We show theoretically that contrastive pre-training can learn features that vary subtantially across domains but still generalize to the target domain, by disentangling domain and class information.
We empirically validate our theory on benchmark vision datasets.
\end{abstract}

\section{Introduction}
Machine learning models can perform poorly when the train and test data are drawn from different distributions, which is especially troublesome for performance-critical applications such as image recognition for self-driving cars~\citep{yu2020bdd100k,sun2020scalability} or medical image diagnosis~\citep{albadawy2018tumor,dai2018dark}.
In this work, we study the unsupervised domain adaptation (UDA) setting where we have access to labeled data from a source domain and unlabeled data from a target domain, and the goal is to get high accuracy on the target domain.

Conventional algorithms for UDA aim to learn domain-invariant features~\citep{tzeng2014confusion,ganin2016domain,tzeng2017domain,shu2018dirtt,sun2019unsupervised}---intuitively, if the distributions over features for the source and target domains are indistinguishable and the accuracy is high on the source, then the accuracy should be high on the target as well%
.
This is typically intuitively justified by theoretical notions such as $\hdh$-divergence, which measures the distinguishability of source and target feature spaces~\citep{ben2010theory}.
However,~\citet{zhao2019zhao} show that domain invariance is not sufficient for target generalization, and thus some recent works have begun to develop principled algorithms for domain adaptation~\citep{kumar2020gradual,wei2021expansion,cai2021subpopulation}.


In this paper, we find that a surprisingly simple and effective method for UDA is out-of-the-box contrastive pre-training on source and target unlabeled data, followed by fine-tuning on source labeled data.
In our experiments, contrastive pre-training obtains comparable or better results to strong UDA methods based on domain adversarial neural networks~\citep{ganin2016domain,shu2018dirtt} and self-training~\citep{prabhu2021sentry} on visual adaptation benchmarks including DomainNet, BREEDS Living-17, BREEDS Entity-30, and STL-10$\to$CIFAR-10 (results in Table~\ref{table:main-empirical}).

However, we show that contrastive pre-training diverges from conventional UDA intuitions and learns features that are \textit{easily separable} between domains; for example in the learned feature space in DomainNet, we can predict the domain of an image with only 8\% error, which is much lower than in the DANN feature space (14\%)---see Table~\ref{table:connectivity} in Section~\ref{sec:connectivity}.
In fact, in the contrastive pre-trained feature space for DomainNet, it is as easy to distinguish betweeen two domains as it is to distinguish between two classes.

How does contrastive pre-training learn features that generalize to the target domain without domain invariance?
To theoretically analyze contrastive pre-training, we consider a graph (see Figure~\ref{fig:1}) where input examples are nodes and the edge weight between two inputs is the probability that they are generated by sampling two augmentations of the same original input.
By considering the total edge weight between two class-domain 
pairs (e.g., all edges connecting a clock photo to a butterfly sketch), we can define \textit{class-domain connectivities}. These connectivities describe how much overlap to expect amongst the randomly augmented inputs from two class-domain pairs.
Our key assumption intuitively states that augmentations of clock photos should be more similar to clock sketches and butterfly photos than to butterfly sketches. 
Concretely, we assume that that both (1) the probability that augmentation changes the domain only, keeping the class constant, and (2) the probability that augmentation changes the class only, keeping the domain constant, should be higher than (3) the probability that augmentation changes both domain and class.
Under a stochastic block model (a standard randomized model for the edges of the graph), we prove that contrastive pre-training achieves good target accuracy without domain invariance---by learning a ``disentangled'' feature space that is simultaneously predictive of domain and class (Theorem~\ref{thm:sbm}).
Importantly, in our analysis the source and target domains can be very far apart, and the features can be easily distinguished (Proposition~\ref{proposition:domain_separation}).

We empirically validate our theory on benchmark UDA datasets.
We first heuristically estimate the connectivities between different class-domain pairs (e.g., clock photos and butterfly sketches) on Living-17~\citep{santurkar2020breeds} and DomainNet (12 domain pairs)~\citep{peng2019moment}, showing that they satisfy our theoretical condition for contrastive pre-training to learn transferable features.
Next, we find that the target accuracy of contrastive pre-training on DomainNet can be well-explained ($R^2 = 0.78$) using the terms that appear in our theory: the ratios of the estimated connectivities.
Furthermore, pre-training on a selected subset of the data to adversely change the connectivities consistently underperforms pre-training on randomly selected subsets of the same size (resulting in 4\%-16\% drops in target accuracy).
Finally, we give evidence that contrastive pre-training learns features that contain class and domain information along approximately different dimensions, allowing the features to be transferable across domains without requiring domain invariance.
Our results provide a first step for theoretically understanding contrastive pre-training for UDA, which can be used to improve data augmentation and selection strategies.

\begin{figure}[t]
	\centering
    \includegraphics[width=\arxivstyle{0.6\textwidth}{\linewidth}]{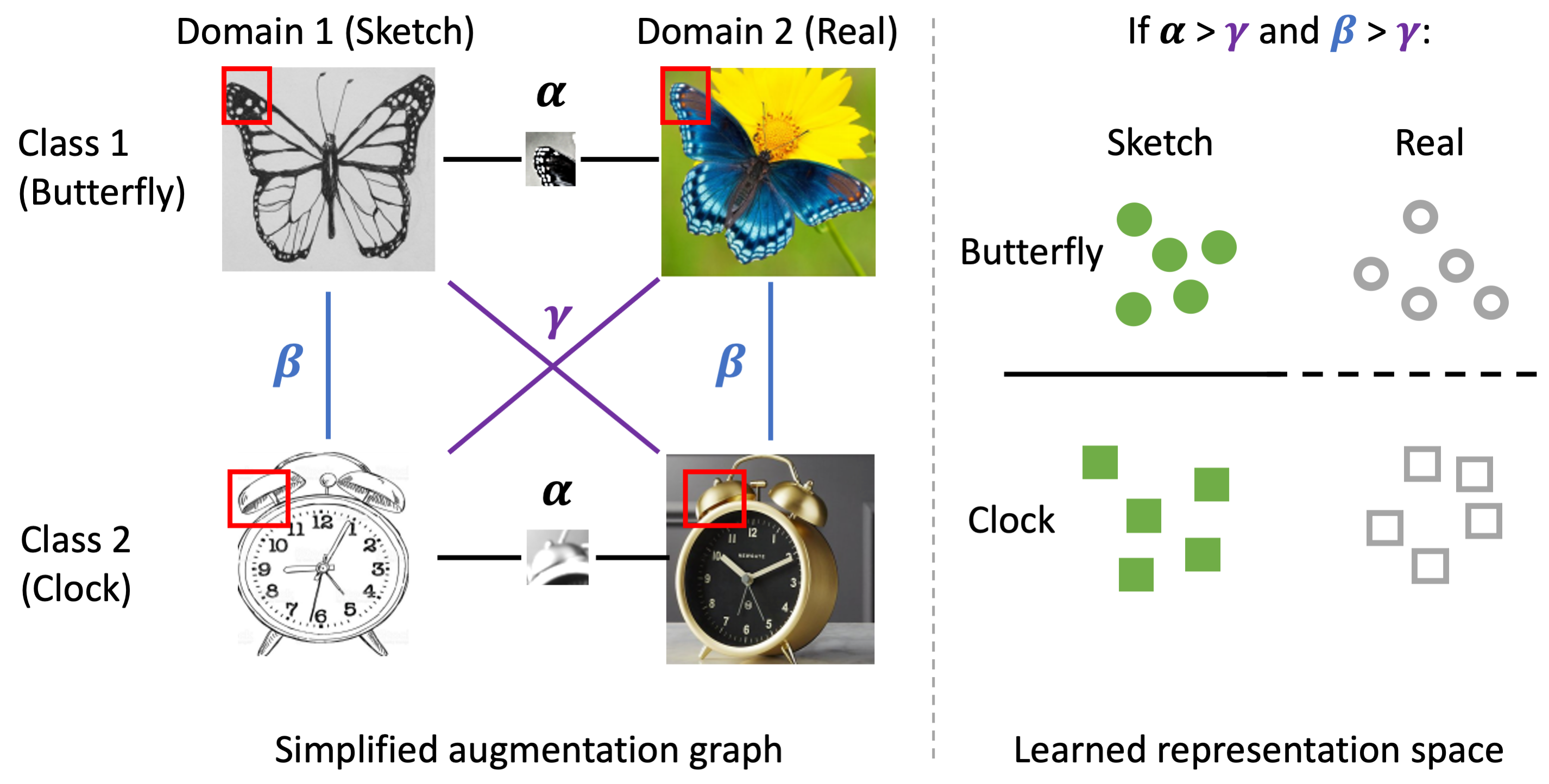}
        \captionof{figure}{\textbf{(Left)} Edges between class-domain pairs indicate how ``connected'' they are by augmentations (e.g., cropping, colorization). For example, $\acrossdomain$ denotes the probability that augmentations connect examples of the same class across different domains. Small crop sizes (e.g., the red rectangles) can increase the connectivity between disparate domains. 
        \textbf{(Right)} When $\acrossdomain$ and $\acrossclass$ (which measure the connectivity between images of the same class or same domain respectively) are larger than $\acrossboth$ (different domain and different class), fine-tuning with labeled source data (sketch: green, filled) in the feature space achieves high accuracy on the target (real: gray, hollow).
        }
	\label{fig:1}
\end{figure}

%

\section{Related work}

\paragraph{Conventional domain adaptation methods bring source and target domains together.}
\citet{ben2010theory} prove generalization bounds for domain adaptation that rely on small $\hdh$-divergence between the source and target (i.e., distributions of $\encoder(\inputx)$ for both domains are similar and therefore difficult to distinguish).
Many methods such as domain adversarial training~\citep{tzeng2014confusion,ganin2016domain,tzeng2017domain} and moment matching~\citep{peng2019moment} have objectives that aim to minimize $\hdh$-divergence between the source and target features.
More recent works propose alternate measures of domain divergence to minimize~\citep{li2021mdd}, propose dynamically weighting domain divergence and classification losses~\citep{xiao2021dynamic}, and improve on prior work that uses optimal transport to bring the domains together~\citep{li2020transport}.

\paragraph{Contrastive learning for self-supervised learning.}
Contrastive learning methods~\citep{chen2020simclr,he2020moco,caron2020swav} perform self-supervised learning of a feature encoder on unlabeled data.
Broadly, the contrastive learning objective aims to learn an encoder \pl{feature map? [use consistent terminology]}\ak{I think we've now changed to encoder}that maps augmentations (e.g., crops) of the same input to similar features and augmentations of random inputs to different features.
 \pl{'random inputs' not precise}\ak{it's randomly sampled inputs from the entire dataset, but seems a bit too technical to say this in related works}
The pre-trained encoder is then fine-tuned on a downstream dataset like ImageNet~\citep{russakovsky2015imagenet} where the downstream train and test examples come from the \textit{same} distribution.
Several theoretical works provide guarantees on the performance of contrastive learning in this setting of identical train and test distributions~\citep{arora2019contrastive, tosh2021topic, haochen2021spectral}.
In contrast, we consider contrastive pre-training for domain adaptation, where the downstream training data are from the source domain and the test data are from the target domain.


\paragraph{Domain adaptation with self-supervision.}
Prior works have explored the application of self-supervision to UDA.
However, to the best of our knowledge they have all attempted to minimize domain distance either as part of their objective~\citep{kang2019can,wang2021cdcl,thota2021cda} or for model selection~\citep{sun2019unsupervised}. Conversely, contrastive pre-training has no explicit objective involving domain distance and is still competitive with conventional UDA methods, a phenomenon that is explained by our connectivity theory.
Recent work~\citep{mishra2021surprisingly} has shown that in the semi-supervised domain adaptation setting (where a small numbers of target labels are given), a method based on self-supervised rotation prediction and input consistency regularization simultaneously achieves higher accuracy and learns features with larger domain distance than prior methods.


\paragraph{Assumptions about distribution shift.}
A classical assumption is covariate shift~\citep{shimodaira2000improving}, a setting in which the marginal distribution over $\inputspace$ can change between domains but the label conditioned on $\inputx$ is fixed.
However, if the source and target have disjoint support then covariate shift is trivially satisfied, and therefore recent works have proposed alternative definitions.
\citet{kumar2020gradual} assumes the shift is gradual in a Wasserstein metric and does not explicitly assume covariate shift.
\citet{wei2021expansion} analyzes UDA with two assumptions: 1) the source classifier has reasonable accuracy on the target, and the target satisfies an ``expansion'' condition.
\citet{cai2021subpopulation} assume that there is significant overlap (expansion) between the source and target after augmentations and analyzes self-training for UDA.
Our theoretical connectivity model is an alternative method for describing distribution shift so that contrastive pre-training adapts even under large $\hdh$-divergence. 

\section{Setup}
\subsection{Unsupervised Domain Adaptation (UDA)}

We consider a classification problem from an input space $\inputspace \subseteq \R^d$ to a label space $\labelspace = \{1,\dots, \numcls\}$.

\paragraph{Data, model, and metrics.}
Each input $\inputx \in \inputspace$ is associated with a deterministic label $\labelx \in \labelspace$ and a domain $\domainx \in \{1,2\}$ (where domain 1 is the source and domain 2 is the target)\footnote{The deterministic assumption is not strictly necessary, but greatly simplifies the setup. In general, we have a label distribution $p(\labelvar \mid \inputx)$ and domain distribution $p(\domainvar \mid \inputx)$ that are shared across all inputs, thus making the covariate shift assumption.}.
Let $\sourcedist$ and $\targetdist$ denote the source and target input distributions respectively (both over $\inputspace$).
Let $\unlabeldist = \theta \sourcedist + (1 - \theta) \targetdist$ for some $\theta \in [0, 1]$ be a mixture of the source and target domains.
This mixture distribution serves as the \textit{unlabeled} data distribution over both source and target domains.
For simplicity, we consider the population data setting in the theory.

We consider the unsupervised domain adaptation (UDA) setting: algorithms have access to labels $\labelx$ for source examples $\inputx \sim \sourcedist$ but not for target examples $\inputx \sim \targetdist$.
Algorithms are allowed to use \emph{unlabeled} data from all domains.
The goal is to learn a classifier $\clf: \inputspace \rightarrow \R^\numcls$ with low error on the target distribution $\Lzeroone(\clf) = \E_{\inputx \sim \targetdist}[\indicator[\argmax_i \clf(\inputx)_i \neq \labelx]]$.

\paragraph{Augmentations.}
Modern ML models are trained with augmentations such as crops and color distortions~\citep{krizhevsky2012imagenet, chen2020simclr}. For any input $\inputx \in \inputspace$, let $\aug(\cdot \mid \inputx)$ be the distribution of its augmentations (also over $\inputspace$).

\subsection{Methods}
\newcommand{\modelfamily}{\mathcal{F}}
\newcommand{\dattract}{d_+}
\newcommand{\drepel}{d_-}

We consider methods using a classification loss $\ell:\R^{\numcls} \times \labelspace \rightarrow \R$ (e.g., cross entropy or squared loss). We also overload the loss $\ell:\R^{2} \times \{1,2\} \rightarrow \R$ for domain classification depending on the arguments.

\textbf{Empirical risk minimization} (ERM) minimizes the loss $\ell$ over augmentations $x' \sim \aug(\cdot \mid x)$ of source examples $\inputx \sim \sourcedist$ and the original labels $\labelx$:
\begin{align}
    \label{eq:erm-object-setup}
    \Lerm(\clf) = \E_{\inputx \sim \sourcedist, \inputx' \sim \aug(\cdot \mid \inputx)}[ \ell(\clf(\inputx'), y_x) ],
\end{align}
ERM learns a classifier $\empclferm \in \argmin_{\clf} \Lerm(\clf)$.
Augmentations improve the target accuracy of ERM and our results also hold without augmentations.

\textbf{Domain adversarial neural networks} (DANN)~\citep{ganin2016domain}
optimizes the sum of two terms: a source classification loss, and a ``domain confusion'' loss that makes it difficult to predict the domain $\domainx$ from the feature $\encoder(\inputx)$.
Here, $\encoder: \inputspace \rightarrow \R^\embeddim$ is a feature encoder, $\head : \R^{\embeddim} \rightarrow \R^{\numcls}$ is a head that predicts the class, and $\dhead: \R^{\embeddim} \rightarrow \R^2$ is a head that predicts whether the input comes from the source or target.
\arxivstyle{%
\begin{align}
    \label{eq:dann-object-setup}
    \Ldann(\encoder, \head, \domainclf) = \E_{\inputx \sim \sourcedist, \inputx' \sim \aug(\cdot \mid \inputx)}[\ell(\head(\encoder(\inputx’)), \labelx)] - \lambda \; \E_{\inputx \sim \unlabeldist, \inputx' \sim \aug(\cdot \mid \inputx)}[\ell(\domainclf(\encoder(\inputx')), \domainx)],
\end{align}
}{%
\begin{align}
    \label{eq:dann-object-setup}
    \Ldann(\encoder, \head, \domainclf) = \; &\E_{\inputx \sim \sourcedist, \inputx' \sim \aug(\cdot \mid \inputx)}[\ell(\head(\encoder(\inputx')), \labelx)] \\
    - \lambda \; &\E_{\inputx \sim \unlabeldist, \inputx' \sim \aug(\cdot \mid \inputx)}[\ell(\domainclf(\encoder(\inputx')), \domainx)], \nonumber
\end{align}}
where $\lambda > 0$ is a tradeoff parameter between the two losses and $\domainclf$ is a domain classifier.
The classifier $\empclfdann = \empheaddann \circ \empencoderdann$ is learned by minimizing the loss: $\empencoderdann, \empheaddann \in \argmin_{\encoder, \head} \max_{\domainclf} \Ldann(\encoder, \head, \domainclf)$.

\textbf{Contrastive pre-training} aims to learn useful features on the unlabeled data from both source and target by training an encoder which maps data-augmented views of the same input $\inputx$ to similar features (the ``positive pairs'') and data-augmented views of random pairs of inputs to dissimilar features (the ``negative pairs'').
Intuitively, contrastive pre-training will embed augmentations of clock photos relatively closer to clock sketches than to butterfly sketches, but the embeddings may all be far apart in absolute distance.
Formally, let $\pospairdist$ be the distribution (on $\inputspace \times \inputspace$) of ``positive pairs'' (augmentations of a single input $\bar{\inputx}$), given by:
\begin{align}
    \label{eqn:pospair}
    \pospairdist(\inputx, \posx) = \E_{\bar{\inputx} \sim \unlabeldist}[ \aug(\inputx \mid \bar{\inputx}) \aug(\posx \mid \bar{\inputx}) ].
\end{align}
We apply contrastive pre-training to UDA---we pretrain an encoder $\encoder$ on source and target \textit{unlabeled} data from $\unlabeldist$, and then fine-tune a classification head $\head$ on only source labeled data from $\sourcedist$.
\begin{enumerate}
\item (Pre-train on unlabeled source and target data) We first \textit{pre-train} an encoder $\encoder: \inputspace \rightarrow \R^\embeddim$ to minimize the distance $\dattract$ between positive pairs, and maximize the distance $\drepel$ between random pairs of inputs:
\begin{align}
    \label{eq:scl-setup}
    \lpretrain(\encoder) \triangleq \,\, &\E_{(\inputx, \posx) \sim \pospairdist} \left[\dattract(\encoder(\inputx), \encoder(\posx))\right] \\
                                  - \,\, &\E_{\inputx \sim \unlabeldist, \inputx' \sim \unlabeldist}\left[\drepel(\encoder(\inputx), \encoder(\inputx'))\right], \nonumber
\end{align}
where the learned encoder is $\empencoder = \argmin_{\encoder} \lpretrain(\encoder)$.

\item (Fine-tune on labeled source data) We learn a classification head $\emphead = \argmin_{\head} \lfinetune(\head)$ on pretrained features $\empencoder(\inputx)$ and their labels $\labelx$ where the loss is
    \begin{align}
        \lfinetune(\head) \triangleq \E_{\inputx \sim \sourcedist}[\ell(\head(\empencoder(\inputx)), \labelx)].
    \end{align}
The final classifier is $\emphead \circ \empencoder$.
In our experiments, we fine-tune both the head $\head$ and encoder $\encoder$.

\end{enumerate}

\section{Intuitions and Analysis}
\label{sec:connectivity}

\begin{table}[t]
\centering
\begin{minipage}{\linewidth}
\centering
\resizebox{\arxivstyle{0.55\linewidth}{\linewidth}}{!}{
\begin{tabular}{l r r r r r}
	\toprule
    & Feature space & Across & Across & Across \\
    Dataset & learned by & class ($\acrossclass$) & domain ($\acrossdomain$) & both ($\acrossboth$) \\
	\midrule
    Living-17 & Input space & 21.43 & 36.58 & 19.55 \\
    & DANN+\strongaugs & 3.44 & 18.00 & 3.38 \\
              & SwAV & 2.06 & 12.67 & 1.26 \\
	\midrule
    DomainNet & Input space & 32.88 & 27.36 & 25.12 \\
    & CLIP & 1.44 & 4.58 & 0.54 \\
    & DANN+\strongaugs & 5.65 & 13.64 & 3.89 \\
    & SwAV & 7.03 & 7.54 & 2.46 \\
	\bottomrule
\end{tabular}
}
\end{minipage}
\caption{%
    Average error of classifying augmented images of two class-domain pairs (as a proxy for connectivity) in the input space and feature spaces learned by CLIP, DANN+\strongaugs{}, and SwAV (contrastive pre-training).
    Classes are very distinguishable in both the DANN and SwAV feature spaces, but the domains are much more difficult to distinguish in the DANN feature space than in the SwAV feature space.
}
\label{table:connectivity}
\end{table}

In this section, we explain when contrastive pre-training can learn transferable features---i.e., simply training a classifier on labeled source data leads to good predictions on the \textit{unlabeled} target domain---despite keeping the source and target features very different.
We extend the spectral contrastive learning framework~\citep{haochen2021spectral} to incorporate distribution shift---we give intuitions in Section~\ref{sec:intuitions}, a proof for the stochastic block model setting in Section~\ref{subsec:sbm}, and a setting where contrastive pre-training outperforms domain adversarial neural networks~\citep{ganin2016domain} (even when both use the same data augmentations), in Section~\ref{subsec:separation-example}.

\subsection{Theoretical setup}
\label{sec:theory-setup}

\paragraph{Spectral contrastive learning.}
\citet{haochen2021spectral} theoretically analyzed contrastive learning from the perspective of an augmentation graph in which the inputs $\inputx$ from $\inputspace$ constitute the nodes, and the edge weight $\pospairdist(\inputx, \inputxp)$ represents the probability of the inputs being selected as a positive pair (augmentations of the same image $\inputx$).
As in~\citet{haochen2021spectral}, we analyze the spectral contrastive learning objective, which achieves similar empirical results to other contrastive learning methods but is more amenable to theoretical analysis:
\begin{align}
    \label{eq:scl}
    \arxivstyle
    {
        \lpretrain(\encoder) = -2 \cdot \E_{(\inputx, \posx) \sim \pospairdist}\left[\encoder(\inputx)^\top \encoder(\posx)\right] + \E_{\inputx, \inputx' \sim \unlabeldist}\left[\left(\encoder(\inputx)^\top \encoder(\inputx')\right)^2\right].
    }
    {
        \lpretrain(\encoder) = -2 \cdot &\E_{(\inputx, \posx) \sim \pospairdist}\left[\encoder(\inputx)^\top \encoder(\posx)\right] \nonumber \\
        +&\E_{\inputx, \inputx' \sim \unlabeldist}\left[\left(\encoder(\inputx)^\top \encoder(\inputx')\right)^2\right].
    }
\end{align}
\pl{can we relate this to the $d_+$ and $d_-$ notation from earlier, or better yet, just streamline it all?}
The pre-training step learns the encoder $\empencoder: \inputspace \to \R^\embeddim$ by minimizing~\eqref{eq:scl}.

\paragraph{Linear classification.}
After pre-training, we train a linear probe on source domain features
with parameters $\linmat \in \R^{\numcls \times \embeddim}$ (where $\embeddim$ is the feature dimension and $\numcls$ is the number of classes).
We learn the parameters $\emplinmat$ by minimizing
\pl{have more English description of the two terms - encourage blah to be close and blah to be far (orthogonal)}
\begin{align}
    \label{eq:sq-loss}
    \lfinetune(\linmat) = \E_{\inputx \sim \sourcedist}\left[ \ell(\linmat \empencoder(\inputx), \labelx) \right] + \eta \| \linmat \|_F^2,
\end{align}
\pl{earlier we used $x'$ as an augmentation of $x$; go back and change to $x^+$ for augmentations and $x'$ for independent input}
where $\ell$ is the squared loss.
The resulting classifier is $\empclf(\inputx) = \argmax_{i \in [\numcls]} (\emplinmat \empencoder(\inputx))_i$.
\pl{define what the loss function is precisely, because you have a vector of scores and a discrete label, saying squared loss doesn't type check}
\ak{It's on the `mean-subtracted' one hot encoding---defined in the Appendix. TODO: add it here in a clean way.}

\subsection{Simple example}
\label{sec:intuitions}

We first consider a toy example that captures the core intuitions of our theory.
For this example, we compute an exact (closed-form) expression for the representations learned by contrastive pre-training on unlabeled data, which are visualized in Figure~\ref{fig:2}.
We then show that a linear classifier trained on source representations also gets high accuracy on examples from the target domain, \emph{without using any target labels}.

\paragraph{Setup.} The goal is to classify between images of clocks and butterflies---see Figure~\ref{fig:2}.
The input space $\inputspace$ consists of 4 points: [clock sketch, butterfly sketch, clock photo, butterfly photo], where the $i$-th input refers to the $i$-th element in the list (e.g., $i=2$ corresponds to butterfly sketch).
\pl{use this example throughout rather than dog photos}
The label space $\labelspace$ contains two classes, clock ($\labelx=1$) and butterfly ($\labelx=2$).
We only have labeled data for the source inputs (sketches, $\domainx=1$), and the goal is to achieve high accuracy on the target inputs (photos, $\domainx=2$).
The source distribution $\sourcedist$ places equal probability on sketches ($\sourcedist(x) = 0.5$ if $\domainx=1$ and $\sourcedist(x) = 0$ otherwise), and the target distribution $\targetdist$ places equal probability on photos ($\targetdist(x) = 0.5$ if $\domainx=2$ and $\targetdist(x) = 0$ otherwise).
The unlabeled distribution $\unlabeldist$ places equal probability on all images (sketches and photos).

\paragraph{Augmentations.} The key ingredient in contrastive pretraining is the augmentation strategy---contrastive pretraining aims to map augmentations of the same input $\inputx$ to similar representations.
These augmentations include aggressive crops which keep only a small part ($8\%$) of an image (e.g., the tail or fur of an animal), and can blur the line between classes.
Augmentations such as color jitter, which also randomly `drops' colors, can also transform the style of an image, blurring the line between domains such as sketches and photos.
As such, we consider augmentations that can change the class or domain of an image.
We define the probability that an image $x$ can augment to an image $x'$ as follows:
\begin{align*}
    A(x' \mid x) = \begin{cases}
        \withinboth' &\text{if}\quad \labelx = \labelxp, \domainx = \domainxp \\
        \acrossdomain' &\text{if}\quad \labelx = \labelxp, \domainx \neq \domainxp \\
        \acrossclass' &\text{if}\quad \labelx \neq \labelxp, \domainx = \domainxp \\
        \acrossboth' &\text{if}\quad \labelx \neq \labelxp, \domainx \neq \domainxp
    \end{cases},
\end{align*}

\paragraph{Key condition for contrastive pre-training in the UDA setting.}
We explain in our simple example how contrastive pre-training achieves zero target error when $\withinboth' > \max\{\pairproba', \pairprobb'\}$ and $\min\{\pairproba', \pairprobb'\} > \pairprobg'$, and we prove a more general version in Section~\ref{subsec:sbm}.
The first condition is satisfied when an augmented image is more likely to stay within the same domain and class, than change the domain or class.
The second condition is satisfied when data augmentation is less likely to change both the domain \textit{and} class of an image than just one of domain or class.
For example, an augmentation is less likely to augment a clock photo into a butterfly sketch, since it involves changing both the style (sketch vs. real) and the semantic content (butterfly vs. clock).

\paragraph{Deriving a closed-form expression for pre-trained representations.} 
Consider a 4-node graph $G$ where the nodes are input examples in $\inputspace$ and where the edge weight between $\inputx, \posx \in \inputspace$ is the probability of sampling augmentations $\inputx$ and $\posx$ from the same original image $\bar{\inputx}$ (Equation~\ref{eqn:pospair}): $\pospairdist(\inputx, \posx)$.
We illustrate this graph in Figure~\ref{fig:2} (left).
Let $A$ denote the weighted adjacency matrix for $G$, which depends on the augmentation parameters $\withinboth', \acrossdomain', \acrossclass', \acrossboth'$:
\begin{align}
    A &= \begin{bmatrix*}[r]
        \pairprobr & \pairprobb & \pairproba & \pairprobg \\
        \pairprobb & \pairprobr &  \pairprobg & \pairproba \\
        \pairproba & \pairprobg & \pairprobr & \pairprobb \\
        \pairprobg & \pairproba & \pairprobb & \pairprobr \\
    \end{bmatrix*}, \text{ where }
    \begin{cases}
        \withinboth = \frac{1}{4}\hspace{-3mm}&({\withinboth'}^2 + {\acrossdomain'}^2 \\
        &+{\acrossclass'}^2 + {\acrossboth'}^2) \\
        \acrossdomain = \frac{1}{2}\hspace{-3mm}&(\withinboth' \acrossdomain' + \acrossclass'\acrossboth')\\
        \acrossclass = \frac{1}{2}\hspace{-3mm}&(\withinboth' \acrossclass' + \acrossdomain'\acrossboth')\\
        \acrossboth = \frac{1}{2}\hspace{-3mm}&(\withinboth' \acrossboth' + \acrossdomain'\acrossclass')
    \end{cases}\nonumber
\end{align}
Recall that we assumed $\withinboth' > \max\{\pairproba', \pairprobb'\}$ and $\min\{\pairproba', \pairprobb'\} > \pairprobg'$.
With some algebra, this implies that the same condition holds for the positive pair probabilities as well: $\withinboth > \max\{\pairproba, \pairprobb\}$ and $\min\{\pairproba, \pairprobb\} > \pairprobg$.

Contrastive pre-training learns an encoder $\empencoder : \inputspace \to \R^\embeddim$ given by Equation~\ref{eq:scl}---we use $\embeddim=3$ here.
Let $\empencoderMatrix \in \R^{4 \times \embeddim}$ be the learned feature matrix, where the $i$-th row of $\empencoderMatrix$ contains representations for example $i$, so $\empencoderMatrix_i = \empencoder(i)$.

From a few lines of algebra (Lemma 3.2 in~\citet{haochen2021spectral}), $\lpretrain(\empencoder) = \frac{1}{16} \| 16 A - \empencoderMatrix \empencoderMatrix^\top \|_2^2 + c$, where $c$ is a constant that does not depend on $\empencoder$.
Here, $A$ has rank 4, while $\empencoderMatrix \empencoderMatrix^\top$ has rank 3---the minimizer $\empencoderMatrix$ is given by the eigenvectors corresponding to the 3 largest eigenvalues of $A$.
We can compute the (unordered) eigenvalues $\lambda_a, \lambda_b, \lambda_c, \lambda_d$ and corresponding eigenvectors $u_a, u_b, u_c, u_d$ of $A$ explicitly:
\begin{align}
&\lambda_a = \acrossdomain + \acrossclass + \acrossboth + \withinboth, \quad
&\lambda_b = -\acrossdomain + \acrossclass - \acrossboth + \withinboth, \nonumber \\
&\lambda_c = \acrossdomain - \acrossclass - \acrossboth + \withinboth, \quad
&\lambda_d = -\acrossdomain - \acrossclass + \acrossboth + \withinboth, \nonumber\\
&u_a = (1, 1, 1, 1)^\top, \quad
&u_b = (1, 1, -1, -1)^\top, \nonumber \\
&u_c = (1, -1, 1, -1)^\top, \quad
&u_d = (1, -1, -1, 1)^\top.
\nonumber
\end{align}
Since $\acrossclass > \acrossboth$ and $\acrossdomain > \acrossboth$, we have:
\begin{enumerate}
    \item $\lambda_a > \lambda_d$ since all $\withinboth, \acrossclass, \acrossdomain, \acrossboth$ are nonnegative,
    \item $\lambda_b > \lambda_d$ when $\acrossclass - \acrossboth > - \acrossclass + \acrossboth$, which is implied by $\acrossclass > \acrossboth$,
    \item $\lambda_c > \lambda_d$ when $\acrossdomain - \acrossboth > -\acrossdomain + \acrossboth$, which is implied by $\acrossdomain > \acrossboth$.
\end{enumerate}
So $\lambda_d$ is the smallest eigenvalue, and the learned features are given by (up to rotational symmetries):
\begin{align}
\empencoderMatrix &= \begin{bmatrix*}[c]
    \empencoder(\text{clock sketch})^\top  \\
    \empencoder(\text{butterfly sketch})^\top  \\
    \empencoder(\text{clock photo})^\top \\
    \empencoder(\text{butterfly photo})^\top
\end{bmatrix*} =
4\begin{bmatrix*}[r]
        \sqrt{\lambda_a} & \sqrt{\lambda_b} & \sqrt{\lambda_c}  \\
        \sqrt{\lambda_a} & \sqrt{\lambda_b} & -\sqrt{\lambda_c} \\
        \sqrt{\lambda_a} & -\sqrt{\lambda_b} & \sqrt{\lambda_c} \\
        \sqrt{\lambda_a} & -\sqrt{\lambda_b} & -\sqrt{\lambda_c}
    \end{bmatrix*}
\nonumber
\end{align}
\begin{figure}[tbp]
\centering
\begin{minipage}{\linewidth}
    \centering
    \includegraphics[width=\arxivstyle{0.70\linewidth}{\linewidth}]{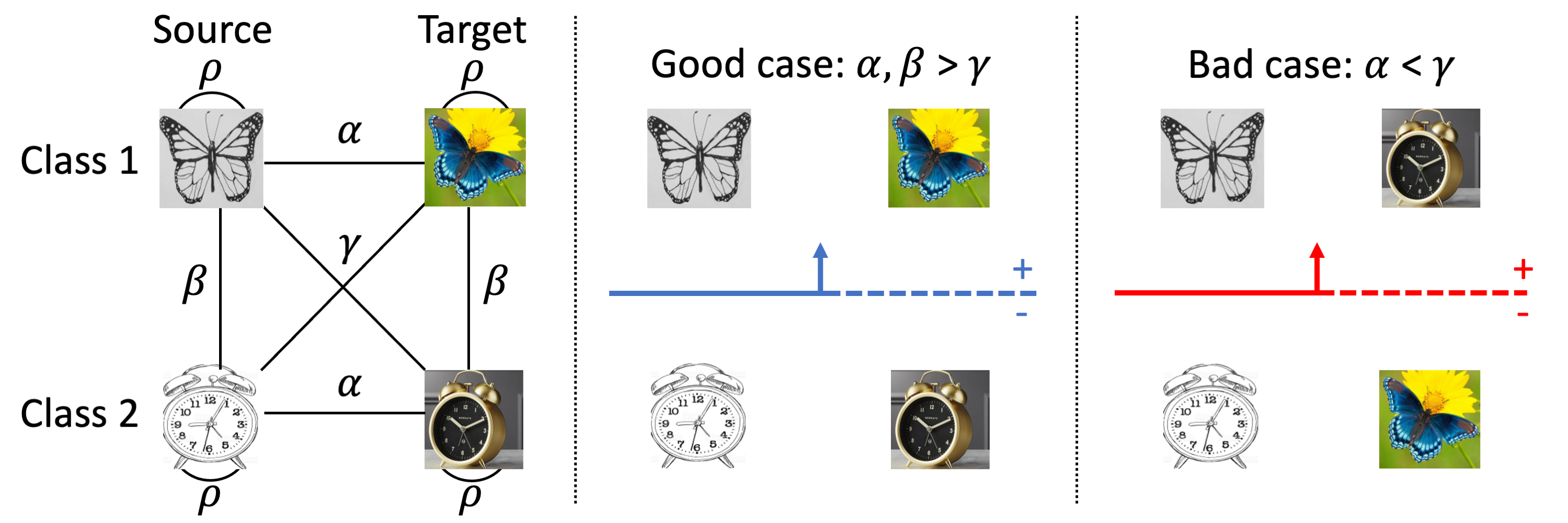}
    \captionof{figure}{%
    \textbf{(Left)}
    Illustrative example in binary classification, where each class-domain pair is a single node in the graph.
    Edge weights denote connectivity (probability of sampling the endpoints as a positive pair). The self-loop weight $\pairprobr$ denotes the probability of sampling a pair of the same domain and class.
    \textbf{(Middle)}
    When $\acrossdomain$ (same class, different domain) and $\acrossclass$ (different class, same domain) are greater than $\acrossboth$ (different domain, different class), the features are oriented so that a source-trained linear classifier generalizes to the target.
    The class and domain information are disentangled along the vertical and horizontal axes, respectively.
    \textbf{(Right)}
    When $\acrossdomain < \acrossboth$, the target features are flipped and the source-trained classifier does not generalize.}
    \label{fig:2}
\end{minipage}
\end{figure}

\paragraph{Fine-tuning on source examples gets zero error on target.}
So far, we have calculated the features learned by contrastive pretraining.
We note that the first feature dimension is the same for all examples and therefore not useful---we visualize the remaining 2-D features in the middle panel of Figure~\ref{fig:2}.
We see that the features disentangle the domain and class information---the second feature dimension can be used to distinguish domains while the third feature can distinguish classes.
Therefore, training a max-margin classifier $\empclf$ on only labeled sketch images (first two rows of the feature matrix $\embedding$) produces a classifier that uses only the third feature, which extrapolates to photos correctly and gets \emph{zero OOD error} ($\Lzeroone(\empclf) = 0$).
The key difference from conventional intuitions in domain adaptation is that contrastive pre-training does not learn domain-invariant features---the source and target are perfectly distinguishable but the domain and class information are disentangled, enabling generalization.

Note that if our key condition on the augmentations does not hold e.g., $\acrossdomain < \acrossboth$ but $\acrossboth < \acrossclass$, then the learned features $\embedding$ exclude the eigenvector corresponding to $\lambda_c$ as the smallest.
Again, the first feature is constant and not useful.
In the visualization (Figure~\ref{fig:2} right), we show the features learned by contrastive pre-training.
In this case, a max-margin classifier $\empclf$ trained on sketch images (red line) mislabels the photos ($\Lzeroone(\empclf)=1$) because the orientation of the classes is flipped between domains.

\pl{This is a really nice example and basically the heart of this paper, so we should really spend sometime to make this as polished and understandable as we can;
I think the main thing is to connect up the notation more explicitly and say what the different feature vectors of the points are (e.g., $\phi(x_1) = [...]$, etc.);
the indirection also makes it hard; maybe in Fig 2, we can label the points 1, 2, 3, 4
since we need to linearize all the points in the matrix;
do we need to number the classes and domains, or can we use {sketch,photo}, {clock, butterfly} directly, which would reduce the indirection
}

\pl{in figure 2, presumably the geometry of the embedded feature space depends on the exact values of $\alpha, \beta, \gamma$;
part of me is thinking that it would be nice to define the actual augmentation functions (write $\aug(...) = $) and also plug
in actual values for $\alpha,\beta,\gamma$ at least for the figure to make things more concrete
}
\ak{I agree with being more explicit about what the features are, so I added that. In the toy example, the geometry of the embedded features don't depend much on the values (as long as our key condition holds). It could become rectangular instead of square. If you have some asymmetries then it will look different (happy to explain more)}

\pl{in figure 2, could we add another bad case where $\beta < \gamma$?}

\pl{also interesting that $\rho$ doesn't seem to matter at all, maybe we want to say/discuss that?}
\ak{Yeah that's a good point, I changed this section to start from the augmentation probabilities (not positive pair probabilities), in which case $\rho$ does matter.}

\subsection{Theory for multi-class stochastic block model}
\label{subsec:sbm}

We extend the simple example in Section~\ref{sec:intuitions} to a stochastic block model setting with $r$ classes, where we have $\cdsize$ points for each class and domain.
The unlabeled data distribution $\unlabeldist$ is the uniform distribution over $\inputspace$.

The stochastic block model (SBM)~\citep{holland1983sbm} is a random graph model widely used in theoretical computer science to model communities, networks, and topic models in NLP~\citep{bouveyron2016sbmtext,abbe2018sbm,mehta2019sbmgnn}.
The graph $\sbmg = (\inputspace, \sbme)$ is defined as follows:
for each pair of nodes $\inputx, \inputxp \in \inputspace$, the corresponding edge (undirected, unweighted) is in the edge set $\sbme$ with probability:
\begin{align*}
    \begin{cases}
        \withinboth &\text{if}\quad \labelx = \labelxp, \domainx = \domainxp \\
        \acrossdomain &\text{if}\quad \labelx = \labelxp, \domainx \neq \domainxp \\
        \acrossclass &\text{if}\quad \labelx \neq \labelxp, \domainx = \domainxp \\
        \acrossboth &\text{if}\quad \labelx \neq \labelxp, \domainx \neq \domainxp
    \end{cases},
\end{align*}
and all edges are sampled independently.
The distribution of positive pairs $\pospairdist$ is the uniform distribution over $\sbme$.
\pl{does this actually correspond to some augmentation function? it's important that this actual exist, and we can know what kind of augmentations can give rise to this}
\pl{in particular, given this, we don't have independent control over the marginal distribution over $x$}

\pl{it would be nice to have an example of a random graph model where you draw the edges and use color to denote the class and domain}

Our key condition is that 1) augmentations are less likely to change both the domain and class ($\min\{\acrossdomain, \acrossclass\} > \acrossboth$) and 2) augmentations are more likely to keep both class and domain unchanged than to change either one of the two ($\withinboth > \max\{ \acrossdomain, \acrossclass \}$) \pl{how come this didn't come up before?}.
Assuming this \pl{/Given these two assumptions} and using a feature dimension $\numcls + 1$, the following result shows that contrastive pre-training on unlabeled source and target data learns transferable features so that a linear head trained on the source features achieves high accuracy on the target.
\pl{need to talk more explicitly about learning the feature map since that that's the star of the show - there should be notation for it}
\begin{theorem}
    \label{thm:sbm}
    In the SBM setting above, let $\withinboth > \max\{ \acrossdomain, \acrossclass \}$ and $\min\{ \acrossdomain, \acrossclass \} > \acrossboth$ and let the pre-training feature dimension be $r+1$.
    Then, for any $n \geq \Omega\left(\frac{r}{\min\{\acrossdomain - \acrossboth, \acrossclass - \acrossboth\}^2}\right)$ and regularization strength \pl{where does regularization enter the algorithm? pre-training or fine-tuning?}\ak{fine-tuning, defined in the previous section, but maybe worth adding here too} $\eta \in \left(0, \frac{\acrossdomain - \acrossboth}{8 r \withinboth}\right)$, with probability at least $0.999$, we have
    \begin{align*}
        \Lzeroone(\empclf) \leq
        O\left(\frac{1}{\eta^4 \cdot (\min\{\acrossdomain,\acrossclass\} - \acrossboth)^2 \cdot \cdsize} \right) \cdot \text{poly}\left(\numcls\right).
    \end{align*}
    where $\empclf(\inputx)$ is the linear classifier trained on source domain pre-trained features (Section~\ref{sec:theory-setup}).
\end{theorem}
The rate at which the error decreases with the number of samples $n$ (per class-domain pair) is controlled by the gap between the across-class/across-domain connectivities $\acrossclass,\acrossdomain$ and the across-both connectivity $\acrossboth$.
Thus, augmentations that tend to change only one of class or domain will improve the downstream transferability of the learned features.

\pl{should interpret things a bit: the ideal augmentation would want $\alpha$ to be large, that's the main thing that would increase $\eta$ and thus lower the error}

\pl{also the warmup case of $\beta = \gamma$ makes this bound blow up ...is there a way around this?}

\paragraph{Proof sketch.}
The features learned by contrastive pre-training are given by the top $\embeddim$ eigenvectors of the adjacency matrix $A$, which is a random matrix defined by the SBM.
For the expected adjacency matrix $\E[A]$, similarly to the simple example, we can compute the eigenvectors in closed form and show that the linear head learned on the source data achieves low error on the target.
The main challenge is to show an analogous result for the \textit{sampled} graph, where each of the data points per class-domain pair can have a different set of edges.
We use matrix concentration bounds to concentrate the top eigenvectors of $A$ to those of $\E[A]$ and use matrix perturbation analysis to show that the predictor learned using $A$ is close to the ``ideal'' one learned using $\E[A]$.
This shows that contrastive pre-training on the random graph defined by $A$ also learns transferable features with high probability, which gives the result.\qed

The full proof is in Appendix~\ref{app:sbm}, where we show that the result holds even when we have more than 2 domains.
\pl{should this be surprising? italics/phrasing suggests it is, but it is natural but important}
This suggests that we can pre-train one model on the unlabeled data of many domains and the features can transfer across all of them.
Our bound in the appendix also holds with probability arbitrarily close to 1.

As discussed earlier, contrastive pre-training does not merge source and target features.
Our next result mirrors this observation theoretically in the SBM setting, showing that a linear classifier on top of the pre-trained features can classify the domain with low error.
Letting $\dhead: \Real^\embeddim \rightarrow \{1, 2\}$ be a domain classifier \pl{use the same notation as before (which was $\xi$)}, we define the 0-1 loss for domain classification as $\LDzeroone(\dhead\circ\empencoder) = \E_{\inputx \sim \unlabeldist}[\indicator[\dhead(\empencoder(\inputx)) \neq \domainx]]$.
\begin{prop}
	In the setting of Theorem~\ref{thm:sbm}, with probability at least $0.999$, there exists a linear classifier $\dhead: \Real^\embeddim \rightarrow \{1, 2\}$, such that $\dhead$ composed with the encoder $\empencoder$ can distinguish the domains:
	\label{proposition:domain_separation}
	\begin{align*}
		\LDzeroone(\dhead \circ \empencoder)  \le O\left(\frac{1}{ \min\{\acrossdomain - \acrossboth, \acrossclass - \acrossboth\}^6 \cdot \cdsize} \right) \cdot \text{poly}\left(\numcls\right).
	\end{align*}
\end{prop}
The proof is in Appendix~\ref{app:sbm} and is similar to Theorem~\ref{thm:sbm}.
\pl{rhetorically a bit strange since we say how great it is that we can do more than 2 domains, but the next result is only for two domains}

\pl{what about disentangling - this is never defined generally beyond just having different coordinates be class versus domain?}

\subsection{Setting where ERM/DANN underperform contrastive pre-training}
\label{subsec:separation-example}

We theoretically show a simple setting in which contrastive learning on unlabeled data can have higher target extrapolation accuracy than both ERM and DANN (Equations~\ref{eq:erm-object-setup} and~\ref{eq:dann-object-setup}, respectively).
The main intuition of our construction is that ERM and DANN may underperform when there are subsets of the target data distribution that are unreachable via data augmentation on any labeled source input---in other words, inputs $\inputx$ with $\domainx=2$ such that $\aug(\inputx \mid \inputx') = 0$ for all $\inputx'$ with $\domainxp=1$.
This makes the problem underconstrained for both ERM and DANN, while contrastive pre-training better leverages the connectivity information from augmentations on unlabeled source and target inputs to learn transferable features.
\begin{restatable}{prop}{separationprop}
    \label{prop:separation}
    There exists a set $\inputspace$, a distribution $\unlabeldist$ and data augmentation $\aug$, such that for some feature dimension $k\in\Z^+$, a linear classifier trained on contrastive pre-trained features achieves 0 target error: $\Lzeroone(\empclf) = 0$.
    However, for all $\embeddim\in\Z^+$, there exist minimizers $\empclf_\text{erm}$ and $\empclf_\text{dann}$ of the ERM and DANN objectives, respectively, that have non-zero error: $\Lzeroone(\empclf_\text{erm}) = \Lzeroone(\empclf_\text{dann}) = 1/3$.
\end{restatable}
The proof is in Appendix~\ref{app:separation-proof}.
For ERM, an optimal predictor that minimizes the ERM objective~\eqref{eq:erm-object-setup} can make arbitrary predictions on unseen target examples.
For DANN, simply matching the marginal input distributions (even with augmentations) across domains can potentially swap the classes of the unreachable target data---this is consistent with the well-known drawback that DANN with an expressive feature extractor can learn domain-invariant features by arbitrarily pairing the source and target examples~\citep{shu2018dirtt,zhao2019zhao}.
Conversely, contrastive pre-training learns features that enable perfect transferability.

\section{Contrastive pre-training is a strong domain adaptation method}
\label{sec:main-results}


Our theory gives conditions for which contrastive learning produces effective features for domain adaptation.
We now empirically validate contrastive pre-training as a competitive domain adaptation method.
We observe that contrastive pre-training achieves comparable performance to strong UDA methods on four standard domain adaptation datasets.

\paragraph{Datasets.}
We conduct experiments on DomainNet~\citep{peng2019moment,prabhu2021sentry}, which contains 40 classes and 4 domains, BREEDS Living-17 and Entity-30~\citep{santurkar2020breeds}, which are adaptation benchmarks derived from ImageNet, and STL-10$\to$CIFAR-10~\citep{coates2011stl10,krizhevsky2009learningmultiple,french2018selfensembling}, which are two classical image recognition datasets often paired together for domain adaptation.

\paragraph{Contrastive pre-training algorithm.}
We use SwAV~\citep{caron2020swav}, a contrastive pre-training algorithm with high accuracy on ImageNet, for our ImageNet-like datasets (BREEDS and DomainNet).
SwAV uses a multi-crop data augmentation strategy with several crops of different sizes, followed by horizontal flipping, color distortion, and Gaussian blurring.
We pre-train on the combined source and target unlabeled data and fine-tune on the source using the same augmentation pipeline used during pre-training.
For STL$\to$CIFAR, we use a publicly available SimCLR~\citep{chen2020simclr} model that is pre-trained on CIFAR.
SimCLR applies random cropping, color distortions, and Gaussian blur for data augmentation.
In our theory we linear probe with the squared loss, but in our experiments we follow a more standard approach and fine-tune the model (pre-trained encoder and randomly initialized linear head) with the cross-entropy loss.

\paragraph{Baselines.}
We compare with standard ERM on the labeled source data and two strong domain adaptation methods: DANN~\citep{ganin2016domain} and SENTRY~\citep{prabhu2021sentry}.
SENTRY achieves SoTA results on DomainNet when initialized with ImageNet-pre-trained models.
However, to ensure the methods all use the same data, we initialize the domain adaptation methods with the ERM baseline (no ImageNet pre-training).
For STL$\to$CIFAR, we compare to Dirt-T~\citep{shu2018dirtt}, a SoTA domain adaptation algorithm for this task.
For a more complete comparison, we consider not only using each method's default augmentations but also using the contrastive pre-training augmentations to all baselines (denoted +\strongaugs).

\subsection{Results}
\paragraph{Main comparison (Table~\ref{table:main-empirical}).}
\begin{table*}[t]
\begin{minipage}{\linewidth}
\centering
\resizebox{0.9\linewidth}{!}{
\begin{tabular}{l r r r r r r r r r}
	\toprule
    Method & \multicolumn{2}{c}{ERM} & \multicolumn{2}{c}{SENTRY} & \multicolumn{2}{c}{DANN} & Pre-training & SwAV+extra \\
    \cmidrule(lr){2-3} \cmidrule(lr){4-5} \cmidrule(lr){6-7} \cmidrule(lr){8-8} \cmidrule(lr){9-9}
    Augmentations & Standard & Strong & Standard & Strong & Standard & Strong & Strong & Strong \\
	\midrule
    DomainNet (avg. of 12 pairs) & 26.67 & 37.68 & 31.58 & 25.82 & 38.38 & \textbf{45.81} & \underline{44.91} & 51.73 \\
    Living-17 & 60.17 & 63.29 & \textbf{75.53} & 74.53 & 65.59 & 71.29 & \underline{75.12} & 82.00 \\
    Entity-30 & 52.68 & 52.52 & 56.10 & \underline{58.90} & 57.45 & 57.52 & \textbf{62.03} & 65.90 \\
    STL-10$\to$CIFAR-10 & 44.83 & \underline{57.40} & 53.84 & 43.71 & 46.77 & 55.18 & \textbf{75.41} & N/A \\
    \bottomrule
\end{tabular}
}
\end{minipage}
\caption{%
    Test accuracy (\%) of baselines and contrastive pre-training on 4 benchmark visual adaptation datasets.
    The first and second highest numbers for each dataset are bolded and underlined, respectively.
    The pre-training algorithm is SwAV for DomainNet, Living-17, and Entity-30, and SimCLR for STL$\to$CIFAR.
    Contrastive pre-training is competitive with the baselines for all datasets.
    On STL$\to$CIFAR, SimCLR is slightly higher than Dirt-T, the previous SoTA (75.3\% as reported in~\citet{shu2018dirtt}).
    SwAV+extra is not bolded as it uses additional data and is therefore not directly comparable.
    }
\label{table:main-empirical}
\end{table*}

On average over all pairs of DomainNet, SwAV is within 1\% of the best baseline, DANN+\strongaugs{} (44.91\% vs. 45.81\%).
Contrastive pre-training is comparable with the best baseline (SENTRY) on Living-17 (75.12\% vs. 75.53\%) and improves over the best baseline (SENTRY+\strongaugs) by 3.1\% on Entity-30.
On STL$\to$CIFAR, the SimCLR target accuracy (75.4\%) is very close to the Dirt-T accuracy reported in the original paper (75.3\%).
We find that replacing the default data augmentation used by each baseline with the augmentations used by contrastive pre-training consistently boosts performance for DANN but only sometimes boosts performance for SENTRY.
Table~\ref{table:app-empirical-domainnet} provides the individual results for each pair of domains on DomainNet, and Tables~\ref{table:app-empirical-domainnet} and~\ref{table:app-breeds} contain results from additional contrastive pre-training methods, MoCo-V2~\citep{chen2020improved} and MoCo-V3~\citep{chen2021empirical} (all trained on the same data).

\paragraph{Extra unlabeled data from related domains.}
We also consider adding extra unlabeled data from other (related) domains in DomainNet, Living-17, and Entity-30.
In particular, we pre-train once on an unlabeled dataset consisting of a superset of the domains we want to adapt to, fine-tune on labeled data from one of the domains, and evaluate on all other domains.
While this is not a fair comparison to other domain adaptation methods, the ability to scale to large unlabeled datasets by pre-training on all domains simultaneously is a natural advantage of pre-training.
This method, which we denote as SwAV+extra, gives further improvements over SwAV (nearly 7\% on DomainNet, 7\% on Living-17, and 4\% on Entity-30).
In Table~\ref{table:app-breeds} we show results from pre-training on different splits of the BREEDS data, as well as using additional pre-training methods: Dino+extra~\citep{caron2021emerging} and Barlow Twins+extra~\citep{zbontar2021barlow}.

\section{Evaluating connectivity on real datasets}
\label{sec:connectivity-governs}



Recall that our theory (Theorem~\ref{thm:sbm}) provides conditions on connectivities between classes and domains for which contrastive pre-training provably obtains good target accuracy.
We heuristically estimate each of the connectivity measures 
on Living-17 and DomainNet to verify the predictions from our theory in Section~\ref{sec:connectivity}.
We show that the empirical connectivities satisfy our theoretical conditions for contrastive pre-training to learn transferable features (e.g., across-domain $>$ across-both connectivity) and that the connectivity ratios (e.g., across-domain / across-both) are predictive of target domain accuracy of contrastive pre-training.

\paragraph{Estimating connectivity on benchmark datasets.}
To verify whether real datasets and augmentations satisfiy the connectivity requirements for Theorem~\ref{thm:sbm}, we compute empirical estimates of the connectivity measures.
Using augmented images from 2 class-domain pairs, we train a classifier to predict the class-domain pair from which the image originated.
For example, to estimate across-domain connectivity ($\acrossdomain$ from Section~\ref{sec:intuitions}), we choose 2 class-domain pairs with the same class and different domain.
We find that connectivity ratios satisfy our theoretical conditions in all cases,
and the estimates are reported in Table~\ref{table:connectivity}.
The input space connectivity is calculated by training classifiers from scratch on individual class-domain pairs, leading to much smaller training dataset sizes compared to SwAV and DANN (which are trained on all classes and both domains).
Therefore, because the input space connectivity numbers may be overestimates of the true connectivity, we also considered fine-tuning a CLIP~\citep{radford2021clip} pre-trained model rather than training from scratch.
See Appendix~\ref{sec:app-connectivity} for more discussion and details.

\begin{figure*}[t]
    \centering
    \includegraphics[align=c,height=3.4cm]{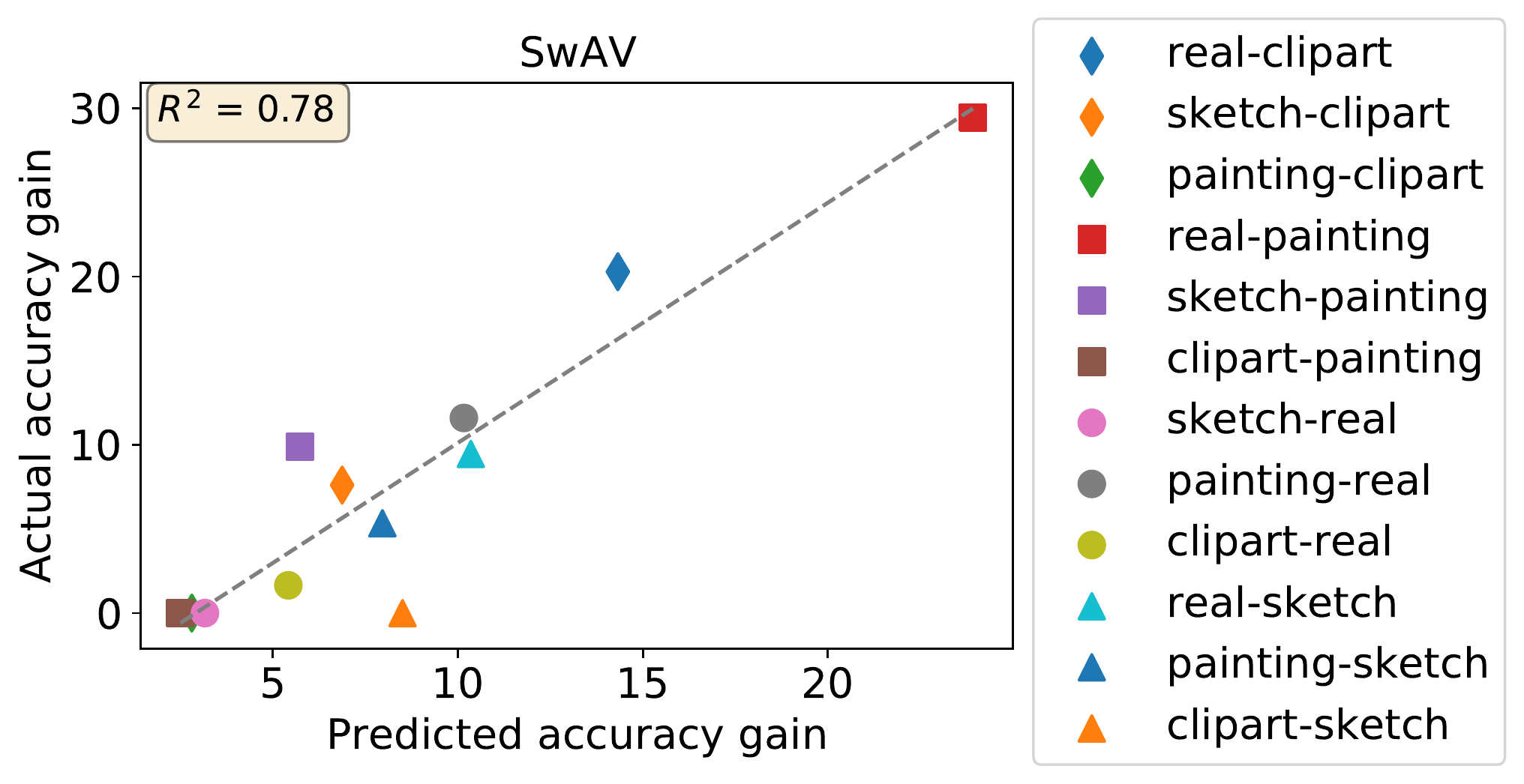}
    \hspace{0.08\linewidth}
	\includegraphics[align=c,height=3.4cm]{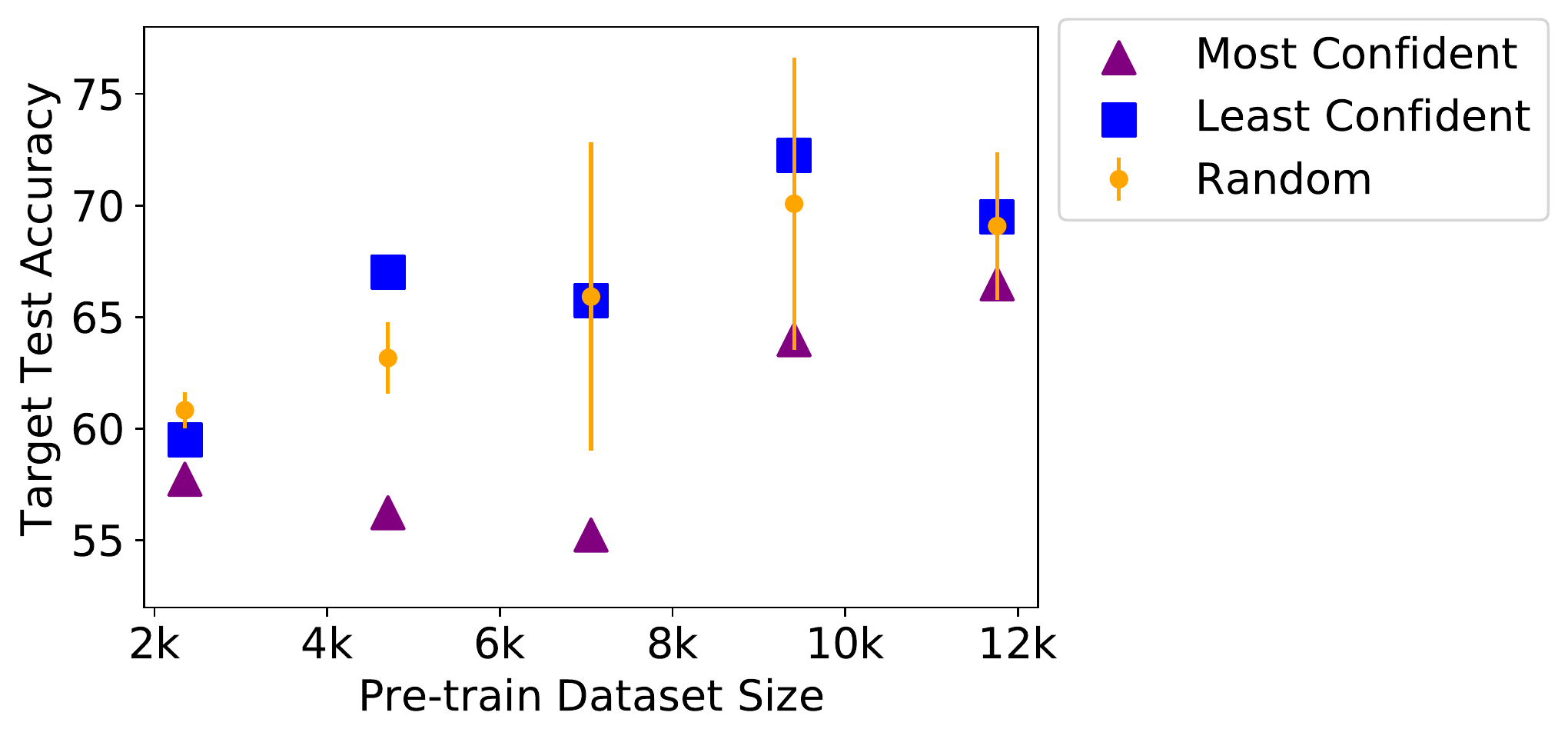}
	\captionof{figure}{%
    \textbf{(Left)}
    True target accuracy of SwAV (gain over the worst source/target pair with the same target) vs. the predicted accuracy using the connectivity ratios on 12 pairs of DomainNet domains.
    The predictions have a high coefficient of determination with the observed ($R^2 = 0.78$).
    \textbf{(Right)}
    Ablation of connectivity: pre-training only on inputs far from the margin of a domain classifier (``most confident'') consistently degrades target accuracy more than both on random subsets (mean and standard deviation over 3 random subsets shown) and on inputs nearest the margin (``least confident'').
  }
	\label{fig:connectivity-governs}
\end{figure*}

\paragraph{Connectivity ratios correlate with target accuracy.}
\newcommand{\wone}{\ensuremath{w_1}}
\newcommand{\wtwo}{\ensuremath{w_2}}
Section~\ref{sec:connectivity} suggests that when the across-domain ($\acrossdomain$) and across-class ($\acrossclass$) connectivities are larger than the across-both ($\acrossboth$) connectivity (i.e., $\acrossdomain/\acrossboth$ and $\acrossclass/\acrossboth$ are large),
contrastive pre-training learns features with good target accuracy.
To investigate the relation between the connectivity ratios and target accuracy, we consider fitting the following function, with parameters $\wone$ and $\wtwo$, to the SwAV target accuracy on the 12 pairs of DomainNet domains:
\begin{align}
    \label{eq:connect-ratio}
    \text{target accuracy} \approx (\acrossdomain / \acrossboth)^{\wone} \cdot (\acrossclass / \acrossboth)^{\wtwo}.
\end{align}
Since our theory predicts that target accuracy is high when both ratios are large, we multiply them to express the logical ``and''.
To normalize by the differing intrinsic difficulties of transferring to each target domain of DomainNet, we fit Eq.~\ref{eq:connect-ratio} to the improvement over the worst pair over all source/target pairs with the same target domain.
Linear regression in log space yields $\wone = 14.9$ and $\wtwo = 2.7$, with a strong coefficient of determination ($R^2 = 0.78$) between the predicted and observed accuracies (left panel of Figure~\ref{fig:connectivity-governs}).
Thus, the target accuracy can be well-explained \textit{using the connectivity ratios alone}.

Figures~\ref{fig:trend-plots-app-one} and~\ref{fig:trend-plots-app-two} show that the predicted target accuracies using this method are also accurate for MoCo-V2~\citep{chen2020improved} ($R^2 = 0.79$) and MoCo-V3~\citep{chen2021empirical} ($R^2 = 0.60$), but much less accurate for the baselines (both with and without augmentations; average $R^2 = 0.21$).

\paragraph{Ablating connectivity degrades target accuracy.}
We find that the across-domain connectivity is very important, as those examples ``bridge'' the domains---for instance, our linear regression fit puts the largest weight on the across-domain connectivity ratio.
We verify this intuition by training a domain classifier on Living-17 and pre-training on the subset of examples on which the classifier is most confident (i.e., the examples farthest from the classifier's decision boundary, which intuitively contribute the least to across-domain connectivity).
As controls, we also pre-train on
1) random subsets of the same size and
2) the examples on which the classifier is least confident.
For 5 different subset sizes, our confidence pruning method consistently reduces target accuracy compared to both controls (Figure~\ref{fig:connectivity-governs} right).

\paragraph{Pre-trained features approximately disentangle class and domain.}
\begin{table}[t]
  \centering
\scalebox{0.85}{
\begin{tabular} {l r r r}
\toprule
    & Class (S) & Class (S) & Class (T) \\
    & vs. Class (T) & vs. Domain & vs. Domain \\
    \midrule
    Living-17 & 0.397 & 0.013 & 0.016 \\
    DomainNet (avg.) & 0.187 & 0.018 & 0.018 \\
    \bottomrule
\end{tabular}
}
\caption{%
    Cosine similarity of class and domain classifiers trained on SwAV representations on Living-17 and DomainNet (average over all classes).
    Linear classifiers trained to predict the class on the source and target representations individually learn similar weights and are nearly orthogonal to the linear weights learned by domain classifiers.}
\label{table:dot-products}
\end{table}

In Section~\ref{sec:connectivity}, we showed theoretically that contrastive pre-training can learn a feature space that is simultaneously predictive of the class and domain by disentangling the information along separate directions.
Here, we verify this empirically using the SwAV feature space.
Given a source and target dataset, we test disentanglement by training the following linear classifiers on the feature space:
1) Source classifier, which predicts class in the source domain;
2) Target classifier, which predicts class in the target domain;\footnote{Although this cannot be done in practice, we only use the target labels here for exploratory analysis.}
and 3) Domain classifier, which predicts domain (source or target).
If the class and domain information are approximately disentangled along different dimensions, then the linear weights of the source/target classifier should be orthogonal to the weights of the domain classifier.

We train the linear classifiers and compute the cosine similarity between their weights on Living-17 and all pairs of domains from DomainNet (Table~\ref{table:dot-products}).
First, we find that the domain classifier is nearly orthogonal to the source and target classifiers, with an average cosine similarity $< 0.02$ between the linear weights of the domain classifier and the weights of the source and target classifiers.
Second, we find that the source and target classifiers are relatively well-aligned: on average over all classes, the cosine similarity of linear weights for the same class from the source and target classifiers is high (0.40 on Living-17 and average 0.19 on DomainNet).
Full results are in Appendix~\ref{app-verification}.

\section{Conclusion}
While off-the-shelf contrastive pre-training is not intentionally a domain adaptation method, we find that not only does it theoretically and empirically transfer well across distribution shifts, but it does so in a way that runs counter to conventional domain adaptation theory and practice by not learning domain-invariant features.
Given some of the practical advantages of pre-training (such as being able to pre-train once and then finetune for many different downstream tasks), we hope that our connectivity theory leads to improvements for contrastive pre-training as a domain adaptation method, such as improving pre-training data selection, developing augmentations to increase connectivity, and improving fine-tuning methods that exploit the geometry of the pre-trained feature space.
Our connectivity theory can also potentially explain when contrastive pre-training for domain adaptation does \emph{not} perform well, such as on the iWildCam2020-WILDS dataset~\citep{sagawa2022uwilds}.
Future work may also lead to new domain adaptation methods that focus on increasing connectivity rather than collapsing domains.

Our theoretical setup is highly stylized and the stochastic block model may not be a realistic assumption; in particular, it requires a uniform marginal input distribution.
Follow-up work~\citep{haochen2022beyond} proves the linear transferability of contrastive pre-training in a more general setting and only requires that the same-class across-domain connectivity is larger than the across-class across-domain connectivity (rather than also requiring across-class same-domain connectivity to be large).
Their analysis leads to an improved linear probing algorithm which outperforms both SwAV and MoCo-V3 on Entity-30.






\section{Ethics}

Our experiments use standard benchmark vision datasets in domain adaptation, which is publicly available data.
In general, however, large-scale unsupervised learning can involve data scraped from the internet which may lead to privacy concerns.
Care should be taken when curating datasets for large scale contrastive pre-training in practice.

\section{Acknowledgements}
This work was in part supported by the Open Philanthropy Project and NSF Award Grant No. 1805310, and NSF IIS 2045685. KS was also supported by the Stanford Computer Science department’s CURIS program. AK was also supported by the Rambus Corporation Stanford Graduate Fellowship. SMX was also supported by an NDSEG Fellowship. We would like to thank Kaylee Burns, Aditi Raghunathan, and the anonymous reviewers, for helpful comments on our draft.

\bibliography{refdb/all, main}
\bibliographystyle{iclr2022}

\newpage
\appendix
\onecolumn

%

\newcommand{\norm}[1]{\left\lVert#1\right\rVert}

\newcommand{\nclass}{r}
\newcommand{\ndomain}{m}
\newcommand{\ndata}{n}
\newcommand{\Ndata}{N}
\newcommand{\dom}[1]{d_{#1}}
\newcommand{\class}[1]{{y_{#1}}}
\newcommand{\id}[1]{{\textup{id}_{#1}}}
\newcommand{\data}{\mathcal{X}}
\newcommand{\pss}{\rho} 
\newcommand{\psd}{\beta} 
\newcommand{\pds}{\alpha} 
\newcommand{\pdd}{\gamma} 
\newcommand{\sdata}{\mathcal{S}} 
\newcommand{\tdata}{\mathcal{T}} 

\newcommand{\shead}{\hat{h}} 
\renewcommand{\head}{h} 

\newcommand{\dfeature}{k} 
\newcommand{\pred}{\textup{pred}} 
\newcommand{\probP}{\text{I\kern-0.15em P}}


\section{Additional details for Section~\ref{sec:connectivity}}

\subsection{Proof of Theorem~\ref{thm:sbm}}
\label{app:sbm}

We first summarize the setting of the multi-domain multi-class classification problem that we will analyze.  

Consider a multi-way classification problem, where $\nclass$ is the number of classes, $\ndomain$ is the number of domains, $\ndata$ is the number of data in each class of each domain. The total number of data is $\Ndata = \nclass\ndomain\ndata$. The set of all nodes is $\data$. 

For data $x\in\data$, we use $\dom{x}\in[\ndomain]$ and $\class{x}\in[\nclass]$ to denote its domain and class, respectively.

We consider a stochastic block model for the graph, where the probability of existence of an edge between $x$ and $x'$ is: (1) $\pss$ if $\dom{x}=\dom{x'}$ and $\class{x}=\class{x'}$, (2) $\pds$ if $\dom{x}\ne\dom{x'}$ and $\class{x}=\class{x'}$, (3) $\psd$ if $\dom{x}=\dom{x'}$ and $\class{x}\ne \class{x'}$, (4) $\pdd$ if $\dom{x}\ne \dom{x'}$ and $\class{x}\ne \class{x'}$. 

Let $A\in\Real^{\Ndata\times\Ndata}$ be the adjacency matrix of the graph. A random positive pair $(x, x^+)$ is a uniform random edge (i.e., $p_+(x, x^+) = A_{xx^+} / \sum_{x', x''} A_{x'x''}$), a negative pair $(x, x^-)$ are two uniform random nodes. Let $\dfeature=\ndomain+\nclass-1$ be the feature dimension and $f: \data \rightarrow \Real^{\dfeature}$ be the feature map learned by minimizing the population spectral contrastive loss:

\begin{align}
	-2\Exp_{x, x^+}	[f(x)^\top f(x^+)] + \Exp_{x, x^-}\left[\left(f(x)^\top f(x^-)\right)^2\right].
\end{align}

Let $\sdata = \{ x\in \data: \dom{x}=1\}$ and $\tdata = \{x\in\data: \dom{x}\ne 1\}$ be source and target domains, respectively. Given the labeled source domain data, we learn the linear head 
\begin{align}
\shead = \argmin_{\head \in \Real^{(\dfeature \times \nclass)}} \sum_{x\in\sdata} \left(\norm{b^\top f(x) - \vec{y}_{x}}_2^2 + \eta \norm{b}_F^2\right),
\end{align}
where $\vec{y}_{{x}} = e_{\class{x}}  - \frac{1}{\nclass} \cdot \mathbf{1}\in \Real^\nclass$ is the mean-zero one-hot embedding of the label, $\eta>0$ is the regularization strength. For data $x\in\tdata$, we use $\pred(x)=\argmax_i (\hat{b}^\top f(x))_i$ as the predictor.

We have the following more generalized version of Theorem~\ref{thm:sbm}.

\begin{theorem}\label{theorem:sbm}
	Let $\zeta>0$ and $\epsilon\in\left(0, \frac{1}{2}\right)$ be arbitrary constants.
	In the above stochastic block model, assume $\rho > \max\{\alpha, \beta\}$ and $\gamma < \min\{\alpha, \beta\}$.
	Then, there exists $\tilde{\xi}\in[1-\epsilon, 1]$, such that for any
	$\ndata\ge\Omega\left(\frac{\nclass\ndomain}{\min\{\alpha-\gamma, \beta-\gamma\}^2}\right)$ and regularization strength $\eta\in\left(0, \frac{(\alpha-\gamma)\epsilon}{2\nclass\rho}\right]$,
	with probability at least $1-\ndata^{-\zeta}$, we have
	
	\begin{align}
		\sum_{x\in\tdata} \norm{\hat{b}^\top f(x) - \tilde{\xi}\vec{y}_{x}}_2^2 \le O\left(\frac{1}{\eta^4 \cdot \min\{\alpha-\gamma, \beta-\gamma\}^2}\right)\cdot \textup{poly}(\nclass, \ndomain).
	\end{align}
	Furthermore, the target error can be bounded by
	\begin{align}
		\probP_{x\sim\tdata}\left(\pred(x)\ne \class{x}\right) \le O\left(\frac{1}{\eta^4 \cdot \min\{\alpha-\gamma, \beta-\gamma\}^2 \cdot \ndata}\right)\cdot \textup{poly}(\nclass, \ndomain).
	\end{align}	

\end{theorem}

Note that setting $\epsilon=\frac{1}{4}$ and $\zeta=1$ proves Theorem~\ref{thm:sbm}. In the rest of this subsection, we will give a proof of Theorem~\ref{theorem:sbm}.

Let $A_k\in\Real^{\Ndata\times\Ndata}$ be the rank-$k$ approximation of the adjacency matrix $A$, which contains the top $k$ components of $A$'s SVD decomposition. We use $A_{k, (\tdata, \sdata)}$ to denote the matrix by restricting $A_k$ to the rows corresponding to the source and the columns corresponding to the target. We use $A_{k, (\sdata, \sdata)}$ to denote the matrix by restricting $A_k$ to the rows and columns corresponding to the source.

Let $Y\in\Real^{\Ndata\times \nclass}$ be the label matrix where on the $x$-th row it contains the label target $\vec{y}_{{x} = }e_{\class{x}} - \frac{1}{\nclass}\cdot \mathbf{1}$. Let $Y_\sdata\in\Real^{|\sdata|\times \nclass}$ and $Y_\tdata \in\Real^{|\tdata|\times \nclass}$ be the matrices by restricting $Y$ to the source and target domain, respectively. 

Let $\pred\in\Real^{\Ndata\times \nclass}$ be the matrix with $\hat{b}^\top f(x)$ on its $x$-th row. Let $\pred_\tdata$ be the matrix by restricting $\pred$ to the target domain. The following lemma gives a closed-form expression for the prediction on the target domain.

\begin{lemma}\label{lemma:closed_form_expression}
	In the setting of Theorem~\ref{theorem:sbm}, let $|E|\triangleq\sum_{x, x'}A_{xx'}$ be the total number of edges. Then, 
	\begin{align}
		\pred_\tdata =  A_{k, (\tdata, \sdata)} \left(A_{k, (\sdata, \sdata)}+ \frac{|E|}{N^2} \cdot \eta\cdot |S|\cdot  I\right)^\dagger Y_\sdata,
	\end{align}
	where $(\cdot)^\dagger$ is the Moore–Penrose inverse, $|\sdata|$ is the number of data in the source domain.
\end{lemma}

\begin{proof}[Proof of Lemma~\ref{lemma:closed_form_expression}]

By the definition of spectral contrastive loss, we can rewrite the loss function as 
\begin{align}
  &-2\sum_{x, x'} \frac{A_{xx'}}{|E|} f(x)^\top f(x') + \sum_{x, x'} \frac{1}{N^2} \left(f(x)^\top f(x')\right)^2\\
  =& \norm{\frac{\Ndata}{|E|}\cdot {A} - \left(\frac{1}{\sqrt{N}}\cdot F\right)\left(\frac{1}{\sqrt{N}}\cdot F\right)^\top}_F^2 + \textup{const},
\end{align}
where $F\in\Real^{\Ndata\times k}$ is the matrix which the $x$-th row contains $f(x)^\top$. According to the Eckart–Young–Mirsky theorem, the minimizer of this loss function is $F = \frac{N}{\sqrt{|E|}} S_k$ where $S_k\in\Real^{\Ndata\times k}$ is a matrix such that $A_k = S_k^\top S_k$. 

Let $S_{k, \sdata}  \in \Real^{|\sdata|\times k}$ be the matrix gotten by restricting $S_k$ to the rows corresponding to the source data, and $S_{k, \tdata}$ be the matrix gotten by restricting $S_k$ to the rows corresponding to the target data. 
The head learned on the source domain can be expressed as 
\begin{align}
	\shead &= \argmin_{\head \in \Real^{(\dfeature \times \nclass)}} \sum_{x\in\sdata} \left(\norm{\head^\top f(x) - \vec{y}_{{x}}}_2^2 + \eta \norm{\head}_F^2\right)\\
	&= \frac{\sqrt{|E|}}{\Ndata} \cdot S_{k, \sdata}^\top \left(S_{k, \sdata} S_{k, \sdata}^\top  + \frac{|E|}{N^2} \cdot \eta\cdot |S|\cdot  I \right)^\dagger Y_\sdata.
\end{align}

Therefore, the prediction on the target domain $\pred_\tdata$ is
\begin{align}
	\pred_\tdata =  F_\tdata \shead =  S_{k, \tdata} S_{k, \sdata}^\top \left(S_{k, \sdata} S_{k, \sdata}^\top+ \frac{|E|}{N^2} \cdot \eta\cdot |S|\cdot  I\right)^\dagger Y_\sdata = A_{k, (\tdata, \sdata)} \left(A_{k, (\sdata, \sdata)}+ \frac{|E|}{N^2} \cdot \eta\cdot |S|\cdot  I\right)^\dagger Y_\sdata.
\end{align}

\end{proof}

The following lemma shows that the prediction in the expectation graph is accuracte.
\begin{lemma}\label{lemma:ideal_prediction}
	Let  $\tilde{A} \triangleq \Exp[A]$ be the expectation of the adjacency matrix. Then, for any $\xi>0$, we have
	\begin{align}
		\tilde{A}_{k, (\tdata, \sdata)} \left(\tilde{A}_{k, (\sdata, \sdata)} +\xi I\right)^\dagger Y_\sdata = \frac{\lambda_c}{\lambda_c + \ndomain\xi} \cdot Y_\tdata,
	\end{align}
where $k=\nclass+\ndomain-1$ and $\lambda_c \triangleq \ndata\rho - \ndata\beta + \ndata(\ndomain-1)\alpha - \ndata(\ndomain-1)\gamma$. Furthermore, if we use $\tilde{\lambda}_i$ to denote the $i$-th largest eigenvalue of $\tilde{A}$, we have $\tilde{\lambda}_{k}-\tilde{\lambda}_{k+1} = \ndata\min\left\{\nclass (\beta-\gamma), \ndomain (\alpha-\gamma)\right\}$ and $\tilde{\lambda}_1 \le \ndata\nclass\ndomain\rho$.
\end{lemma}

\begin{proof}[Proof of Lemma~\ref{lemma:ideal_prediction}]

By the definition of the graph, every entry $\tilde{A}_{xx'}$ is in the set of $\{\rho, \beta, \alpha, \gamma\}$, depending on whether $x$ and $x'$ belong to the same domain/class. We can index every node $x$ as $(\dom{x}, \class{x}, \id{x})$, where $\id{x}\in[n]$ is the index of $x$ within domain $\dom{x}$ and class $\class{x}$. For any integer $i\ge1$, we use $\mathbf{1}_i$ to denote the $i$-dimensional all-one vector, and $\bar{\mathbf{1}}_i = \mathbf{1}_i / \norm{\mathbf{1}_i}$ be its normalized unit vector. We use $\Bbb{S}^i$ to denote the $i$-dimensional unit-norm sphere.

It can be verified that $\tilde{A}$ can be decomposed into the sum of several matrix Kronecker products:

\begin{align}
	\tilde{A} =\ \ \  &(\beta-\gamma) \cdot I_\ndomain \otimes (\mathbf{1}_\nclass \mathbf{1}_\nclass^\top) \otimes (\mathbf{1}_\ndata \mathbf{1}_\ndata^\top) \\
	+ &(\alpha-\gamma) \cdot  (\mathbf{1}_\ndomain \mathbf{1}_\ndomain^\top) \otimes I_\nclass \otimes (\mathbf{1}_\ndata \mathbf{1}_\ndata^\top)\\
	+ & (\rho-\beta-\alpha+\gamma) \cdot I_\ndomain \otimes I_\nclass \otimes  (\mathbf{1}_\ndata \mathbf{1}_\ndata^\top) \\
	+ & \gamma \cdot (\mathbf{1}_\ndomain \mathbf{1}_\ndomain^\top) \otimes (\mathbf{1}_\nclass \mathbf{1}_\nclass^\top) \otimes  (\mathbf{1}_\ndata \mathbf{1}_\ndata^\top).
\end{align}

As a result, $\tilde{A}$ has the following four sets of eigenvectors with non-zero eigenvalues:
\begin{itemize}
	\item{$\bar{\mathbf{1}}_\ndomain\otimes \bar{\mathbf{1}}_\nclass \otimes \bar{\mathbf{1}}_\ndata$. The corresponding eigenvalue is $\lambda_a \triangleq \ndata\rho + \ndata(\nclass-1)\beta + \ndata(\ndomain-1)\alpha + \ndata(\ndomain-1)(\nclass-1)\gamma$.}
	\item{$u\otimes \bar{\mathbf{1}}_\nclass \otimes \bar{\mathbf{1}}_\ndata$, where $u\in\Bbb{S}^\ndomain$ and $u^\top \mathbf{1}_\ndomain=0$. The corresponding eigenvalue is $\lambda_b \triangleq\ndata\rho + \ndata(\nclass-1)\beta - \ndata\alpha - \ndata(\nclass-1)\gamma$.}
	\item{$\bar{\mathbf{1}}_\ndomain \otimes v \otimes \bar{\mathbf{1}}_\ndata$, where $v\in\Bbb{S}^\nclass$ and $v^\top \mathbf{1}_\nclass=0$. The corresponding eigenvalue is $\lambda_c \triangleq\ndata\rho - \ndata\beta + \ndata(\ndomain-1)\alpha - \ndata(\ndomain-1)\gamma$.}
	\item{$u \otimes v \otimes \bar{\mathbf{1}}_\ndata$, where $u\in\Bbb{S}^\ndomain$, $v\in\Bbb{S}^\nclass$ and $u^\top \mathbf{1}_\ndomain=0$, $v^\top \mathbf{1}_\nclass=0$. The corresponding eigenvalue is $\lambda_d \triangleq\ndata\rho - \ndata\beta - \ndata\alpha + \ndata\gamma$.}
\end{itemize}

Since $\rho>\max\{\beta, \alpha\}$ and $ \min\{\beta, \alpha\} > \gamma$, we know that all of these eigenvalues are positive. When $k=\nclass+\ndomain-1$, $\tilde{A}_k$ will contain exactly the first three sets of eigenvectors since they correspond to the top-$k$ eigenvalues. This suggests that we can write $\tilde{A}_k$ as follows 
\begin{align}
	\tilde{A}_k = \lambda_a \cdot \bar{\mathbf{1}}_\ndomain\bar{\mathbf{1}}_\ndomain^\top \otimes \bar{\mathbf{1}}_\nclass \bar{\mathbf{1}}_\nclass^\top \otimes \bar{\mathbf{1}}_\ndata \bar{\mathbf{1}}_\ndata^\top
	+ \lambda_b \cdot (I_\ndomain - \bar{\mathbf{1}}_\ndomain\bar{\mathbf{1}}_\ndomain^\top)\otimes \bar{\mathbf{1}}_\nclass\bar{\mathbf{1}}_\nclass^\top \otimes \bar{\mathbf{1}}_\ndata\bar{\mathbf{1}}_\ndata^\top
	+ \lambda_c \cdot \bar{\mathbf{1}}_\ndomain\bar{\mathbf{1}}_\ndomain^\top \otimes (I_\nclass -  \bar{\mathbf{1}}_\nclass \bar{\mathbf{1}}_\nclass^\top) \otimes \bar{\mathbf{1}}_\ndata \bar{\mathbf{1}}_\ndata^\top.
\end{align}

Restricting to the source domain, we have
\begin{align}
	\tilde{A}_{k, (\sdata, \sdata)} = \frac{\lambda_a + (\ndomain-1)\lambda_b}{\ndomain} \cdot \bar{\mathbf{1}}_\nclass \bar{\mathbf{1}}_\nclass^\top \otimes \bar{\mathbf{1}}_\ndata \bar{\mathbf{1}}_\ndata^\top + \frac{\lambda_c}{\ndomain} \cdot (I_\nclass -  \bar{\mathbf{1}}_\nclass\bar{\mathbf{1}}_\nclass^\top ) \otimes\bar{\mathbf{1}}_\ndata \bar{\mathbf{1}}_\ndata^\top.
\end{align}

By the definition of pseudoinverse, we have 
\begin{align}
\left(\tilde{A}_{k, (\sdata, \sdata)}+\xi I\right)^\dagger =\left( \frac{\lambda_a + (\ndomain-1)\lambda_b}{\ndomain} + \xi\right)^{-1} \cdot \bar{\mathbf{1}}_\nclass \bar{\mathbf{1}}_\nclass^\top \otimes \bar{\mathbf{1}}_\ndata \bar{\mathbf{1}}_\ndata^\top + \left(\frac{\lambda_c}{\ndomain} + \xi\right)^{-1} \cdot (I_\nclass -  \bar{\mathbf{1}}_\nclass\bar{\mathbf{1}}_\nclass^\top ) \otimes\bar{\mathbf{1}}_\ndata \bar{\mathbf{1}}_\ndata^\top.
\end{align}
Notice that $Y_\sdata$ satisfies $\left(\bar{\mathbf{1}}_\nclass \bar{\mathbf{1}}_\nclass^\top \otimes \bar{\mathbf{1}}_\ndata \bar{\mathbf{1}}_\ndata^\top \right)Y_\sdata=0$ and $\left((I_\nclass -  \bar{\mathbf{1}}_\nclass\bar{\mathbf{1}}_\nclass^\top ) \otimes\bar{\mathbf{1}}_\ndata \bar{\mathbf{1}}_\ndata^\top\right)Y_\sdata = Y_\sdata$, we have 
\begin{align}
\left(\tilde{A}_{k, (\sdata, \sdata)}+\xi I\right)^\dagger Y_\sdata = \left(\frac{\lambda_c}{\ndomain} + \xi\right)^{-1}  Y_\sdata.
\end{align}

We can also write $\tilde{A}_{k, (\data, \sdata)}$ in the form of kronecker products as follows:
\begin{align}
\tilde{A}_{k, (\data, \sdata)} = \frac{\lambda_a}{\ndomain} \cdot \mathbf{1}_\ndomain \otimes \bar{\mathbf{1}}_\nclass \bar{\mathbf{1}}_\nclass^\top \otimes \bar{\mathbf{1}}_\ndata \bar{\mathbf{1}}_\ndata^\top  + \lambda_b\cdot (e_1 - \frac{1}{m}\mathbf{1}_\ndomain) \otimes \bar{\mathbf{1}}_\nclass \bar{\mathbf{1}}_\nclass^\top \otimes \bar{\mathbf{1}}_\ndata \bar{\mathbf{1}}_\ndata^\top + \frac{\lambda_c}{\ndomain} \cdot \mathbf{1}_\ndomain \otimes (I_\nclass - \bar{\mathbf{1}}_\nclass \bar{\mathbf{1}}_\nclass^\top)\otimes \bar{\mathbf{1}}_\ndata \bar{\mathbf{1}}_\ndata^\top .
\end{align}
Again, using the fact that $\left(\bar{\mathbf{1}}_\nclass \bar{\mathbf{1}}_\nclass^\top \otimes \bar{\mathbf{1}}_\ndata \bar{\mathbf{1}}_\ndata^\top \right)Y_\sdata=0$ and $\left((I_\nclass -  \bar{\mathbf{1}}_\nclass\bar{\mathbf{1}}_\nclass^\top ) \otimes\bar{\mathbf{1}}_\ndata \bar{\mathbf{1}}_\ndata^\top\right)Y_\sdata = Y_\sdata$, we have
\begin{align}
	\tilde{A}_{k, (\data, \sdata)} \left(\tilde{A}_{k, (\sdata, \sdata)}+\xi I\right)^\dagger Y_\sdata = \frac{\lambda_c}{\lambda_c + \ndomain\xi} \mathbf{1}_\ndomain \otimes Y_\sdata.
\end{align}

Finally, noticing that $\mathbf{1}_\ndomain \otimes Y_\sdata = Y$ finishes the proof.

\end{proof}

The followng lemma shows that when a matrix $A$ is close to $\tilde{A}$, their rank-$k$ approximations $A_k$ and $\tilde{A}_k$ are also close.

\begin{lemma}\label{lemma:low_rank_perturbation}
	Let $\tilde{A} \triangleq \Exp[A]$ be the expectation of the adjacency matrix. Let $A_k$ and $\tilde{A}_k$ be the rank-$k$ approximations of $A$ and $\tilde{A}$, respectively. Let $\tilde{\lambda}_i$ be the $i$-th largest eigenvalue of $\tilde{A}$, $\norm{\cdot}$ be the operator norm of a matrix or $\ell_2$-norm of a vector. Then, when $\norm{A-\tilde{A}}< \tilde{\lambda}_k - \tilde{\lambda}_{k+1}$, we have
	\begin{align}
	\norm{A_k - \tilde{A}_k} \le \left(1+\frac{2\norm{A-\tilde{A}} + 2\norm{\tilde{A}}}{(\tilde{\lambda}_k - \tilde{\lambda}_{k+1}) - \norm{A-\tilde{A}}}\right) \cdot \norm{A-\tilde{A}}.
	\end{align}
\end{lemma}

\begin{proof}[Proof of Lemma~\ref{lemma:low_rank_perturbation}]
	Let the SVD decomposition of $A$ and $\tilde{A}$ be as follows:	
	\begin{align}
		A = \begin{bmatrix} 
		U_1 & U_2
		\end{bmatrix} 
	    \begin{bmatrix}
	    	\Sigma_1 & 0\\ 0 & \Sigma_2
	    \end{bmatrix}
		\begin{bmatrix}
			U_1^\top \\ U_2^\top
		\end{bmatrix},		
	\end{align}
	\begin{align}
	\tilde{A} = \begin{bmatrix} 
		\tilde{U}_1 & \tilde{U}_2
	\end{bmatrix} 
	\begin{bmatrix}
		\tilde{\Sigma}_1 & 0 \\ 0 & \tilde{\Sigma}_2
	\end{bmatrix}
	\begin{bmatrix}
		\tilde{U}_1^\top \\ \tilde{U}_2^\top
	\end{bmatrix},		
\end{align}
where $\Sigma_1$ and $\tilde{\Sigma}_1$ contain the top $k$ eigenvalues of $A$ and $\tilde{A}$, respectively. By the definition of rank-$k$ approximation, we have $A_k = U_1 \Sigma_1 U_1^\top$ and $\tilde{A}_k = \tilde{U}_1 \tilde{\Sigma}_1 \tilde{U}_1^\top$. Therefore,

\begin{align}
	&\norm{A_k-\tilde{A}_k}  = \norm{U_1 \Sigma_1 U_1^\top - \tilde{U}_1 \tilde{\Sigma}_1 \tilde{U}_1^\top} 
	= \max_{v\in\Real^\Ndata: \norm{v}=1} \norm{\left(U_1 \Sigma_1 U_1^\top - \tilde{U}_1 \tilde{\Sigma}_1 \tilde{U}_1^\top\right)v}\\
	\le& \underbrace{\max_{v\in\Real^\Ndata: \norm{v}=1, v^\top\tilde{U}_2=0} \norm{\left(U_1 \Sigma_1 U_1^\top - \tilde{U}_1 \tilde{\Sigma}_1 \tilde{U}_1^\top\right)v}}_{C_1} + \underbrace{\max_{v\in\Real^\Ndata: \norm{v}=1, v^\top\tilde{U}_1=0} \norm{\left(U_1 \Sigma_1 U_1^\top - \tilde{U}_1 \tilde{\Sigma}_1 \tilde{U}_1^\top\right)v}}_{C_2}.
\end{align}

Let $\delta \triangleq \min\{\lambda_k - \tilde{\lambda}_{k+1}, \tilde{\lambda}_k - \lambda_{k+1}\}$. According to Weyl's inequality, we have 
\begin{align}
\norm{	    \begin{bmatrix}
		\Sigma_1 & 0\\ 0 & \Sigma_2
	\end{bmatrix}
 - 	\begin{bmatrix}
 	\tilde{\Sigma}_1 & 0 \\ 0 & \tilde{\Sigma}_2
 \end{bmatrix}
} \le \norm{A-\tilde{A}}.
\end{align}
So we have $\delta \ge\left( \tilde{\lambda}_k - \tilde{\lambda}_{k+1}\right) - \norm{A-\tilde{A}} > 0$. Now we bound $C_1$ and $C_2$ separately. To bound $C_1$, we have
\begin{align}
	C_1 \le& \norm{A-\tilde{A}} + \max_{v\in\Real^\Ndata: \norm{v}=1, v^\top\tilde{U}_2=0} \norm{\left(U_2 \Sigma_2 U_2^\top - \tilde{U}_2 \tilde{\Sigma}_2 \tilde{U}_2^\top\right)v}\\
    \le &  \norm{A-\tilde{A}} + \norm{U_2} \cdot \norm{\Sigma_2} \cdot \norm{U_2^\top \tilde{U}_1}\\
    \le& \left(1+\frac{\norm{\Sigma_2}}{\delta}\right) \norm{A-\tilde{A}}
\end{align}
where the first line follows the triangle inequality, the second line uses $v^\top \tilde{U}_2=0$, and the third line uses Davis-Kahan theorem~\citep{davis1970rotation}. Similarly, to bound $C_2$, we have
\begin{align}
	C_2 &= \max_{v\in\Real^\Ndata: \norm{v}=1, v^\top \tilde{U}_1=0} \norm{U_1\Sigma_1 U_1^\top v}_2\\
	&\le \norm{U_1}\cdot \norm{\Sigma_2} \cdot \norm{U_1^\top \tilde{U}_2} \le \frac{\norm{\Sigma_1}\cdot \norm{A-\tilde{A}}}{\delta}.
\end{align}
Adding $C_1$ and $C_2$ together we have
\begin{align}
	\norm{A_k -\tilde{A}_k} \le& \left(1+\frac{\norm{\Sigma_1} + \norm{\Sigma_2}}{\delta}\right)\cdot \norm{A-\tilde{A}}\\
	\le& \left(1+\frac{2\norm{A-\tilde{A}} + 2\norm{\tilde{A}}}{(\tilde{\lambda}_k - \tilde{\lambda}_{k+1}) - \norm{A-\tilde{A}}}\right) \cdot \norm{A-\tilde{A}},
\end{align}
where the second line again uses Weyl's inequality.
\end{proof}

To bound the differene between $A$ and $\tilde{A}$, we use the following lemma which is adapted from Theorem 5.2 of~\citep{lei2015consistency}.
\begin{lemma}\label{lemma:perturbation}
	Let $A$ be the adjacency matrix of a random graph on $\Ndata$ nodes in which edges occur
	independently. Let $E[A] = \tilde{A} = (\tilde{A}_{ij})_{i,j=1,...,\Ndata}$ be the expectation adjacency matrix and assume that $\Ndata \max_{ij} \tilde{A}_{ij} \ge \log \Ndata$. Then, for any $\zeta>0$ there exists a constant $C = C(\zeta)$
	such that
	\begin{align}
		\norm{A-P}\le C\sqrt{\Ndata}
	\end{align}
	with probability at least $1-\Ndata^{-2\zeta}$.	
\end{lemma}

Now we prove Theorem~\ref{theorem:sbm} using the above lemmas.

\begin{proof} [Proof of Theorem~\ref{theorem:sbm}]

By Lemma~\ref{lemma:closed_form_expression}, we have that the prediction on the target domain is
\begin{align}
	\pred_\tdata =  A_{k, (\tdata, \sdata)} \left(A_{k, (\sdata, \sdata)}+ \frac{|E|}{N^2} \cdot \eta\cdot |S|\cdot  I\right)^\dagger Y_\sdata.
\end{align}

Let $|\tilde{E}| \triangleq \mathbf{1}_\Ndata^\top \tilde{A} \mathbf{1}_\Ndata $ be the expectation of number of edges in the graph. We define the ideal prediction as 
\begin{align}
	\widetilde{\pred}_\tdata = \tilde{A}_{k, (\tdata, \sdata)} \left(\tilde{A}_{k, (\sdata, \sdata)}+ \frac{|\tilde{E}|}{N^2} \cdot \eta\cdot |S|\cdot  I\right)^\dagger Y_\sdata.
\end{align}

We will bound the difference between $\pred_\tdata$ and $\tilde{\pred}_\tdata$. For every class $c\in [\nclass]$, define the following error vector
\begin{align}
	e^c \triangleq A_{k, (\tdata, \sdata)} \left( A_{k, (\sdata, \sdata)} + \frac{|E|}{N^2} \cdot \eta\cdot |S| \cdot I\right)^\dagger Y^c_\sdata - \tilde{A}_{k, (\tdata, \sdata)} \left( \tilde{A}_{k, (\sdata, \sdata)} + \frac{|E|}{N^2} \cdot \eta \cdot |S| \cdot I\right)^\dagger Y^c_\sdata,
\end{align}
where $Y^c_\sdata$ is the $c$-th column of $Y_\sdata$.

Using the perturbation bound for the pseudoinverse matrix ~\citep{stewart1977perturbation}, we have
\begin{align}
	\norm{e^c} \le& \norm{A_{k, (\tdata, \sdata)} - \tilde{A}_{k, (\tdata, \sdata)}} \cdot \norm{\left( \tilde{A}_{k, (\sdata, \sdata)} + \frac{|E|}{N^2} \cdot \eta\cdot |S| \cdot I\right)^\dagger} \cdot \norm{Y_\sdata^c} \\
	&+  \norm{A_{k, (\tdata, \sdata)}} \cdot \norm{\left( A_{k, (\sdata, \sdata)} + \frac{|E|}{N^2} \cdot \eta\cdot |S| \cdot I\right)^\dagger - \left( \tilde{A}_{k, (\sdata, \sdata)} + \frac{|E|}{N^2} \cdot \eta\cdot |S| \cdot I\right)^\dagger } \cdot \norm{Y_\sdata^c} \\
	\le & \norm{A_k - \tilde{A}_k} \cdot \frac{N^2}{\eta|E|\cdot |S|} \cdot \norm{Y_\sdata^c} + \norm{A_k}\cdot \frac{1+\sqrt{5}}{2} \cdot  \left(\frac{N^2}{\eta|E|\cdot |S|} \right)^2\cdot \norm{A_k - \tilde{A}_k}\cdot \norm{Y_\sdata^c}. \label{equation:error_bound_one_class}
\end{align}

By lemma~\ref{lemma:perturbation}, there exists constant $C=C(\zeta)$ such that $\norm{A-\tilde{A}} \le C \sqrt{\Ndata}$ with probability at least $1-\Ndata^{-2\zeta}$ for any $N>\Omega(1/\rho)$. From now on, we assume this high probability event happens. Let $\Delta \triangleq \min\{\nclass (\beta-\gamma), \ndomain (\alpha-\gamma)\}$. According to Lemma~\ref{lemma:ideal_prediction}, we know that $\tilde{\lambda}_{k}-\tilde{\lambda}_{k+1} = \ndata\Delta$. If our choice of $\Ndata$ further satisfies $N\ge \left(\frac{2\nclass\ndomain C}{\Delta}\right)^2$, we have $\norm{A-\tilde{A}} \le \frac{1}{2} (\tilde{\lambda}_{k}-\tilde{\lambda}_{k+1})$, so from Lemma~\ref{lemma:low_rank_perturbation} we have 
\begin{align}
\norm{A_k - \tilde{A}_k} \le O\left(\frac{\tilde{\lambda}_1}{\tilde{\lambda}_{k}-\tilde{\lambda}_{k+1}}\cdot \norm{A-\tilde{A}}\right) \le O\left(\frac{\tilde{\lambda}_1}{\tilde{\lambda}_{k}-\tilde{\lambda}_{k+1}}\cdot \sqrt{\Ndata}\right).  
\end{align}

According to Lemma~\ref{lemma:ideal_prediction}, we know that $\tilde{\lambda}_1 \le \ndata\nclass\ndomain\rho$, so we have 
\begin{align}\label{equation:rank_k_error}
	\norm{A_k - \tilde{A}_k} \le O\left(\frac{\nclass\ndomain\rho\sqrt{\Ndata}}{\Delta}\right).
\end{align}

By Hoeffding's inequality, with probability at least $1-2e^{-2\Ndata^2}$ we have $||E| - |\tilde{E}||\le \Ndata$. From now on, we assume this high-probability event happens. The total failure probability so far is $\Ndata^{-2\zeta}+2e^{-2\Ndata^2} \le N^{-\zeta}$. By the definition of graph, we have $|\tilde{E}|\ge \frac{\rho\Ndata^2}{\nclass\ndomain}$. If our choice of $\Ndata$ further satisfies $N\ge \frac{2\nclass\ndomain}{\rho}$, we have $|E|\ge \frac{\rho\Ndata^2}{2\nclass\ndomain}$, hence
\begin{align}\label{equation:bound_edge}
	\frac{|E|\cdot |S|}{\Ndata^2} \ge \frac{\rho\Ndata}{2\nclass\ndomain^2} .
\end{align}

Substituting \eqref{equation:rank_k_error} and \eqref{equation:bound_edge} into \eqref{equation:error_bound_one_class} gives:
\begin{align}
	\norm{e^c}\le  O\left(\frac{1}{\Delta\eta^2\sqrt{\Ndata}}\right)\cdot\textup{poly}(\nclass, \ndomain) \cdot \norm{Y_\sdata^c}.
\end{align}
Summing over all classes $c$ and noticing that $\norm{Y_\sdata^c} \le \sqrt{\Ndata}$ leads to 
\begin{align}\label{equation:final_error_1}
	\norm{e}_F \le O\left(\frac{1}{\Delta\eta^2} \right)\cdot\textup{poly}(\nclass, \ndomain)
\end{align}

Let $\xi = \frac{|E|}{N^2} \cdot \eta\cdot |S|$ and $\tilde{\xi} = \frac{|\tilde{E}|}{N^2} \cdot \eta\cdot |S|$. We have
\begin{align}
	\left\vert\frac{\lambda_c}{\lambda_c+\ndomain\xi} - \frac{\lambda_c}{\lambda_c+\ndomain\tilde{\xi}}\right\vert \le \frac{\ndomain}{\lambda_c} \cdot |\xi-\tilde{\xi}| \le \frac{\eta}{\Ndata} \cdot\textup{poly}(\nclass, \ndomain).
\end{align}

From Lemma~\ref{lemma:ideal_prediction} we have 
\begin{align}
	&\norm{\tilde{A}_{k, (\tdata, \sdata)} \left(\tilde{A}_{k, (\sdata, \sdata)}+ \frac{|{E}|}{N^2} \cdot \eta\cdot |S|\cdot  I\right)^\dagger Y_\sdata  - \tilde{A}_{k, (\tdata, \sdata)} \left(\tilde{A}_{k, (\sdata, \sdata)}+ \frac{|\tilde{E}|}{N^2} \cdot \eta\cdot |S|\cdot  I\right)^\dagger Y_\sdata}_F \\
	=&\left\vert\frac{\lambda_c}{\lambda_c+\ndomain\xi} - \frac{\lambda_c}{\lambda_c+\ndomain\tilde{\xi}}\right\vert \cdot \norm{Y_\tdata}_F \le \frac{1}{\sqrt{\Ndata}} \cdot \textup{poly}(\nclass, \ndomain).\label{equation:final_error_2}
\end{align}

Combining \eqref{equation:final_error_1} and \eqref{equation:final_error_2} we have
\begin{align}
	\norm{\pred_\tdata - \widetilde{\pred}_\tdata}_F \le O\left(\frac{1}{\eta^2 \cdot \min\{\alpha-\gamma, \beta-\gamma\}}\right)\cdot \textup{poly}(\nclass, \ndomain).
\end{align}

Notice that the $x$-th row of $\widetilde{\pred}_\tdata$ is $\frac{\lambda_c}{\lambda_c+\ndomain\tilde{\xi}} \cdot y_\class{x}$. Since $\lambda_c = n \left(\rho - \beta+(\ndomain-1)\alpha - (\ndomain-1)\gamma\right)\ge \frac{1}{2}\ndata\ndomain( \alpha-\gamma)$, and $\tilde{\xi} = \frac{|\tilde{E}|}{N^2} \cdot \eta\cdot |S| = \frac{\ndata^2\nclass\ndomain \rho +\ndata^2 \ndomain(\nclass^2-\nclass)\beta+\ndata^2\nclass(\ndomain^2-\ndomain)\alpha+\ndata^2(\nclass^2-\nclass)(\ndomain^2-\ndomain)\gamma}{n^2\ndomain^2\nclass^2} \cdot \eta\cdot \nclass\ndata \le \eta \nclass\ndata\rho$, we have 
\begin{align}
	\frac{\lambda_c}{\lambda_c+\ndomain\tilde{\xi}} \ge \frac{1}{1+\frac{2\nclass\rho\eta}{\ndomain(\alpha-\gamma)}}\ge 1-\epsilon,
\end{align}
where the second inequality follows our assumption on $\eta$.

Since $\pred_\tdata$ is incorrect on the $x$-th row only if its difference from the $x$-th row of $\widetilde{\pred}_\tdata$ has larger norm than $\Omega(1-\epsilon)$, we know the final total error on the target domains is bounded by $O\left(\frac{1}{\eta^4\Ndata \cdot \min\{\alpha-\gamma, \beta-\gamma\}^2}\right)\cdot \textup{poly}(\nclass, \ndomain)$.

Collecting all the requirements of $\Ndata$, this bound holds so long as $N\ge \Omega\left( \left(\frac{\nclass\ndomain}{\min\{\alpha-\gamma, \beta-\gamma\}}\right)^2\right)$, which is equivalent to $\ndata\ge\Omega\left(\frac{\nclass\ndomain}{\min\{\alpha-\gamma, \beta-\gamma\}^2}\right)$.

\end{proof}

\begin{proof}[Proof of Proposition~\ref{proposition:domain_separation}]
	Notice that the roles of domain and class are the same in stochastic block model. Let $\zeta>0$ and $\epsilon\in\left(0, \frac{1}{2}\right)$ be arbitrary constants. If we train a domain classifier head on class 1 with regularization strength $\eta\in\left( 0, \frac{(\beta-\gamma)\epsilon}{2r\rho} \right]$, then following the same proof of Theorem~\ref{theorem:sbm}, we know that when  $\ndata\ge\Omega\left(\frac{\nclass\ndomain}{\min\{\alpha-\gamma, \beta-\gamma\}^2}\right)$, with probability at least $1-n^{-\zeta}$, we have that the global error of this domain predictor is bounded by $O\left(\frac{1}{\eta^4 \cdot \min\{\alpha-\gamma, \beta-\gamma\}^2 \cdot \ndata}\right)\cdot \textup{poly}(\nclass, \ndomain)$. Setting $\epsilon=\frac{1}{4}$, $\zeta=1$ and $\eta = \frac{\beta-\gamma}{8r\rho}$ finishes the proof.
\end{proof}

%
%
%
%

\subsection{Proof of Proposition~\ref{prop:separation}}
\label{app:separation-proof}

\newcommand{\reachableset}{\sR}
\newcommand{\zpos}{z_{\posclass}}
\newcommand{\zneg}{z_{\negclass}}
\newcommand{\aaa}{a}
\newcommand{\bbb}{b}
\newcommand{\uone}{u_1}
\newcommand{\utwo}{u_2}

\begin{figure}[th]
    \centering
	\includegraphics[width=0.4\linewidth]{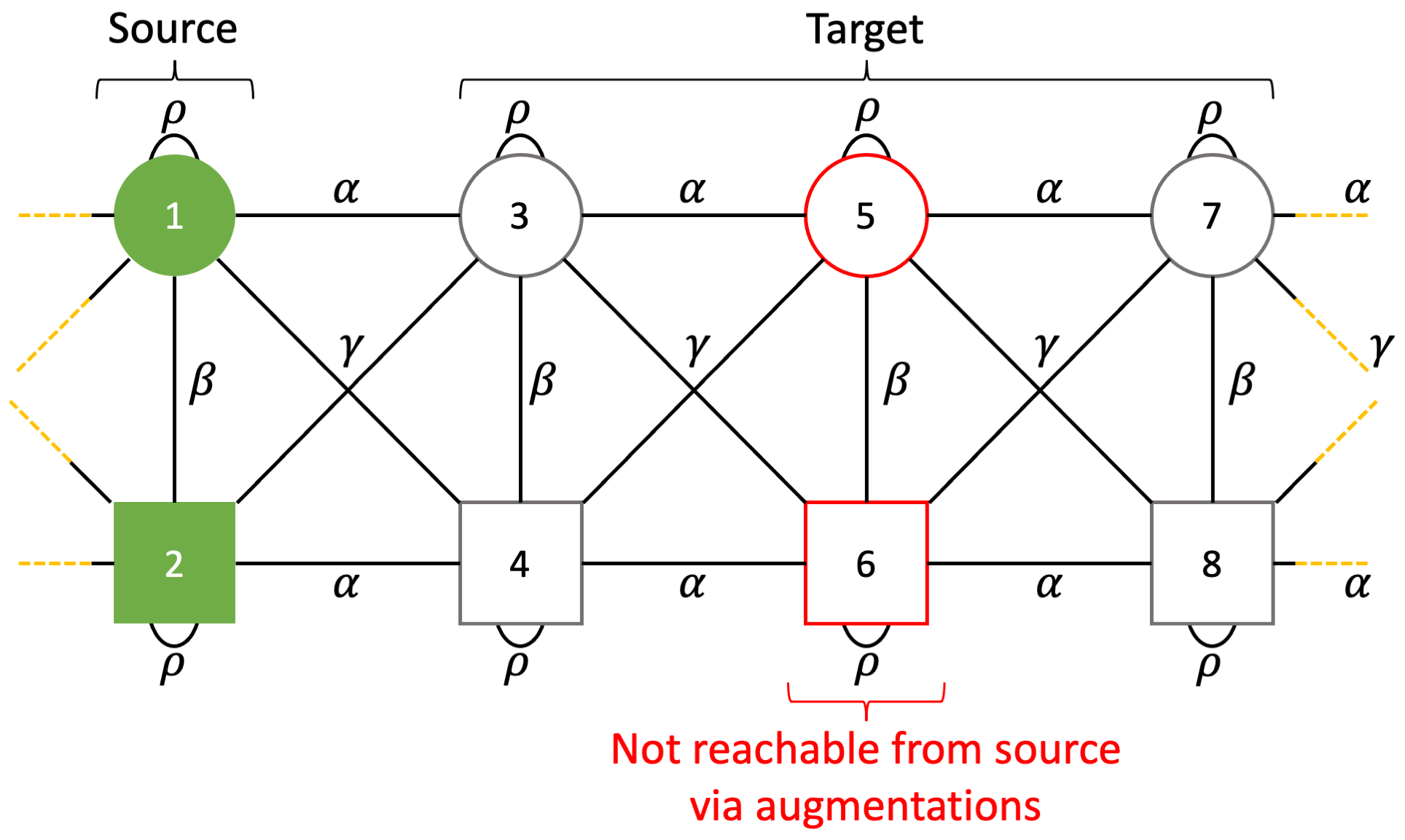}
	\captionof{figure}{%
    Example distribution of data and augmentations in which ERM and DANN generalize more poorly than contrastive learning.
    ERM and DANN minimizers are only guaranteed to correctly label inputs that are either in the source domain or reachable from source inputs via data augmentation, while connectivity throughout the target domain allows contrastive pre-training to generalize even to target inputs that are multiple ``hops'' away from the source.
The orange dotted lines on the far left and far right connect to each other but are illustrated as straight dotted lines for cleanliness.}
    \label{fig:separation}
\end{figure}

\paragraph{Data distribution (Figure~\ref{fig:separation}).}
In the setting of binary classification with 2 domains, let $\sS = \{ \inputx \in \inputspace: \domainx = 1 \}$ and $\sT = \{ \inputx \in \sT: \domainx = 2\}$ and assume that $\sourcedist$ and $\targetdist$ are uniform over $\sS$ and $\sT$, respectively.
We refer to each input by its ID, so that $\inputspace = \{1, \dots, 8\}$.
The source domain $\sS=\{1,2\}$ contains 2 points while the target domain $\sT=\{3,4,5,6,7,8\}$ contains 6 points.
The label space is $\labelspace = \{\negclass, \posclass\}$. The label for $\inputx\in \{1,3,5,7\}$ is $\labelx=\posclass$ and the label for $\inputx\in \{2,4,6,8\}$ is $\labelx=\negclass$.
The marginal distribution over unlabeled examples $\unlabeldist$ (used by contrastive pre-training in Equation~\ref{eq:scl-setup}) is the uniform distribution over $\inputspace$.
Only the inputs in the source domain $\sS$ have labels for supervised learning.

\paragraph{Augmentation distribution.}
The augmentation distribution $\aug(\cdot \mid \inputx)$ is
\begin{align}
        \aug(\inputxp \mid \inputx) =
    \begin{cases}
        \augprobr & \inputx = \inputxp\\
        \augproba & \labelxp = \labelx, \inputx \neq \inputxp \\
        \augprobb & \{\inputxp, \inputx\} \in \{ \{1, 2\}, \{3, 4\}, \{5, 6\}, \{7, 8\} \} \\ 
        \augprobg & \{\inputxp, \inputx\} \in \{ \{1, 4\}, \{2, 3\}, \{3, 6\}, \{4, 5\}, \{5, 8\}, \{6, 7\}, \{1, 8\}, \{2, 7\} \} \\  
    \end{cases}. 
\end{align}
By the edge structure of the graph, we have that $\augprobr + 2\augproba + \augprobb + 2\augprobg = 1$.
We assume that $\augprobr, \augproba, \augprobb$, and $\augprobg$ are all distinct and obey the following inequalities: $\augprobr > \max\{\augproba, \augprobb\}$ and $\min\{\augproba, \augprobb\} > \augprobg$.

\paragraph{Distribution over positive pairs.}
The augmentation distribution induces a distribution $\pospairdist$ over positive pairs, which is defined similarly:
\begin{align}
    \label{eq:pos-pair-dist}
    \pospairdist(\inputxp, \inputx) =
    \begin{cases}
        \pairprobr/\pairnorm & \inputx = \inputxp\\
        \pairproba/\pairnorm & \labelxp=\labelx, \inputx \neq \inputxp\\
        \pairprobb/\pairnorm & \{\inputxp, \inputx\} \in \{ \{1, 2\}, \{3, 4\}, \{5, 6\}, \{7, 8\} \} \\
        \pairprobg/\pairnorm & \{\inputxp, \inputx\} \in \{ \{1, 4\}, \{2, 3\}, \{3, 6\}, \{4, 5\}, \{5, 8\}, \{6, 7\}, \{1, 8\}, \{2, 7\} \}
    \end{cases}
\end{align}
where the normalization is $\pairnorm=8\pairprobr + 16 \pairproba + 8 \pairprobb + 16 \pairprobg$. The normalization ensures that the sum of $\pospairdist(\inputxp, \inputx)$ over all pairs $\inputxp,\inputx$ is 1.
At the end of this section (paragraph~\ref{app:aug-conversion}), we show how to derive $\pairprobr, \pairproba, \pairprobb, \pairprobg$ as functions of $\augprobr, \augproba, \augprobb, \augprobg$, respectively.
We also show that the assumptions that $\augprobr > \max\{\augproba, \augprobb\}$ and $\min\{\augproba, \augprobb\} > \augprobg$ imply that $\pairprobr > \max\{\pairproba, \pairprobb\}$ and $\min\{\pairproba, \pairprobb\} > \pairprobg$.

We recall Proposition~\ref{prop:separation}:

\separationprop*

We will consider $\ell$ to be the squared loss for all methods in this proof.

\paragraph{ERM.}
We claim that a classifier $\empclferm$ that makes the following predictions is a minimizer of the ERM objective with data augmentation (Equation~\ref{eq:erm-object-setup}):
\begin{align}
    \argmax_{i} (\empclferm(\inputx))_i = \begin{cases}
        \posclass &\text{if}\quad \inputx \in \{1, 3, 6, 7\} \\
        \negclass &\text{if}\quad \inputx \in \{2, 4, 5, 8\}.
    \end{cases}
\end{align}
Since this classifier outputs the incorrect prediction on 2/6 of the target points ($\{5,6\}$), it has a target 0-1 error of 1/3. It remains to show that $\empclferm$ is a minimizer of the ERM objective.

To see why, we first define the reachable set $\reachableset=\{1,2,3,4,7,8\}$ as the inputs that are sampled with probability greater than 0 as augmentations of the labeled source data, and rewrite Equation~\ref{eq:erm-object-setup} as
\begin{align}
\Lerm(\clf) &= \E_{\inputx \sim \sourcedist, \inputxp \sim \aug(\cdot \mid \inputx)}[ \|\clf(\inputxp) - e_{\labelx} \|^2 ]\\
            &= \sum_{\inputx \in \{1,2\}} \sourcedist(\inputx) \sum_{\inputxp \in  \inputspace} \aug(\inputxp \mid \inputx)[ ( \clf(\inputxp) - e_{\labelx} )^2 ]\\
            &= \sum_{\inputx \in \{1,2\}} \sourcedist(\inputx) \sum_{\inputxp \in  \reachableset} \aug(\inputxp \mid \inputx)[ ( \clf(\inputxp) - e_{\labelx} )^2 ]
\end{align}
where $e_{\labelx} \in \R^\numcls$ is a one-hot vector for the label.
In the last step, we replace the sum over the entire input space $\inputspace$ with the reachable set $\reachableset$ since $\aug(\inputxp \mid \inputx)=0$ when $\inputxp$ is not in the reachable set. Since $\inputx \in \{5,6\}$ are not included in this sum, the prediction of $\clf$ on $\{5,6\}$ can be arbitrary (and wrong). This shows that there exists a minimizer of the ERM objective with 0-1 error at least 1/3, since the minimizer can classify 2/6 target inputs incorrectly.

However, all minimizers of ERM output the correct prediction for the other inputs. From above, we have
\begin{align}
    \Lerm(\clf) &= \sourcedist(\inputx = 1) \left(\sum_{\inputx' \in \reachableset} \aug(\inputx' \mid \inputx = 1) \|\clf(\inputx') - e_{\posclass}\|^2 \right) + \sourcedist(\inputx = 2) \left(\sum_{\inputx' \in \reachableset} \aug(\inputx' \mid \inputx = 2) \|\clf(\inputx') - e_{\negclass}\|^2 \right) \\
                 &= \frac{1}{2} \sum_{\inputx' \in \reachableset} \aug(\inputx' \mid \inputx = 1) \cdot \|\clf(\inputx') - e_{\posclass}\|^2 + \aug(\inputx' \mid \inputx = 2) \cdot \|\clf(\inputx') - e_{\negclass}\|^2
\end{align}
since $\sourcedist(\inputx = 1)= \sourcedist(\inputx = 2) = 1/2$.
For each $\inputxp \in \reachableset$, the minimizer of this objective outputs
\begin{align}
    \frac{\aug(\inputx'\mid \inputx = 1) e_{\posclass} + \aug(\inputx' \mid \inputx = 2) e_{\negclass}}{\aug(\inputx'\mid \inputx = 1) + \aug(\inputx' \mid \inputx = 2)}.
\end{align}
The argmax is $\posclass$ if $\aug(\inputxp \mid \inputx = 1) > \aug(\inputxp \mid \inputx=2)$.
For $\inputxp = 1$, we have that $\aug(\inputxp \mid \inputx = 1) = \augprobr$, which by assumption is larger than $\aug(\inputxp \mid \inputx=2)=\augprobb$.
For $\inputxp \in \{3,7\}$, we have that $\aug(\inputxp \mid \inputx = 1) = \augproba$, which by assumption is larger than $\aug(\inputxp \mid \inputx=2)=\augprobg$.
Therefore any minimizer of the ERM objective will output a vector with argmax $\posclass$ for $\inputx \in \{1,3,7\}$. By symmetry, any minimizer of the ERM objective will output a vector with argmax $\negclass$ for $\inputx \in \{2,4,8\}$. Thus any minimizer of the ERM objective will correctly classify these 6 inputs (including 4/6 of the target inputs). This shows that $\empclferm$ is a minimizer of the ERM objective and has target 0-1 error 1/3.

\paragraph{DANN.}
Let $\zpos, \zneg \in \R^{\embeddim}$ be distinct points in the representation space such that $\zpos \neq \zneg$.
We claim that the following encoder $\empencoderdann$ and classification head  $\empheaddann$ (with the following constraint) minimize the DANN objective (Equation~\ref{eq:dann-object-setup}):
\begin{align}
    \empencoderdann(\inputx) = \begin{cases}
        \zpos &\text{if}\quad \inputx \in \{ 1, 3, 6, 7 \} \\
        \zneg &\text{if}\quad \inputx \in \{ 2, 4, 5, 8 \}
    \end{cases}
    \qquad \text{and} \qquad
    \argmax_i {\empheaddann(z)_i} = \begin{cases}
        \posclass &\text{if}\quad z = \zpos \\
        \negclass &\text{if}\quad z = \zneg
    \end{cases}.
\end{align}
The first term in the DANN objective is equivalent to the ERM objective, and the classifier $\empclfdann: \empheaddann \circ \empencoderdann$ outputs the same predictions as the ERM minimizer $\empclferm$.
Therefore, $\empclfdann$ minimizes the first term of Equation~\ref{eq:dann-object-setup} by the same argument from the ERM section, and also gets a target 0-1 error of 1/3.

It remains to show that $\empencoderdann$ is an optimal solution for the second term of Equation~\ref{eq:dann-object-setup},
\begin{align}
    \lambda \max_{\encoder} \min_{\dhead} \E_{\inputx\sim{\unlabeldist}, \inputxp \sim \aug(\cdot \mid \inputx)} [\|\dhead(\encoder(\inputxp)) - e_{\domainx}\|^2].
\end{align}
We write this as a maximization for simplicity, while the DANN objective is a minimization over the negation.

First, we show that this max-min loss is upper bounded by $\frac{3\lambda}{8}$.
For \textit{any} encoder $\encoder$, we could choose $\dhead$ such that it always outputs the vector $\frac{1}{4}e_1 + \frac{3}{4}e_2$ for all input representations.
Since $1/4$ of the inputs sampled from $\unlabeldist$ are from the source domain, on these inputs the domain classifier would have squared loss
\begin{align}
    \frac{\lambda}{4} \| \frac{3}{4}(e_2 - e_1)\|^2 = \frac{9\lambda}{32}.
\end{align}
For inputs originally from the target domain, the domain classifier incurs squared loss
\begin{align}
    \frac{3\lambda}{4}\|\frac{1}{4}(e_1=e_2)\|^2 = \frac{3\lambda}{32}.
\end{align}
In total, the squared loss is $\frac{3\lambda}{8}$. This is an upper bound since any encoder will incur at most this loss.

Second, we show that the proposed $\empencoderdann$ attains the upper bound. Fixing the encoder to be $\empencoderdann$, we compute the objective and maximize over the domain classification head $\dhead$:
\begin{align}
    \lambda \min_{\dhead} \E_{\inputx\sim{\unlabeldist}, \inputxp \sim \aug(\cdot \mid \inputx)}& [\|\dhead(\empencoderdann(\inputxp)) - e_{\domainx}\|^2] \\
                                                                                                &=\lambda\min_{\dhead}\sum_{\inputx \in \inputspace} \unlabeldist(\inputx) \sum_{\inputxp \in \inputspace} \aug(\inputxp\mid \inputx) \|\dhead(\empencoderdann(\inputxp)) - e_{\domainx}\|^2\\
                                                                                                &=\min_{\dhead} \frac{\lambda}{8}\sum_{\inputx \in \inputspace} \sum_{\inputxp \in \inputspace} \aug(\inputxp\mid \inputx) \|\dhead(\empencoderdann(\inputxp)) - e_{\domainx}\|^2 ~~~~~ (\text{since }\unlabeldist(\inputx)=\frac{1}{8}).
\end{align}
We can further break this down into a sum of 4 terms, which simplify the loss terms:
\begin{equation}
\begin{split}
    \min_{\dhead} &\frac{\lambda}{8}\sum_{\inputx \in \sourcedata} \sum_{\inputx \in \{1,3,6,7\}} \aug(\inputxp \mid \inputx) \|\dhead(\zpos) - e_1\|^2 \\
    +&\frac{\lambda}{8}\sum_{\inputx \in \sourcedata} \sum_{\inputx \in \{2,4,5,8\}} \aug(\inputxp \mid \inputx) \|\dhead(\zneg) - e_1\|^2 \\
    +&\frac{\lambda}{8}\sum_{\inputx \in \targetdata} \sum_{\inputx \in \{1,3,6,7\}} \aug(\inputxp \mid \inputx) \|\dhead(\zpos) - e_2\|^2 \\
    +&\frac{\lambda}{8}\sum_{\inputx \in \targetdata} \sum_{\inputx \in \{2,4,5,8\}} \aug(\inputxp \mid \inputx) \|\dhead(\zneg) - e_2\|^2.
\end{split}
\end{equation}
Using the augmentation graph and the fact that $\augprobr + 2\augproba + \augprobb + 2\augprobg = 1$, we can calculate the sum of the augmentation probabilities for each term, resulting in the objective
\begin{equation}
\begin{split}
    \min_{\dhead}&\frac{\lambda}{8} \|\dhead(\zpos) - e_1\|^2 \\
    +&\frac{\lambda}{8} \|\dhead(\zneg) - e_1\|^2 \\
    +&\frac{3\lambda}{8} \|\dhead(\zpos) - e_2\|^2 \\
    +&\frac{3\lambda}{8} \|\dhead(\zneg) - e_2\|^2.
\end{split}
\end{equation}
The minimizer of this objective (taking the gradient and setting to 0) always outputs
\begin{align}
    \dhead(\zpos) = \dhead(\zneg) = \frac{1}{4}e_1 + \frac{3}{4}e_2
\end{align}
for any input representation.
This attains the error $\frac{3\lambda}{8}$, showing that $\empencoderdann$ is optimal for the second term of the DANN loss.
Therefore, the proposed $\empencoderdann$ and $\empheaddann$ minimize the DANN objective and get a target 0-1 error of 1/3.

\paragraph{Contrastive pre-training.}
\newcommand{\normalizer}{D}
We show that contrastive pre-training achieves 0 target error on this example as long as $\pairproba > \pairprobg + \pairprobb$.
Examining Figure~\ref{fig:separation}, $\pairproba$ intuitively corresponds to the probability that an image $x$ augments to a different image $x'$ of the same class, while $\pairprobg$ and $\pairprobb$ correspond to the probability that $x$ augments to $x'$ of a different class.
Roughly speaking, the condition $\pairproba > \pairprobg + \pairprobb$ means that augmentations should preserve the class more often than not.

We prove that contrastive pre-training achieves 0 target error by first deriving the representations learned by contrastive learning on unlabeled data $\unlabeldist$ (Equation~\ref{eq:scl}), and then deriving the linear probe learned by optimizing the squared loss on source examples $\sS$ (Equation~\ref{eq:sq-loss}) and examining its accuracy on target examples $\sT$.

\textbf{Step 1 (Pretraining)}:
During the pre-training phase, we use inputs sampled from $\unlabeldist$ to learn an encoder $\empencoder$ which minimizes the spectral contrastive loss (Equation~\ref{eq:scl}).
We define the \textit{embedding} matrix $\embedding \in \R^{8 \times k}$ as follows: the $i$-th row of $\embedding$ is the feature vector for example $i$---that is, $\embedding_i = \empencoder(i)$ where $\embedding_i \in \R^k$ denotes the $i$-th row of $\embedding$.

In this step of the proof we will compute $\embedding$ analytically.
$\embedding$ is given by the top $k$ eigenvectors and eigenvalues of the adjacency matrix $A$, following the analysis in~\citet{haochen2021spectral}.
We begin by computing these eigenvectors and eigenvalues.

Let $A$ be the adjacency matrix, where $A_{ij} = \pospairdist(i, j)$ is the probability of sampling a positive pair $(i, j)$. Writing $A$ out explicitly,

\begin{equation}
    A = \frac{1}{\pairnorm}
    \begin{bmatrix*}[r]
        \pairprobr & \pairprobb & \pairproba & \pairprobg & 0 & 0 & \pairproba & \pairprobg \\
        \pairprobb & \pairprobr & \pairprobg & \pairproba & 0 & 0 & \pairprobg & \pairproba \\
        \pairproba & \pairprobg & \pairprobr & \pairprobb & \pairproba & \pairprobg & 0 & 0 \\
        \pairprobg & \pairproba & \pairprobb & \pairprobr & \pairprobg & \pairproba & 0 & 0 \\
        0 & 0 & \pairproba & \pairprobg & \pairprobr & \pairprobb & \pairproba & \pairprobg \\
        0 & 0 & \pairprobg & \pairproba & \pairprobb & \pairprobr & \pairprobg & \pairproba \\
        \pairproba & \pairprobg & 0 & 0 & \pairproba & \pairprobg & \pairprobr & \pairprobb \\
        \pairprobg & \pairproba & 0 & 0 & \pairprobg & \pairproba & \pairprobb & \pairprobr \\
    \end{bmatrix*}
\end{equation}
for normalization constant $\pairnorm=8\pairprobr + 16 \pairproba + 8 \pairprobb + 16 \pairprobg$.

Let $\lambda' \in \R^8$ be the eigenvalues of $A$, and $U \in \R^{8 \times 8}$ be a matrix where the columns are corresponding unit-norm eigenvectors of $A$.
We can explicitly write out these eigenvectors and eigenvalues as follows:
\begin{equation}
\label{eqn:eigens_u_sep}
    U =
    \begin{bmatrix*}[r]
        1 & 1 & 0 & -1 & 0 & 1 & -1 & -1 \\
        1 & -1 & 0 & -1 & 0 & -1 & -1 & 1 \\
        1 & 0 & 1 & 0 & -1 & -1 & 1 & -1 \\
        1 & 0 & -1 & 0 & -1 & 1 & 1 & 1 \\
        1 & -1 & 0 & 1 & 0 & 1 & -1 & -1 \\
        1 & 1 & 0 & 1 & 0 & -1 & -1 & 1 \\
        1 & 0 & -1 & 0 & 1 & -1 & 1 & -1 \\
        1 & 0 & 1 & 0 & 1 & 1 & 1 & 1
    \end{bmatrix*} \cdot
    \normalizer,
\end{equation}
where each column of $U$ is an eigenvector of $A$, and $\normalizer$ is a diagonal matrix that normalizes the eigenvectors to unit norm, given by:
\begin{equation}
\normalizer = \mbox{diag}\left(\left[\frac{1}{\sqrt{8}}, \frac{1}{2}, \frac{1}{2}, \frac{1}{2}, \frac{1}{2}, \frac{1}{\sqrt{8}}, \frac{1}{\sqrt{8}}, \frac{1}{\sqrt{8}}\right]\right).
\end{equation}
The corresponding eigenvalues are given by $\lambda' \in \R^8$, where $\lambda'_i$ is the eigenvalue corresponding to the $i$-th column of $U$.
For convenience, let $\lambda = C \lambda'$ to avoid re-writing the $C$ term, and we have:
\begin{align*}
    \lambda_1 &= \pairprobr + 2\pairproba + \pairprobb + 2\pairprobg \\
    \lambda_2 &= \lambda_3 = -\pairprobb + \pairprobr \\
    \lambda_4 &= \lambda_5 = \pairprobb + \pairprobr \\
    \lambda_6 &= -2\pairproba - \pairprobb + 2\pairprobg + \pairprobr \\
    \lambda_7 &= -2\pairproba + \pairprobb - 2\pairprobg + \pairprobr \\
    \lambda_8 &= 2\pairproba - \pairprobb - 2\pairprobg + \pairprobr.
\end{align*}

We choose $k = 3$ for the representation dimension of contrastive learning, following Theorem~\ref{thm:sbm}.
We now show that the top $3$ eigenvalues includes $\lambda_1$ and $\lambda_8$ and \textit{does not} include $\lambda_2$ and $\lambda_6$.

Since $\pairprobr, \pairproba, \pairprobb, \pairprobg > 0$, from basic algebra:
\begin{equation}
  \lambda_1 > \lambda_4 = \lambda_5 > \lambda_2 = \lambda_3.
\end{equation}
Since $\pairproba > \pairprobg + \pairprobb$ and $\pairprobb > 0$, we also have $\pairproba > \pairprobg$, which means:
\begin{equation}
    \lambda_3 > \lambda_6.
\end{equation}
Using these facts, we also get:
\begin{equation}
    \lambda_1 > \lambda_8 > \lambda_2 = \lambda_3.
\end{equation}
We assumed that $\pairproba > \pairprobg + \pairprobb$, which gives us:
\begin{equation}
    \lambda_8 > \lambda_4 = \lambda_5.
\end{equation}
Finally, we note that $\lambda_1 > \lambda_8$ and $\lambda_1 > \lambda_7$ because $\pairprobr, \pairproba, \pairprobb, \pairprobg > 0$.

Collating all these inequalities, we find that $\lambda _1$ is the largest eigenvalue, and that $\lambda_8$ is larger than $\lambda_2, \lambda_3, \lambda_4, \lambda_5, \lambda_6$ (all eigenvalues except possibly $\lambda_7$, which can be larger or smaller than $\lambda_8$). This shows that the top $3$ eigenvalues \textit{includes} $\lambda_1$ and $\lambda_8$.
Since $\lambda_2$ is smaller than $\lambda_1, \lambda_8, \lambda_4, \lambda_5$ and $\lambda_6$ is even smaller than $\lambda_2$, we also get that $\lambda_2, \lambda_6$ (and $\lambda_3$) are \textit{not} among the top 3 eigenvalues.

We now write out the embedding matrix $\embedding$ in terms of the top 3 eigenvectors and eigenvalues.
From above, we know that the top 3 eigenvalues are $\lambda_a, \lambda_1, \lambda_8$, where $a \neq 2$ and $a \neq 6$.
Let $u_a, u_1, u_8$ be the corresponding eigenvectors: the $a$-th, $1$-st, and $8$-th column of $U$ respectively.
From~\citet{haochen2021spectral} (specifically, Lemma 3.2 and then using Eckart–Young–Mirsky theorem), for some orthonormal (rotation) matrix $R \in \R^{k \times k}$, we have:
\begin{equation}
    \embedding =
    \big[ \sqrt{\lambda'_8} u_8; \sqrt{\lambda'_1} u_1; \sqrt{\lambda'_a} u_a ]
    R
    = \frac{1}{\sqrt{C}} \big[ \sqrt{\lambda_8} u_8; \sqrt{\lambda_1} u_1; \sqrt{\lambda_a} u_a ]
    R.
\end{equation}
Let $\tau_8 = \sqrt{\lambda'_8} \cdot \normalizer_{88} = \sqrt{\lambda'_8} / \sqrt{8}$, $\tau_1 = \sqrt{\lambda'_1} \cdot \normalizer_{11} = \sqrt{\lambda'_1} / \sqrt{8}$, and $\tau_a = \sqrt{\lambda'_a} \cdot \normalizer_{aa}$.
We can then write $\embedding$ as follows, where $\tau_8, \tau_1, \tau_a > 0$. Focus on the values in the first column since this will be especially important in step 2 of the proof:
\begin{equation}
    \embedding =
    \begin{bmatrix*}[r]
        -\tau_8 & \tau_1 & \tau_a U_{1a} \\
        \tau_8 & \tau_1 & \tau_a U_{2a} \\
        -\tau_8 & \tau_1 & \tau_a U_{3a} \\
        \tau_8 & \tau_1 & \tau_a U_{4a} \\
        -\tau_8 & \tau_1 & \tau_a U_{5a} \\
        \tau_8 & \tau_1 & \tau_a U_{6a} \\
        -\tau_8 & \tau_1 & \tau_a U_{7a} \\
        \tau_8 & \tau_1 & \tau_a U_{8a} \\
    \end{bmatrix*} R.
\end{equation}
Next, we show that $U_{1a} = U_{2a}$.
To see this, recall that $a \not\in \{1, 2, 6, 8\}$.
Inspecting $U$ above in Equation~\ref{eqn:eigens_u_sep}, we see that for all other possible choices of $a$ (so $a \in \{3, 4, 5, 7\}$), $U_{1a} = U_{2a}$.

Therefore, we can write $\embedding$ as (we have only changed the third column on the second row, notice now that the first two rows only differ on the first coordinate):
\begin{equation}
    \label{eqn:final_embedding_separation}
    \embedding =
    \begin{bmatrix*}[r]
        -\tau_8 & \tau_1 & \tau_a U_{1a} \\
        \tau_8 & \tau_1 & \tau_a U_{1a} \\
        -\tau_8 & \tau_1 & \tau_a U_{3a} \\
        \tau_8 & \tau_1 & \tau_a U_{4a} \\
        -\tau_8 & \tau_1 & \tau_a U_{5a} \\
        \tau_8 & \tau_1 & \tau_a U_{6a} \\
        -\tau_8 & \tau_1 & \tau_a U_{7a} \\
        \tau_8 & \tau_1 & \tau_a U_{8a} \\
    \end{bmatrix*} R.
\end{equation}

\textbf{Step 2 (Linear probing)}:
We consider the target error of a classification head $\emphead$ composed with the encoder $\empencoder$ learned by contrastive pre-training.
The classification head minimizes the following objective (originally defined in Equation~\ref{eq:sq-loss}):
\begin{align}
    \label{eqn:fine-tune-optimization-separation}
    \lfinetune(\linmat) = \E_{\inputx \sim \sourcedist}\left[ \| \empencoder(\inputx)^\top \linmat - \labelx \|_2^2 \right] + \eta \|\linmat\|_2^2
\end{align}
where $\linmat \in \R^{k}$, the regularization is nonzero ($\eta > 0$), and predictions are made by $\text{sign}(\empencoder(\inputx)^\top \linmat)$.
The pre-trained encoder $\empencoder$ is fixed and obtained from step 1 above, and so the minimization is over $\linmat$.
In our case we only have two source labeled points, so the objective is minimized over these two labeled source points $\inputx \in \{1,2\}$.
Since the objective is rotationally symmetric, without loss of generality we can omit the rotation matrix $R$ in the pretrained representation above in Equation~\ref{eqn:final_embedding_separation}.
We consider the minimizer $\emplinmat$ of the fine-tuning objective $\lfinetune$.
Since the regularization strength $\eta$ is strictly greater than 0, the $j$-th coordinate of the minimizer $\emplinmat_j$ is 0 when the features for the labeled data $\inputx \in \{1, 2\}$ are identical on the $j$-th coordinate: $\empencoder(1)_j = \empencoder(2)_j$.

Note that $y_1 = 1$ and $y_2 = -1$, and $\empencoder(1)$ and $\empencoder(2)$ are given above in Equation~\ref{eqn:final_embedding_separation}.
Solving the optimization problem in Equation~\ref{eqn:fine-tune-optimization-separation} analytically (e.g., by taking derivatives setting to $0$), we get for some $\omega > 0$:
\begin{equation}
    \emplinmat = [-\omega, 0, 0].
\end{equation}
Finally we show that this classifier (composed with the features learned in the previous step) predicts every target example correctly.
Inspecting the feature matrix in Equation~\ref{eqn:final_embedding_separation}, we have for $\inputx \in \{1, 3, 5, 7\}$:
\begin{equation}
    \emplinmat^\top \empencoder(\inputx) = \tau \omega > 0 \mbox{ and } y_{\inputx} = 1,
\end{equation}
and for $\inputx \in \{2, 4, 6, 8\}$,
\begin{equation}
    \emplinmat^\top \empencoder(\inputx) = -\tau \omega < 0 \mbox{ and } y_{\inputx} = -1.
\end{equation}
In other words, $\text{sign}(\empencoder(\inputx)^\top \emplinmat) = y_{\inputx}$ for all $\inputx$, and so the target error is 0, as desired.

\paragraph{Conversion from augmentation to positive pair distribution.}
\label{app:aug-conversion}
Above, we worked out the proof for contrastive pre-training in terms of the positive pair distribution $\pospairdist(\inputxp, \inputx)$.
We showed that if $\pairproba > \pairprobg + \pairprobb$ then contrastive pre-training achieves $0$ target error.
We still need to show that such a setting exists---we started out by defining the augmentation probabilities, and we show that there is indeed some setting of $\augprobr, \augproba, \augprobg, \augprobb$ such that the condition $\pairproba > \pairprobg + \pairprobb$ holds.

Recall that the positive pair distribution is defined as
\begin{align}
    \pospairdist(\inputxp, \inputx) = \E_{\bar{\inputx} \sim \unlabeldist} [ \aug(\inputx \mid \bar{\inputx}) \aug(\inputxp \mid \bar{\inputx}) ].
\end{align}
The distribution $\pospairdist$ induced by $\aug$ on $\inputspace$ is
\begin{align}
    \pospairdist(\inputxp, \inputx) = \begin{cases}
        \pairprobr = \frac{1}{8C} [ \augprobr^2 + 2\augproba^2 + \augprobb^2 + 2\augprobg^2 ] & \inputx = \inputxp  \\
        \pairproba = \frac{1}{4C} [ \augprobr \augproba + \augprobb \augprobg ] &\labelx = \labelxp, \inputx \neq \inputxp \\
        \pairprobb = \frac{1}{4C} [ \augprobr \augprobb + 2 \augproba \augprobg ] &\set{\inputxp, \inputx} \in \set{\set{1, 2}, \set{3, 4}, \set{5, 6}, \set{7, 8}} \\
        \pairprobg = \frac{1}{4C} [ \augprobr \augprobg + \augproba \augprobb ] &\set{\inputxp, \inputx} \in \set{\set{1, 4}, \set{2, 3}, \set{3, 6}, \set{4, 5}, \set{5, 8}, \set{6, 7}, \set{1, 8}, \set{2, 7}}
    \end{cases}
\end{align}
where the constant $C$ ensures that the probabilities sum to 1.

By setting $\augprobg = 0, \augprobb = 0, \augproba > 0, \augprobr > 0$ such that $2\augproba + \augprobr = 1$, we get the desired condition $\pairproba > \pairprobg + \pairprobb$.
This shows that there exists one such scenario, where the condition holds, as desired.
However, we note that there are many situations where the condition $\pairproba > \pairprobg + \pairprobb$ holds.

\section{Additional Details for Section~\ref{sec:main-results}}
\label{app-experiment-details}

\subsection{Additional Experiments}
\label{app-additional-exp}

\paragraph{Table~\ref{table:app-empirical-domainnet}.}
Table~\ref{table:app-empirical-domainnet} contains target test accuracies on all individual pairs of DomainNet domains (Table~\ref{table:main-empirical} contains only the average) with all baselines and SwAV, in addition to MoCo-V2 and MoCo-V3 (additional contrastive pre-training methods).
The MoCo-V2 and MoCo-V3 hyperparameters were tuned on ImageNet in their original papers, and we did not tune either method further; therefore, for MoCo-V2 the target test accuracies are lower than several of the baselines on DomainNet (which is very different from ImageNet).
However, MoCo-V3 outperforms all methods except for SwAV and DANN+\strongaugs{} (41.59\% vs. the next best, DANN: 38.38\%).

\paragraph{Tables~\ref{table:app-empirical-domainnet-early-stop} and~\ref{table:app-empirical-breeds-early-stop}.}
Tables~\ref{table:app-empirical-domainnet-early-stop} and~\ref{table:app-empirical-breeds-early-stop} contain target test accuracies on DomainNet and BREEDS of baselines with early stopping.
The top halves of both tables contain results of early stopping with source test accuracy (i.e., selecting the epoch to use based on the highest source test accuracy achieved over all epochs),
and the bottom halves contain results of early stopping with target test accuracy (i.e., the highest target test accuracy achieved over all epochs; this should be interpreted as an ``oracle'' method).
As expected, early stopping improves over simply using the final iterate (as in Tables~\ref{table:main-empirical} and~\ref{table:app-empirical-domainnet}).
However, even with early stopping only DANN+\strongaugs{} outperforms SwAV (which uses no early stopping)---by 1.3\% and 3.2\% for early stopping with source and target accuracy, respectively.

\paragraph{Tables~\ref{table:app-empirical-domainnet-fair} and~\ref{table:app-empirical-breeds-fair}.}
Tables~\ref{table:app-empirical-domainnet-fair} and~\ref{table:app-empirical-breeds-fair} contain target test accuracy when we standardize the compute requirements as much as possible between the baselines and pre-training methods, since our primary experiments allow the baselines to run for nearly twice as many epochs as the pre-training methods.
Specifically, here we enforce that all methods are run for 1) the same the number of epochs with unlabeled data and 2) the same the number of epochs with unlabeled data.
In particular, the baselines are initialized from an ERM baseline trained for 100 to 150 epochs, rather than 500 to 550.
As a result, the resulting target accuracies are significantly lower for SENTRY and SENTRY+\strongaugs{} (SENTRY was developed using ImageNet-pre-trained models and requires a decently accurate initialization); however, the accuracy for DANN and DANN+\strongaugs{} is less substantial.

\paragraph{Table~\ref{table:app-breeds}.}
To verify that other contrastive learning methods are also competitive as UDA methods, Table~\ref{table:app-breeds} contains target test accuracies on BREEDS of SwAV (with source-only and target-only splits of the pre-training data), MoCo-V3 (with source and target pre-training), and Dino and Barlow Twins (both with ImageNet pre-trained weights downloaded from official Github repositories).
Target accuracy increases from SwAV (S) to SwAV (T) to SwAV (S+T), and the accuracies of SwAV (S+T) and MoCo (S+T) are comparable (within 1\% of each other for both Living-17 and Entity-30).
Dino+ and Barlow Twins+ are both higher (up to 3\%) than SwAV+ on Living-17 and slightly lower (1.4\%) on Entity-30.

\subsection{Datasets}

\paragraph{BREEDS.}
BREEDS is a subpopulation shift benchmark derived from ImageNet by constructing a hierarchical tree structure of classes from WordNet. Nodes at a specified depth of the tree become the labels for the classification task, and descendant nodes are treated as subpopulations that can be randomly partitioned into source and target domains.

We use the dataset creation functions defined in the \verb|robustness| Python library to generate the Living-17 and Entity-30 tasks from the original ImageNet dataset~\citep{madrylab2019robustnesslib, russakovsky2015imagenet}.
The Living-17 dataset is an animal classification task which consists of nodes in the subtree rooted at the ``living thing'' node in the WordNet hierarchy.
An example of a label is ``ape'' with subpopulations of gibbons, orangutans, gorillas, and chimpanzees.
The Entity-30 dataset is a much more general task, incorporating nodes in the ``entity'' subtree.
Labels include ``building'' and ``home appliance''.
The trailing number in the dataset name is the total number of classes in that dataset.

\paragraph{DomainNet.}
DomainNet is a large unsupervised domain adaptation task, consisting of approximately 600,000 images and 345 classes in 6 domains. Each image belongs to one of 6 domains, including sketches, clipart, and photographs. For our experiments we utilize the same filtered version of DomainNet from \citet{prabhu2021sentry}, which uses 40 of the 345 classes and the sketch, painting, photograph, and clipart domains. This refinement is done to eliminate much of the noise present in the original DomainNet dataset~\citep{tan2020coal}.

For our baseline experiments with the SENTRY algorithm~\citep{prabhu2021sentry}, we use the authors' official repository (\url{https://github.com/virajprabhu/SENTRY}), which filters the original DomainNet dataset automatically as described above.

\paragraph{STL$\to$CIFAR.}
CIFAR10 consists of $32\times 32$ images from the former TinyImages dataset, and STL10 is derived from ImageNet and contains images with resolution $96\times96$. Both are classical image recognition datasets and are often paired together for domain shift tasks~\citep{shu2018dirtt, french2018selfensembling}.
We resize the STL-10 images to $32\times32$ to match the resolution of CIFAR10, as done in~\citet{french2018selfensembling}.
The two non-overlapping classes (``monkey'' in CIFAR-10 and ``frog'' in STL10) are removed from both datasets before training, making the task a 9-class classification problem.

\subsection{Algorithms and Hyperparameter Tuning}
\label{app-hyp-tuning}

For joint-training baselines, we use the final iterate and we select hyperparameters based on \textbf{target test} accuracy (i.e., OOD accuracy).

\paragraph{ERM.}
\begin{itemize}
    \item
        \textbf{DomainNet.}
        The SENTRY algorithm runs ERM with class balancing (starting with ImageNet-pre-trained initialization) prior to beginning entropy minimization, and therefore the SENTRY repository contains an ERM implementation and hyperparameters for DomainNet.
        Therefore, using that ERM implementation we run ERM for 550 epochs and multiply the initial learning rate by 10x, keeping all other hyperparameters from~\citep{prabhu2021sentry} constant.
    \item
        \textbf{BREEDS.}
        We use almost the same hyperparameters as~\citet{santurkar2020breeds}.
        However, on Entity30, we train for 500 epochs and divide the learning rate by 10 every 500/3 = 167 epochs.
        On Living17, we train for 500 epochs and divide the learning rate by 10 every 167 epochs.
        For both datasets we use a learning rate of 0.1, a weight decay of $10^{-4}$, and a batch size of 128.
    \item
        \textbf{STL $\to$ CIFAR.}
        The STL10 training set is much smaller than those of DomainNet and BREEDS, which gives us more freedom to sweep over hyperparameters. In turn, we vary the number of training epochs amongst $\{100, 150, 200, 250, 300\}$, the learning rate amongst $\{0.0001, 0.001, 0.01, 0.1\}$, the softmax temperature amongst $\{0.5, 0.75, 1.0\}$, and the weight decay parameter amongst $\{0.0005, 0.005, 0.05, 0.5\}$.
        We train the ERM model for 1000 epochs.

\end{itemize}

\paragraph{SENTRY.}
We use the official implementation (\url{https://github.com/virajprabhu/SENTRY}).
\begin{itemize}
    \item
        \textbf{DomainNet.}
        For each pair of domains, we conduct a hyperparameter search through $\lambda_{\text{src}} \in \{ 0.5, 1.0, 1.5 \}$ (the weight on the supervised classification loss) and learning rates $\in \{ 0.01, 0.001 \}$.
        We run all hyperparameter settings for 100 epochs from a 150-epoch ERM checkpoint and select the best one (based on target test accuracy).
        We then initialize the SENTRY model with the ERM model described earlier (i.e., trained for 550 epochs) and run SENTRY for 400 epochs (for a total of 950 epochs).
    \item
        \textbf{BREEDS.}
        We keep the default hyperparameters from the SENTRY repo but search over 3 learning rates $\{ 0.001, 0.01, 0.1 \}$ for 100 epochs and then run the best hyperparameter setting for 300 additional epochs for 400 total epochs.
        We initialize the model with the ERM model described earlier (i.e., trained for 500 epochs).
    \item
        \textbf{STL $\to$ CIFAR.}
        We initialize SENTRY with an ERM model trained for 100 epochs and then train SENTRY for 1000 epochs, varying the learning rate between $\{0.001, 0.01, 0.1\}$, the weight decay between $\{0.0005, 0.005, 0.05, 0.5 \}$, and the unsupervised loss weight between $\{ 0.01, 0.1, 1.0\}$.
\end{itemize}

For the ``strong augmentation'' version of SENTRY, we replace the RandAug algorithm used for consistency regularization with the augmentations used for SwAV (in DomainNet and BREEDS) or SimCLR (for STL / CIFAR).

\paragraph{DANN.}
We use the implementation provided in the SENTRY repository (\url{https://github.com/virajprabhu/SENTRY}).
\begin{itemize}
    \item
        \textbf{DomainNet.}
        For each pair of domains, we conduct a hyperparameter search through learning rates $\in \{ 0.01, 0.001 \}$ and temperature $\in \{ 0.9, 1.0 \}$.
        As with SENTRY, we run all hyperparameter setting for 100 epochs from a 150-epoch ERM checkpoint, select the best one (based on target test accuracy).
        We initialize the DANN model with the ERM model described earlier (i.e., trained for 550 epochs) and run DANN for 400 epochs (for a total of 950 epochs).
    \item
        \textbf{BREEDS.}
        We sweep over two learning rates of 0.001 and 0.0005 for 100 epochs and choose the learning rate that achieves the highest OOD accuracy.
        We then run DANN for 300 additional epochs for 400 total epochs.
        We initialize the SENTRY model with the ERM model described earlier.
    \item
        \textbf{STL $\to$ CIFAR.}
        We follow an identical procedure to that of SENTRY, initializing DANN with an ERM model trained for 100 epochs and sweeping over the same hyperparameter set. DANN was run for 1000 epochs.
\end{itemize}

For the ``strong augmentation'' version of DANN we simply replace the input augmentations with the augmentations used by SwAV (for DomainNet and BREEDS) or SimCLR (for STL / CIFAR) and initialize DANN from an ERM checkpoint trained with SwAV augmentations.

\paragraph{DIRT-T.}
DIRT-T~\citep{shu2018dirtt} is a domain adaptation method that addresses two flaws of domain adversarial neural networks~\citep{ganin2016domain}: 1) distribution matching is a weak constraint, and 2) in some domain adaptation settings there does not exist a good joint classifier on both source and target.
The authors address the first shortcoming by adding a conditional entropy regularization term so that the model's decision boundaries do not overlap high-density regions of data.
This is inspired by the \textit{cluster assumption}, which states that the input space is divided into well-separated clusters, one for each class in the label space.
The lack of a good classifier on both source and target is then addressed via self-training on the unlabeled target data.
We report the STL$\to$CIFAR DIRT-T results from~\citet{shu2018dirtt}.

\subsection{SwAV}
We use the official SwAV implementation (\url{https://github.com/facebookresearch/swav}) and keep almost all of the hyperparameters provided by the original paper for 400 epoch, 256-batch-size training on ImageNet.
However, we use a batch size of 512 and tuned SwAV slightly on Living-17 using only the training curves (no labels) and following the advice from the Github README and issues answered by the original authors:
\begin{itemize}
	\item
        In order to stabilize training, we do not use a queue on DomainNet and the subsampled variants of Living-17.
        For pre-training on the full Living-17 and Entity-30 datasets, we introduce the queue at epoch 60.
	\item
        We set the number of prototypes to be 10 times the number of classes (170, 300, and 400 for Living-17, Entity-30 and DomainNet, respectively).
	\item
        We set $\epsilon = 0.03$.
	\item
        We set the base learning rate to 0.6, following a linear scaling rule based on batch size.
\end{itemize}
We note that the joint-training baselines required extensive hyperparamter tuning for \textit{every} transfer task.
In contrast, because SwAV was tuned for ImageNet training, we selected hyperparameters once (using only the pre-training loss and no labels on Living-17) and use the same hyperparameters (except for the queue) on all our datasets.
For fine-tuning, we always initialize with the final iterate of SwAV pre-training (400 epochs). Additional dataset-specific details:
\begin{itemize}
    \item
        \textbf{DomainNet.} 
        For consistency, we fine-tune SwAV pre-trained models for 150 epochs with strong data augmentations using the ERM implementation in the SENTRY repository (without running any joint-training algorithm), keeping all hyperparameters other than the number of epochs constant.
        We report the target test accuracy of the final iterate (i.e., after 150 epochs).
    \item
        \textbf{Living-17.}
        We fine-tune SwAV models for 100 epochs strong data augmentations and with a cosine learning rate schedule without restarts. We use SGD with initial learning rate 0.1 for the classifier head and 0.01 for the encoder, momentum 0.9, and weight decay 0.0001.
        We use a batch size of 96, and once again report the target test accuracy of the final model (i.e., after 100 epochs).
    \item
        \textbf{Entity-30.}
        We linear probe for 100 epochs instead of fine-tune on Entity-30, due to its large size (300k examples).
\end{itemize}

\subsection{SimCLR}
We use the official SimCLR implementation (\url{https://github.com/google-research/simclr}).
We use a batch size of 256, a learning rate of 0.2, and weight decay $10^{-4}$.
The projection head has two layers and an output dimension of 64.
We pre-train the model for 1000 epochs with square-root learning rate scaling and we train the linear probe for 100 epochs on batches of size 512 and a learning rate of 0.1.

\subsection{MoCo-V2 and MoCo-V3}
We use the official MoCo implementations (\url{https://github.com/facebookresearch/moco} and \url{https://github.com/facebookresearch/moco-v3}).
We use a batch size of 512 and otherwise all the default hyperparameters for MoCo-V2.
We use a batch size of 256 and otherwise all the defualt hyperparameters for MoCo-V3.
We fine-tune both MoCo versions on DomainNet and Living-17 using the same protocol as SwAV fine-tuning for each dataset.

\subsection{Dino and Barlow Twins}
We use ImageNet pre-trained weights from the official Github repositories (\url{https://github.com/facebookresearch/dino} and \url{https://github.com/facebookresearch/barlowtwins}).
We fine-tune on Living-17 and linear probe on Entity-30 using the same protocol as for SwAV.
We note that while SwAV+extra is pre-trained for 400 epochs, the Dino and Barlow Twins models were pre-trained for 800 and 1000 epochs, respectively.

\subsection{Model architectures.}
For all BREEDS experiments, we use a standard ResNet50 from the PyTorch \verb|torchvision| library.
For DomainNet, we use a ResNet50 slightly modified for few-shot learning, following~\citet{prabhu2021sentry}.
On STL$\to$CIFAR for ERM, DANN, SENTRY, and SimCLR baselines we use a standard ResNet18 from \verb|torchvision|, and we additionally report the DIRT-T performance from~\citet{shu2018dirtt}, which uses a custom 18-layer CNN.

\section{Additional Experiments Verifying the Theory}
\label{app-verification}

\subsection{Target accuracy as a function of connectivity}
Eq.~\ref{eq:connect-ratio} can be rewritten as
\begin{align}
    \label{eq:log-connect-ratio}
    \log(\text{target accuracy}) \approx \wone \cdot \log(\acrossdomain/\acrossboth) + \wtwo \cdot \log(\acrossclass/\acrossboth)
\end{align}
and for each baseline, we fit $\wone$ and $\wtwo$ using a linear regression model on data from the 12 source/target domain pairs of DomainNet.
We then compute the coefficient of determination $R^2$ between the observed and predicted target accuracies.
The fitted exponents $\wone, \wtwo$ and $R^2$ for SwAV, MoCo-V2, and MoCo-V3 are provided in Table~\ref{table:app-exponents}.
Figures~\ref{fig:trend-plots-app-one} and~\ref{fig:trend-plots-app-two} plot the observed target accuracies of SwAV, MoCo-V2, MoCo-V3, DANN, and SENTRY (y-axis) against the predicted (x-axis), along with a line of best fit through them as a visual aid.
We find that $R^2$ is high for contrastive learning (0.78, 0.79, and 0.60) but relatively low for the baselines ($R^2 \in [0.23, 0.51]$ for DANN and $R^2 \in [-0.20, 0.22]$ for SENTRY).

\subsection{Ablation of connectivity}
For this experiment on Living-17, we used a modified version of the dataset as follows:
for each of the 17 classes, we trained a ResNet-50 classifier to distinguish between augmented images of that class in the source and target.
Of the resulting classifiers, we selected the 4 that obtained domain classification accuracy $>70\%$: these were classes 6, 11, 12, and 14.
We only kept those 4 classes, and all data selection methods we considered chose subsets from this modified dataset.

\subsection{Disentanglement of class and domain information}
Table~\ref{table:dot-products-app} contains individual results on the cosine similarity between domain and class classifiers on DomainNet (abbreviated version in Table~\ref{table:dot-products}).
In the fine-tuned feature space, the class and domain classifiers remain orthogonal while the cosine similarity between source and target classifiers increases by 41\% from 0.187 to 0.264.

\section{Results and Protocol for Estimating Connectivity Parameters on Benchmark Datasets}
\label{sec:app-connectivity}

Table~\ref{table:fullconnectivity} reports the connectivity estimates for the input space and the feature spaces learned by CLIP, SwAV, and DANN+\strongaugs{} for all pairs of DomainNet domains (abbreviated version in Table~\ref{table:connectivity}).
To estimate the average connectivity between two class-domain pairs $(c, d)$ and $(c', d')$, we use the following algorithm:
\begin{enumerate}
    \item
        Label all training examples of class $c$ and domain $d$ as 0 and all training examples of class $c'$ and domain $d'$ as 1.
        Discard the remaining examples from other classes/domains.
    \item
        Train a ResNet50 for 100 epochs using strong augmentations and SGD with momentum and a cosine learning rate.
        We did not exhaustively tune this training step, but we kept the procedure constant for all classes and domains.
        If estimating connectivity in the input space, we train the entire network; if estimating in the feature space, we train only a linear classifier with a frozen encoder.
    \item
        Create the test set analogously to step 1, and evaluate the classifier on strongly augmented data.
        The test error of the classifier is interpreted as an estimate for connectivity between the two class-domain pairs.
\end{enumerate}
Each domain in DomainNet has a unique label distribution (all of which are far from uniform), and therefore in computing the average connectivity we compute the weighted mean, where each pair of (class, domain) pairs is weighted by the ratio of the less to more frequent label (0 or 1).

Note that the input space connectivity is calculated by training classifiers \textit{from scratch} on pairs of class-domain pairs, resulting in smaller training set sizes compared to SwAV and DANN.
Thus, the numbers provided in Table~\ref{table:connectivity} should be interpreted as \textit{relative} estimates of connectivity.
As an alternative estimate of connectivity, we additionally report results of fine-tuning a CLIP ViT-B/16, using the LP-FT~\citep{kumar2022finetuning} fine-tuning strategy with AdamW~\citep{loshchilov2019decoupled}.
The connectivity estimates in the CLIP feature space are much lower than in the SwAV or DANN feature spaces.
Interestingly, we find that the relative magnitudes of estimated connectivity across class ($\acrossclass$) and across domain ($\acrossdomain$) are swapped between the input space and CLIP feature space.
This discrepancy may be due to biases in the CLIP training data (e.g., captions that are more informative of the class than of the domain, resulting in features that are also more informative of the class).

In the SwAV feature space, the across-domain and across-class connectivities are approximately equal (7.54 vs. 7.03).
On other hand, the across-domain is much higher than across-class in the DANN+\strongaugs{} feature space (13.64 vs. 5.65); intuitively, the classification component of the DANN objective pushes apart the classes in feature space, while the domain discrimination component brings together domains.
The connectivities (i.e., classifier error) are all much higher in the input space because the classifiers are trained from scratch and on much smaller datasets, while for SwAV and DANN the classifiers are trained via linear probing in the feature space.
Thus, to standardize compute between the methods, we used the DANN+\strongaugs{} checkpoints from the compute-standardized experiment (discussed in Section~\ref{app-additional-exp} and Tables~\ref{table:app-empirical-domainnet-fair} and~\ref{table:app-empirical-breeds-fair}).

\FloatBarrier
\section{Appendix Figures and Tables}

The tables and figures for the appendix are listed below in order to improve readability in the previous sections.

\begin{table*}[h]
\begin{minipage}{1.0\linewidth}
\centering
\resizebox{\linewidth}{!}{
\begin{tabular} {l r r r r r r r r r r r r r}
	\toprule
	Source & \multicolumn{3}{c}{Real} & \multicolumn{3}{c}{Sketch} & \multicolumn{3}{c}{Painting} & \multicolumn{3}{c}{Clipart} & Avg. \\
	\cmidrule(lr){2-4}\cmidrule(lr){5-7}\cmidrule(lr){8-10}\cmidrule(lr){11-13}
	Target & Sketch & Painting & Clipart & Real & Painting & Clipart & Real & Sketch & Clipart & Real & Sketch & Painting \\
	\midrule
    ERM & 29.55 & 43.25 & 45.92 & 28.09 & 20.66 & 34.41 & 33.86 & 16.63 & 22.83 & 20.68 & 12.59 & 11.52 & 26.67 \\
    ERM (+DA) & 42.31 & 48.47 & 49.75 & 36.47 & 34.65 & 44.68 & 45.59 & 39.47 & 35.33 & 27.93 & 31.18 & 16.36 & 37.68 \\
    SENTRY & 42.89 & 57.03 & \textbf{65.41} & \textbf{55.81} & \textbf{42.97} & 53.03 & 42.98 & 1.92 & 2.10 & 8.84 & 4.42 & 1.58 & 31.58 \\
    SENTRY (+DA) & 40.43 & 48.92 & 54.02 & 44.00 & 29.39 & 51.18 & 3.51 & 2.58 & 23.82 & 3.60 & 2.88 & 5.50 & 25.82 \\
    DANN & 43.93 & 49.12 & 56.06 & 44.46 & 36.06 & 48.33 & 44.29 & 31.89 & 30.01 & 33.30 & 25.59 & 17.46 & 38.38 \\
    DANN (+DA) & \textbf{54.19} & 51.29 & 58.66 & 48.22 & 41.46 & \textbf{54.95} & 52.40 & \textbf{49.06} & \textbf{37.75} & 38.01 & \textbf{41.52} & 22.24 & \textbf{45.81} \\
    MoCo-V2 & 34.88 & 43.79 & 47.77 & 36.07 & 23.85 & 32.92 & 47.86 & 29.38 & 25.43 & 34.55 & 23.75 & 14.91 & 32.93 \\
    MoCo-V3 & 44.60 & \textbf{57.82} & 53.53 & 48.85 & 29.84 & 32.67 & 63.86 & 36.47 & 31.99 & 50.01 & 27.97 & 21.45 & 41.59 \\
	SwAV & 43.76 & 54.55 & 53.27 & 55.48 & 34.99 & 40.59 & \textbf{67.08} & 39.64 & 32.98 & \textbf{57.13} & 34.30 & \textbf{25.09} & 44.91 \\
	\midrule
    SwAV+ & 44.64 & 57.27 & 54.20 & 58.10 & 46.75 & 53.46 & 69.03 & 48.68 & 41.33 & 59.38 & 46.22 & 41.66 & 51.73 \\
	\bottomrule
\end{tabular}
}
\end{minipage}
\caption{%
    Test accuracy (\%) of baselines, MoCo-V2, MoCo-V3, SwAV, and SwAV+extra on all individual domain pairs of DomainNet.
    SwAV on average is within 1\% of the best baseline, DANN+\strongaugs{} (44.91 vs. 45.81).
}
\label{table:app-empirical-domainnet}
\end{table*}

\begin{table*}[h]
\begin{minipage}{1.0\linewidth}
\centering
\resizebox{\linewidth}{!}{
\begin{tabular} {l r r r r r r r r r r r r r}
	\toprule
	Source & \multicolumn{3}{c}{Real} & \multicolumn{3}{c}{Sketch} & \multicolumn{3}{c}{Painting} & \multicolumn{3}{c}{Clipart} & Avg. \\
	\cmidrule(lr){2-4}\cmidrule(lr){5-7}\cmidrule(lr){8-10}\cmidrule(lr){11-13}
	Target & Sketch & Painting & Clipart & Real & Painting & Clipart & Real & Sketch & Clipart & Real & Sketch & Painting \\
	\midrule
    \multicolumn{6}{l}{\textbf{Early stopping using source test accuracy}} \\
    SENTRY & 43.27 & 55.00 & 63.49 & 54.96 & 42.87 & 51.67 & 41.83 & 20.22 & 25.43 & 21.23 & 10.38 & 10.42 & 36.73 \\
    SENTRY (+DA) & 36.18 & 49.54 & 59.96 & 51.33 & 29.98 & 54.39 & 33.88 & 10.13 & 20.73 & 18.03 & 15.42 & 9.52 & 32.42 \\
    DANN & 40.31 & 46.17 & 53.09 & 48.71 & 38.29 & 51.42 & 46.36 & 31.72 & 28.90 & 34.06 & 27.05 & 17.19 & 38.61 \\
    DANN (+DA) & 55.86 & 52.01 & 57.49 & 46.35 & 42.35 & 56.81 & 50.64 & 50.15 & 42.20 & 37.72 & 40.39 & 23.00 & 46.25 \\
	\midrule
    \multicolumn{6}{l}{\textbf{Early stopping using target test accuracy}} \\
    SENTRY & 48.77 & 58.74 & 66.15 & 58.74 & 43.62 & 54.14 & 47.40 & 20.25 & 28.15 & 22.10 & 10.71 & 11.17 & 39.16 \\
    SENTRY (+DA) & 42.35 & 53.66 & 63.67 & 51.33 & 31.90 & 58.47 & 34.36 & 16.13 & 27.97 & 18.66 & 15.42 & 9.52 & 35.29 \\
    DANN & 45.18 & 50.01 & 56.93 & 50.03 & 40.73 & 52.84 & 47.11 & 34.76 & 32.17 & 35.82 & 28.01 & 20.07 & 41.14 \\
    DANN (+DA) & 56.69 & 53.76 & 61.50 & 49.44 & 42.38 & 58.10 & 52.39 & 50.93 & 43.37 & 39.08 & 44.85 & 24.95 & 48.12 \\
	\bottomrule
\end{tabular}
}
\end{minipage}
\caption{%
    Test accuracy (\%) of DANN and SENTRY on all individual domain pairs of DomainNet with early stopping using source test accuracy (top) and target test accuracy (bottom) as opposed to using the final iterate, which is used in Table~\ref{table:main-empirical}.
    Early stopping SENTRY consistently leads to large boosts in accuracy (5.1\% and 7.5\% for early stopping with source and target), while early stopping DANN leads to much smaller boosts (0.2\% and 2.7\% for early stopping with source and target).
    Compared to SwAV, with either method of early stopping, SENTRY and SENTRY+\strongaugs~are more than 5\% lower and DANN is more than 3\% lower.
    However, DANN+\strongaugs~is better than SwAV by 1.3\% and 3.2\% for early stopping with source and target accuracy, respectively.
}
\label{table:app-empirical-domainnet-early-stop}
\end{table*}

\begin{table*}[h]
\centering
\begin{minipage}{0.32\linewidth}
\centering
\resizebox{\linewidth}{!}{
\begin{tabular} {l r r}
	\toprule
    & Living-17 & Entity-30 \\
	\midrule
    \multicolumn{3}{l}{\textbf{Early stopping using source test accuracy}} \\
    SENTRY & 79.24 & 63.45 \\
    SENTRY (+DA) & 76.06 & 61.55 \\
    DANN & 67.41 & 53.30 \\
    DANN (+DA) & 70.12 & 53.68 \\
	\midrule
    \multicolumn{3}{l}{\textbf{Early stopping using target test accuracy}} \\
    SENTRY & 79.94 & 64.07 \\
    SENTRY (+DA) & 77.18 & 64.10 \\
    DANN & 69.59 & 59.10 \\
    DANN (+DA) & 73.18 & 57.82\\
	\bottomrule
\end{tabular}
}
\end{minipage}
\caption{%
    Test accuracy (\%) of DANN and SENTRY on Living-17 and Entity-30 with early stopping using source test accuracy (top) and target test accuracy (bottom) as opposed to using the final iterate, which is used in Table~\ref{table:main-empirical}.
    Both methods of early stopping SENTRY and SENTRY+\strongaugs{} led to boosts in accuracy over using the final iterate, but early stopping DANN and DANN+\strongaugs{} with source accuracy inconsistently increased performance over the final iterate.
}
\label{table:app-empirical-breeds-early-stop}
\end{table*}

\begin{table*}[h]
\begin{minipage}{1.0\linewidth}
\centering
\resizebox{\linewidth}{!}{
\begin{tabular} {l r r r r r r r r r r r r r}
	\toprule
	Source & \multicolumn{3}{c}{Real} & \multicolumn{3}{c}{Sketch} & \multicolumn{3}{c}{Painting} & \multicolumn{3}{c}{Clipart} & Avg. \\
	\cmidrule(lr){2-4}\cmidrule(lr){5-7}\cmidrule(lr){8-10}\cmidrule(lr){11-13}
	Target & Sketch & Painting & Clipart & Real & Painting & Clipart & Real & Sketch & Clipart & Real & Sketch & Painting \\
	\midrule
    SENTRY & 39.39 & 48.33 & 55.63 & 56.00 & 39.22 & 48.02 & 17.77 & 2.13 & 3.90 & 3.79 & 2.33 & 1.44 & 26.50 \\
    SENTRY (+DA) & 37.14 & 36.58 & 51.42 & 44.88 & 23.20 & 46.23 & 14.35 & 7.50 & 12.93 & 3.27 & 7.00 & 2.65 & 23.93 \\
    DANN & 44.02 & 43.86 & 51.18 & 44.42 & 37.47 & 52.10 & 47.96 & 31.10 & 32.43 & 32.36 & 24.26 & 15.30 & 38.04 \\
    DANN (+DA) & 50.77 & 50.05 & 56.37 & 47.90 & 40.39 & 55.14 & 50.47 & 46.23 & 40.10 & 37.38 & 41.60 & 23.55 & 45.00 \\
	\bottomrule
\end{tabular}
}
\end{minipage}
\caption{%
    Test accuracy (\%) of DANN and SENTRY on all pairs of DomainNet with standardized compute resources as described in Section~\ref{app-additional-exp}.
    On average, DANN and DANN+\strongaugs{} achieve comparable results to those in the main table (Table~\ref{table:main-empirical}), while SENTRY and SENTRY+\strongaugs{} are both lower here than in the main table.
}
\label{table:app-empirical-domainnet-fair}
\end{table*}

\begin{table*}[h]
\centering
\begin{minipage}{0.32\linewidth}
\centering
\resizebox{\linewidth}{!}{
\begin{tabular} {l r r}
	\toprule
    & Living-17 & Entity-30 \\
	\midrule
    SENTRY & 60.65 & 54.71 \\
    SENTRY (+DA) & 69.12 & 56.55 \\
    DANN & 65.65 & 55.93 \\
    DANN (+DA) & 68.06 & 56.00 \\
	\bottomrule
\end{tabular}
}
\end{minipage}
\caption{%
    Test accuracy (\%) of DANN and SENTRY on Living-17 and Entity-30 with standardized compute resources as described in Section~\ref{app-additional-exp}.
    For all 4 methods, the accuracies shown here are somewhat close to ($\leq 2.5\%$ lower than) those shown in the main table (Table~\ref{table:main-empirical}).
}
\label{table:app-empirical-breeds-fair}
\end{table*}

\begin{table*}[h]
\centering
\begin{minipage}{0.85\linewidth}
\centering
\resizebox{\linewidth}{!}{
\begin{tabular}{l r r r r r r r}
	\toprule
    & SwAV (S) & SwAV (T) & SwAV (S+T) & MoCo-V3 (S+T) & SwAV+ & Dino+ & Barlow Twins+ \\
    \midrule
    Living-17 & 62.71 & 70.41 & 75.12 & 74.88 & 82.00 & 83.05 & 85.11 \\
    Entity-30 & 52.33 & 60.33 & 62.03 & 63.00 & 65.90 & 64.50 & 64.48 \\
	\bottomrule
\end{tabular}
}
\end{minipage}
\caption{%
    Test accuracy (\%) of additional contrastive pre-training methods on Living-17 and Entity-30: SwAV (with source-only and target-only pre-training), Moco-V3 (source and target pre-training), Dino+ (ImageNet pre-training), and Barlow Twins+ (ImageNet pre-training).
    For ease of comparison, SwAV (S+T) and SwAV+ results are repeated from Table~\ref{table:main-empirical}.
    On both datasets, with source and target pre-training MoCo-V3 is comparable with SwAV and with ImageNet pre-training Dino and Barlow Twins are comparable with SwAV.
}
\label{table:app-breeds}
\end{table*}

\begin{figure*}[h]
    \centering
    \includegraphics[align=c,height=4cm]{figs/trend-swav.pdf}
    \hspace{0.02\linewidth}
    \includegraphics[align=c,height=4cm]{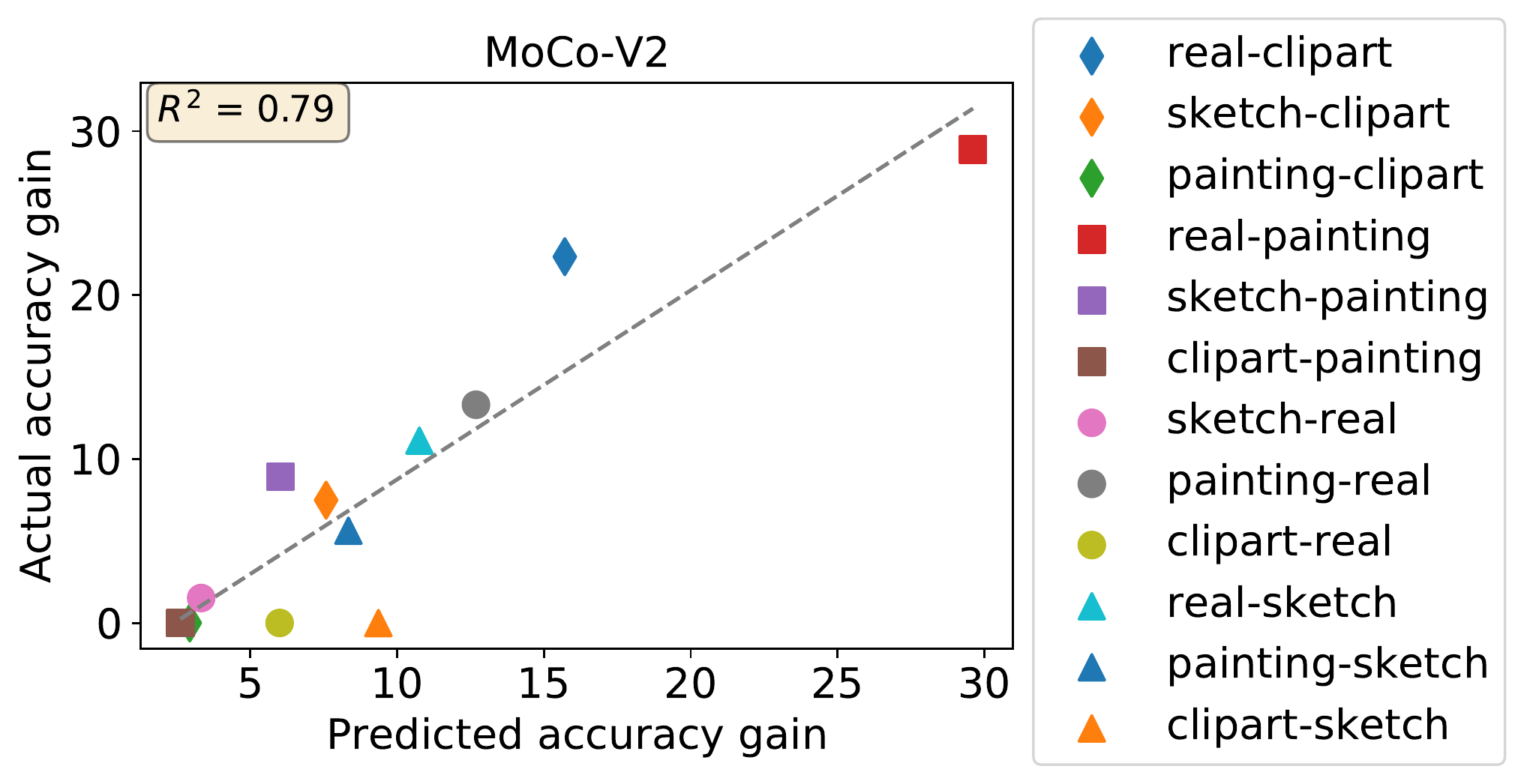} \\
    \includegraphics[align=c,height=4cm]{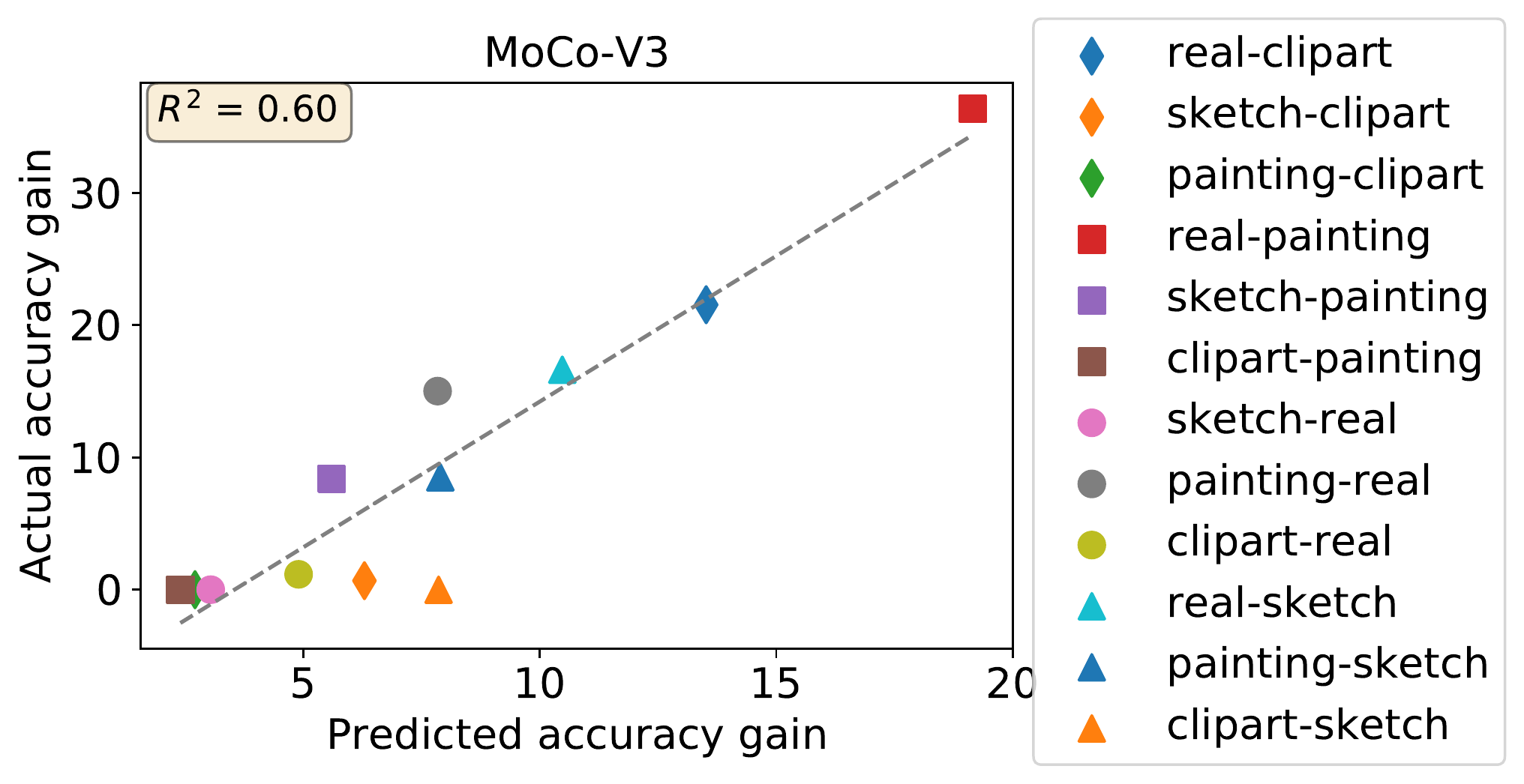}
    \caption{%
    Target accuracies observed vs. accuracies predicted using a function of the connectivity ratios.
    The coefficient of determination ($R^2$) is reported in the upper left corner of each plot.
    The top left panel reports SwAV results (repeated from the left panel of Fig.~\ref{fig:connectivity-governs} for ease of comparison), the top right panel reports MoCo-V2 results, and the bottom panel reports MoCo-V3 results.
    The line of best fit between the observed and predicted is also plotted to help visualize the relationship.}
    \label{fig:trend-plots-app-one}
\end{figure*}
\begin{figure*}[t]
    \centering
    \includegraphics[align=c,height=\arxivstyle{2.8}{2.9}cm]{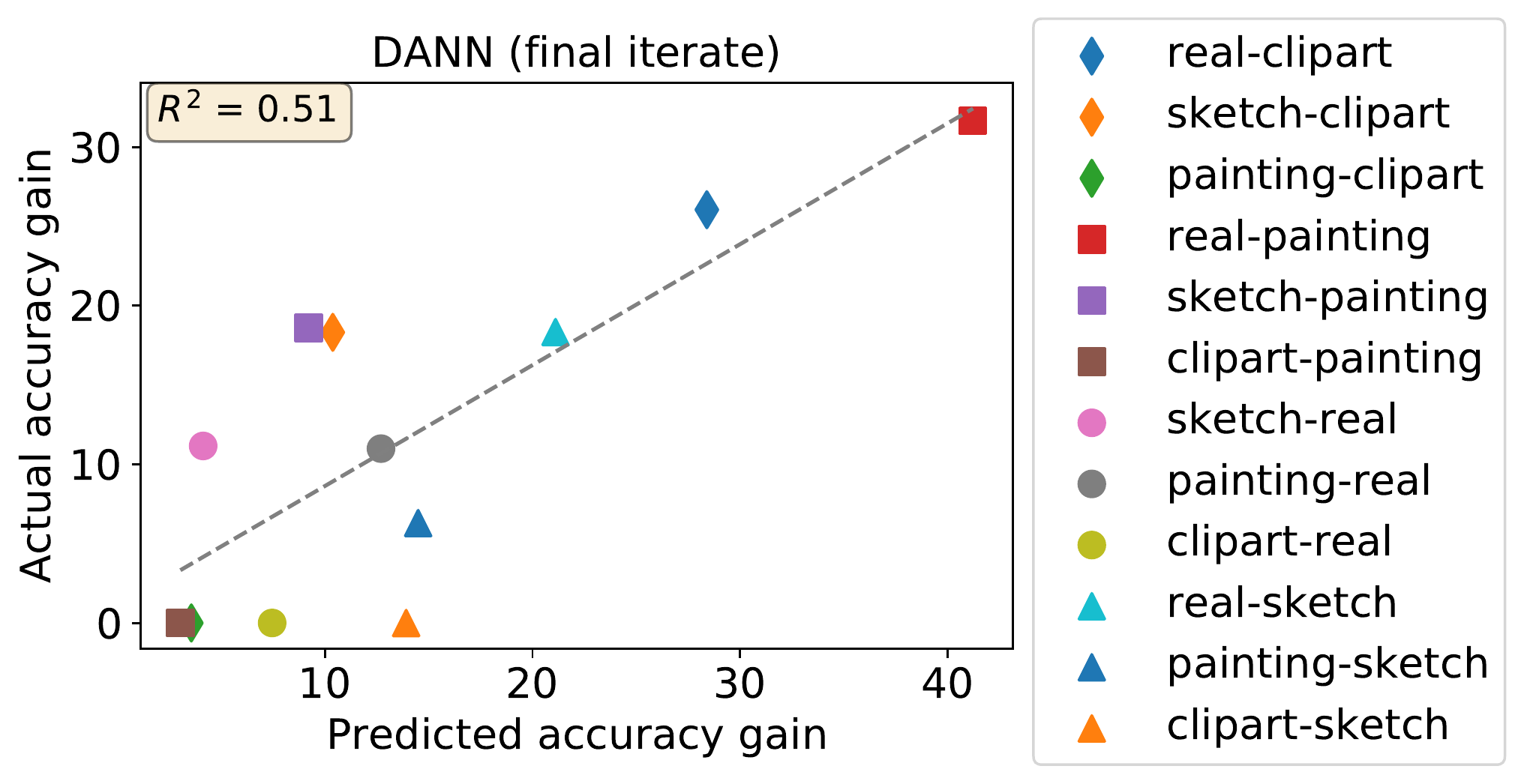}
    \includegraphics[align=c,height=\arxivstyle{2.8}{2.9}cm]{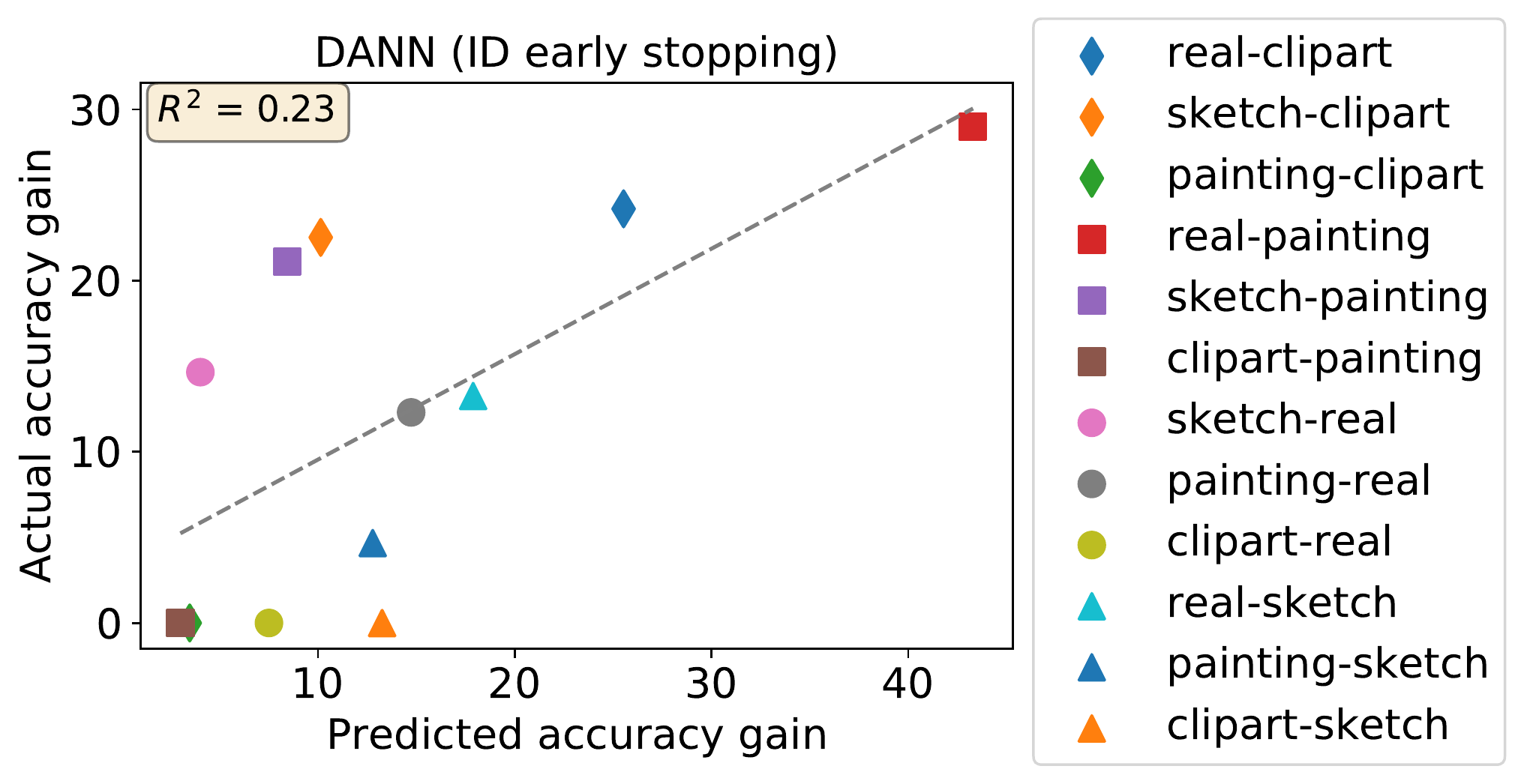}
    \includegraphics[align=c,height=\arxivstyle{2.8}{2.9}cm]{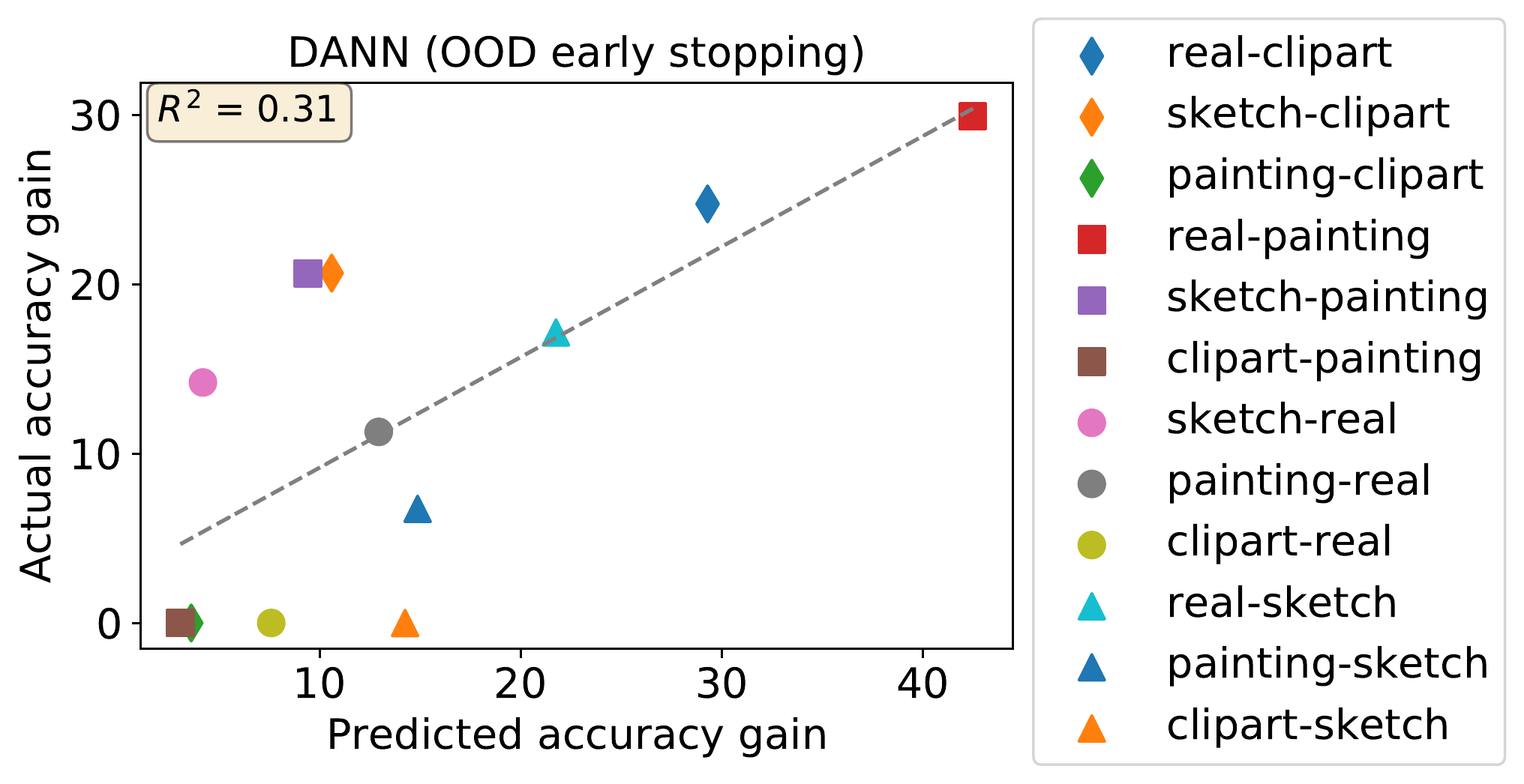} \\
    \includegraphics[align=c,height=\arxivstyle{2.8}{2.9}cm]{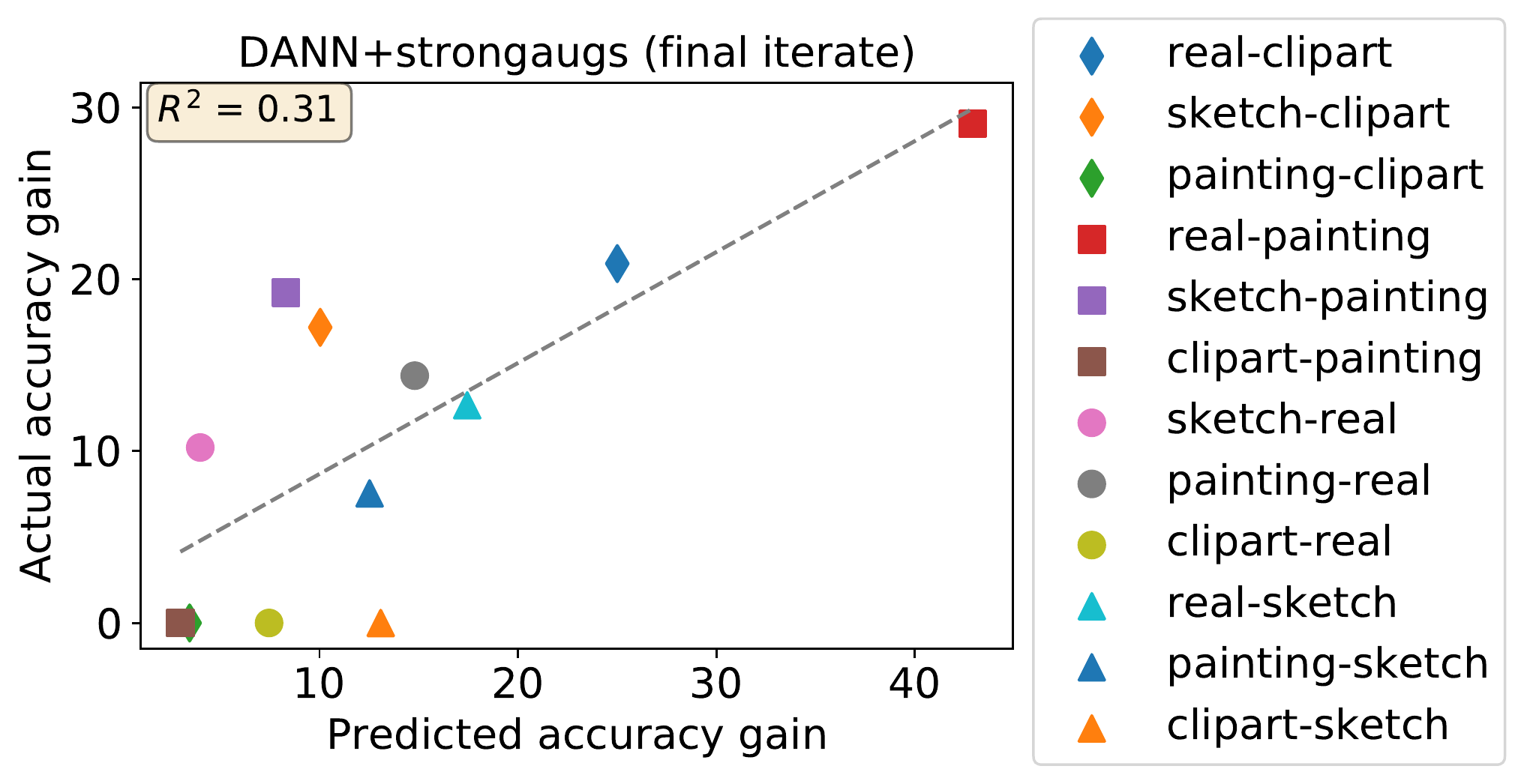}
    \includegraphics[align=c,height=\arxivstyle{2.8}{2.9}cm]{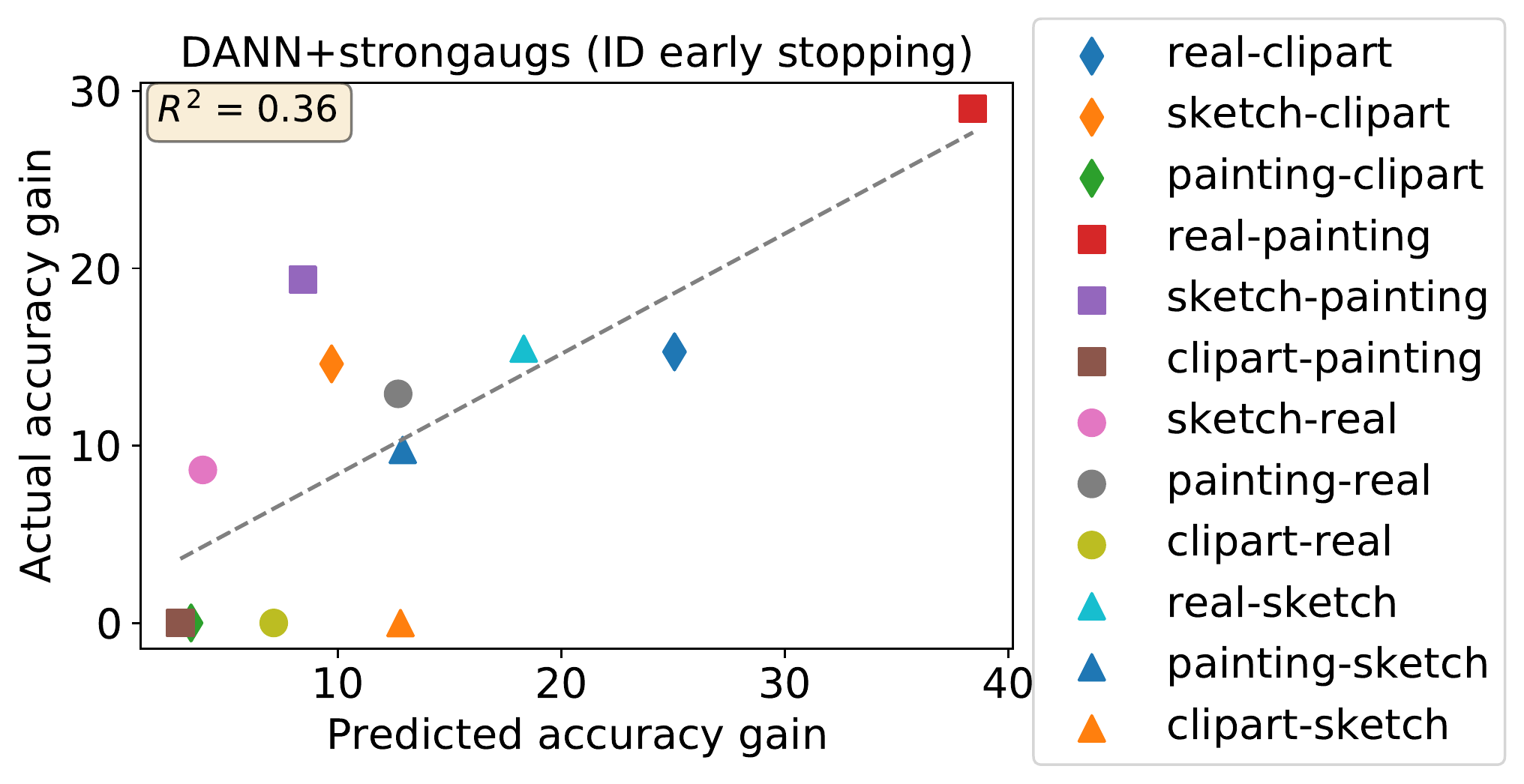}
    \includegraphics[align=c,height=\arxivstyle{2.8}{2.9}cm]{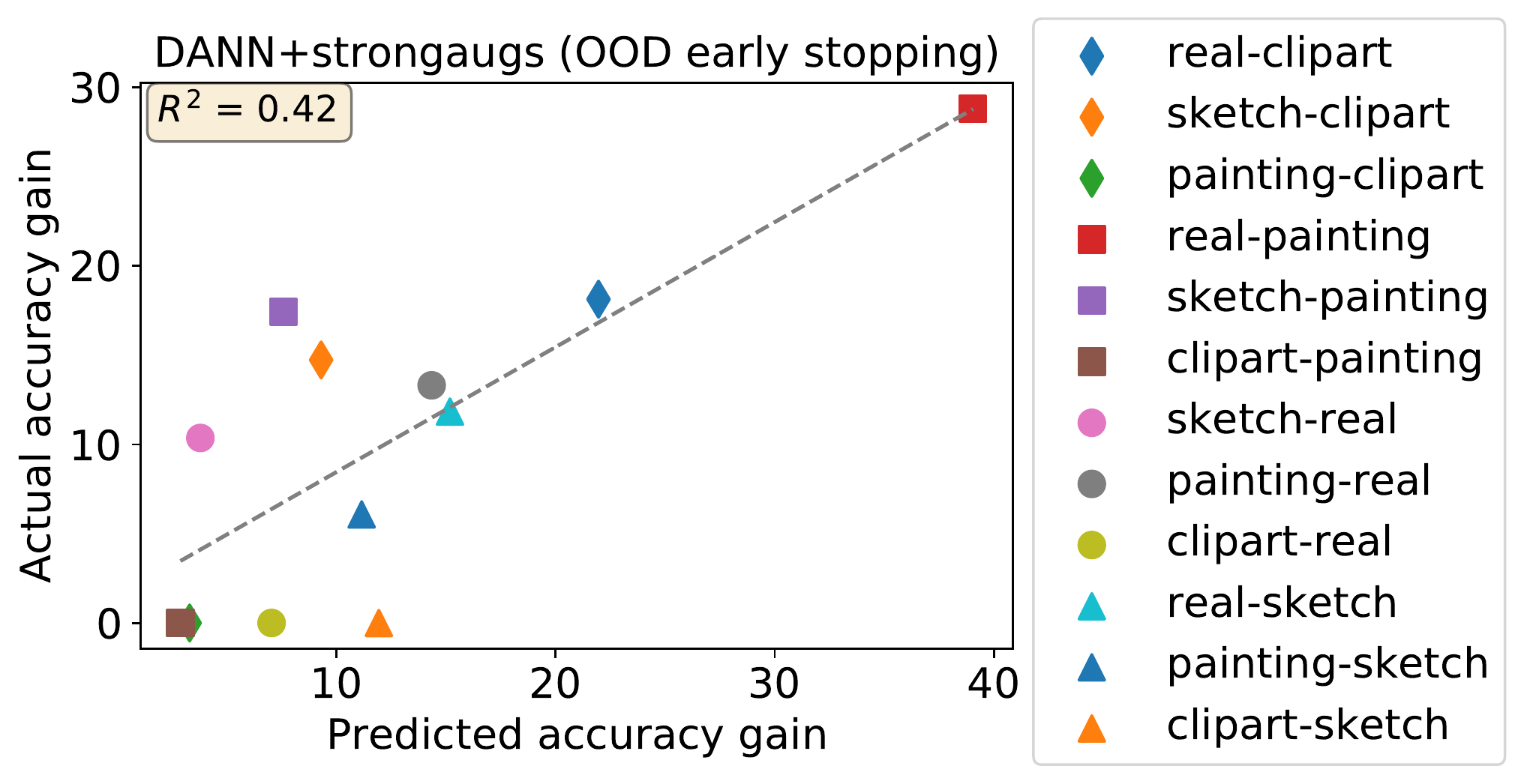} \\
    \includegraphics[align=c,height=\arxivstyle{2.8}{2.9}cm]{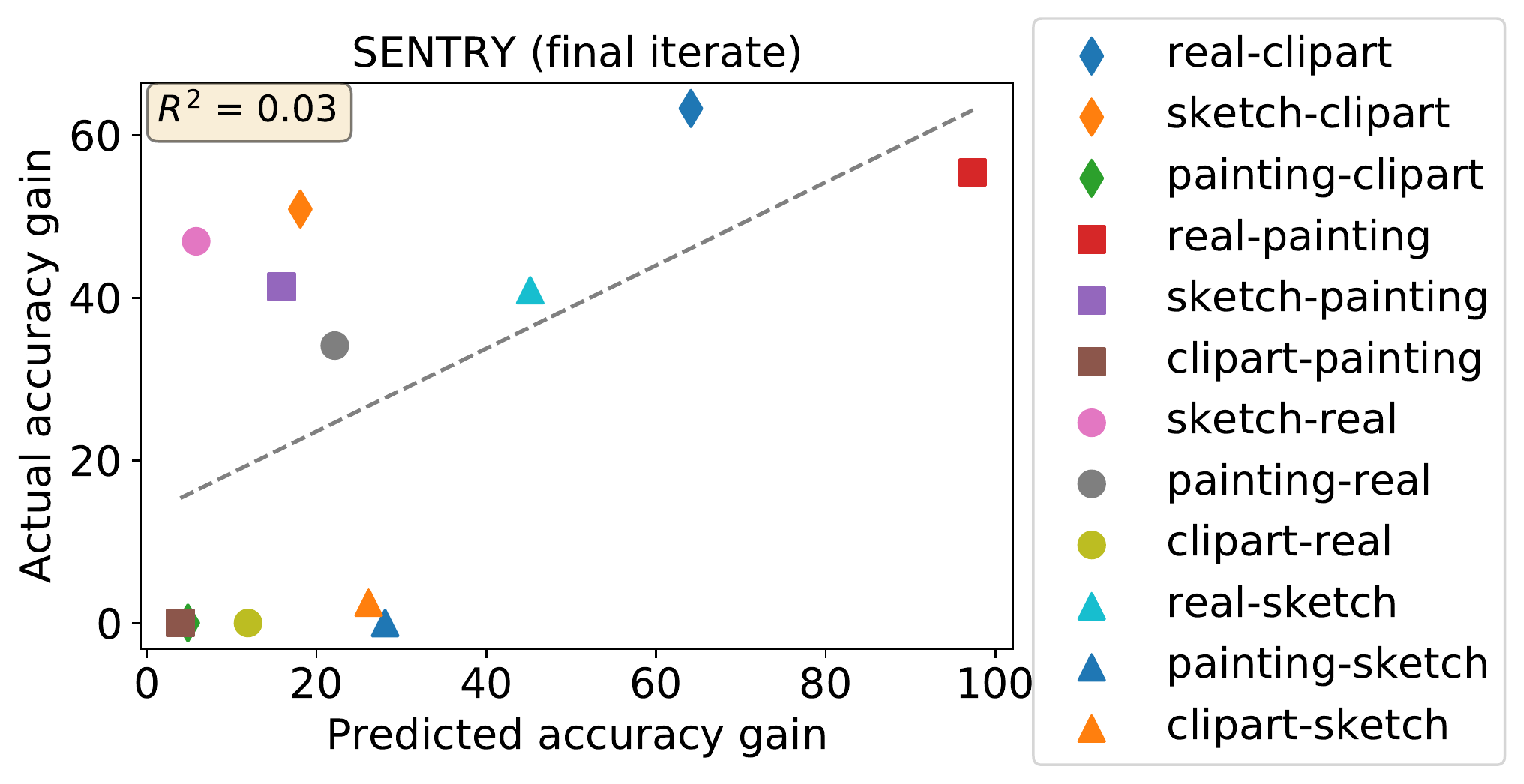}
    \includegraphics[align=c,height=\arxivstyle{2.8}{2.9}cm]{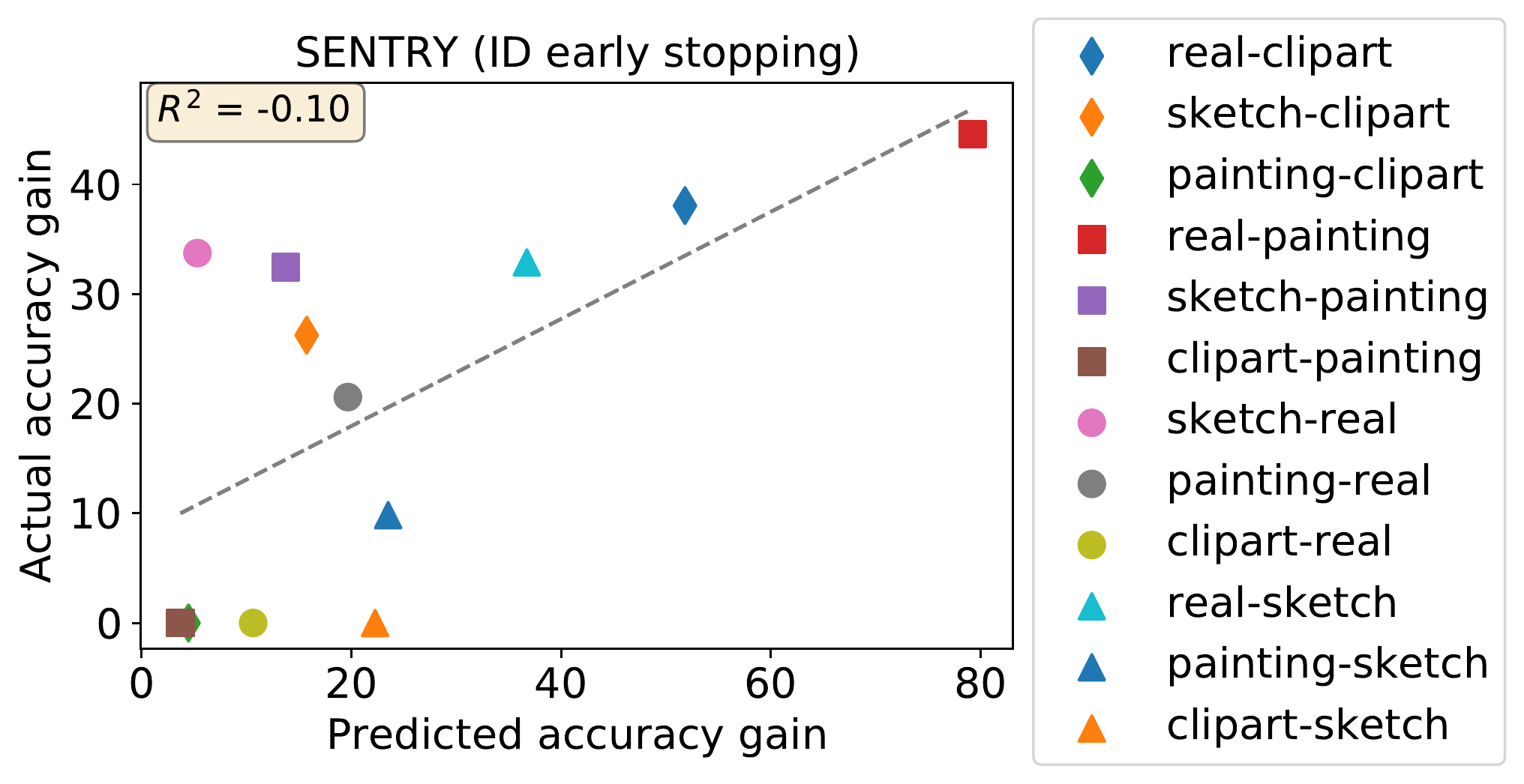}
    \includegraphics[align=c,height=\arxivstyle{2.8}{2.9}cm]{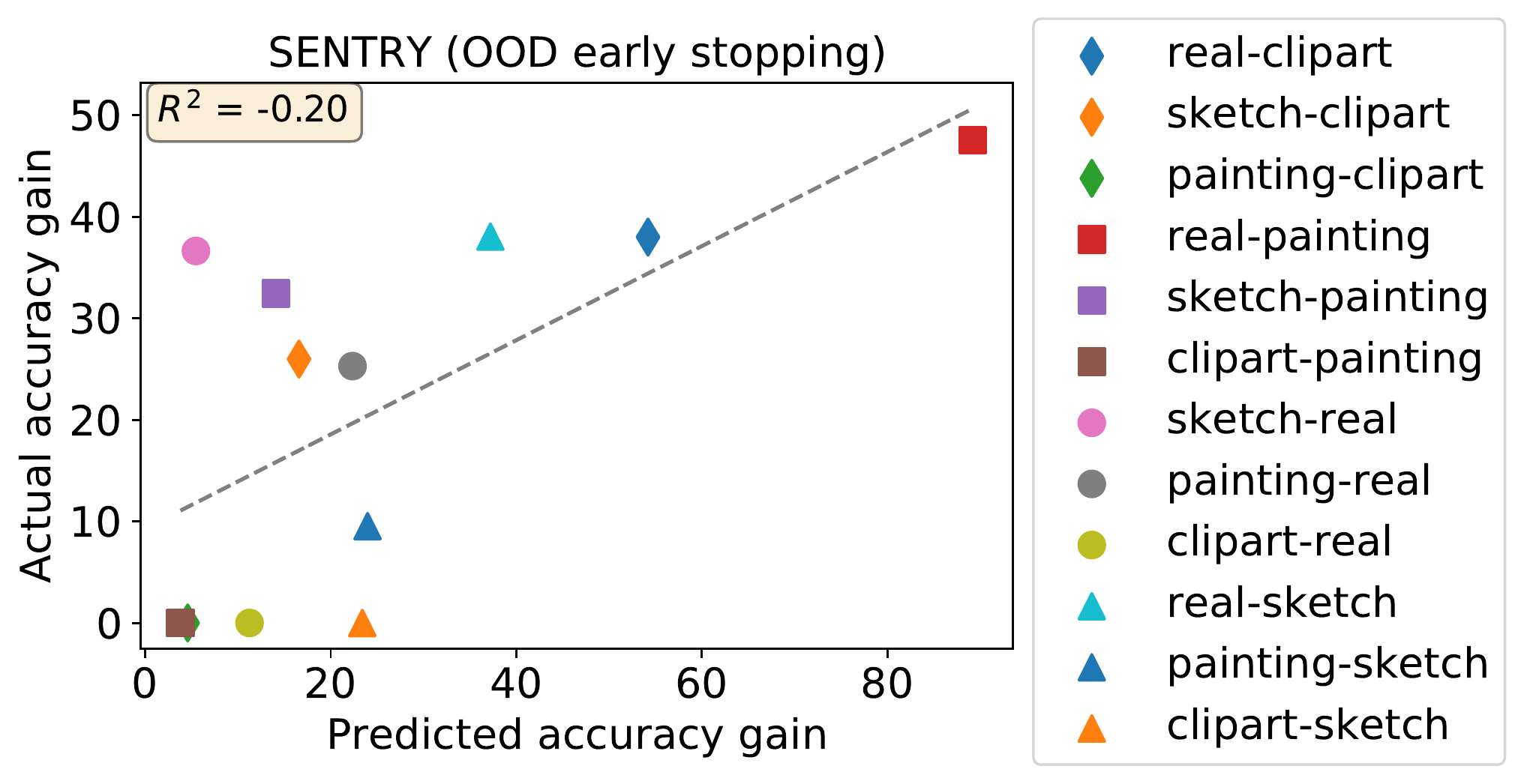} \\
    \includegraphics[align=c,height=\arxivstyle{2.8}{2.9}cm]{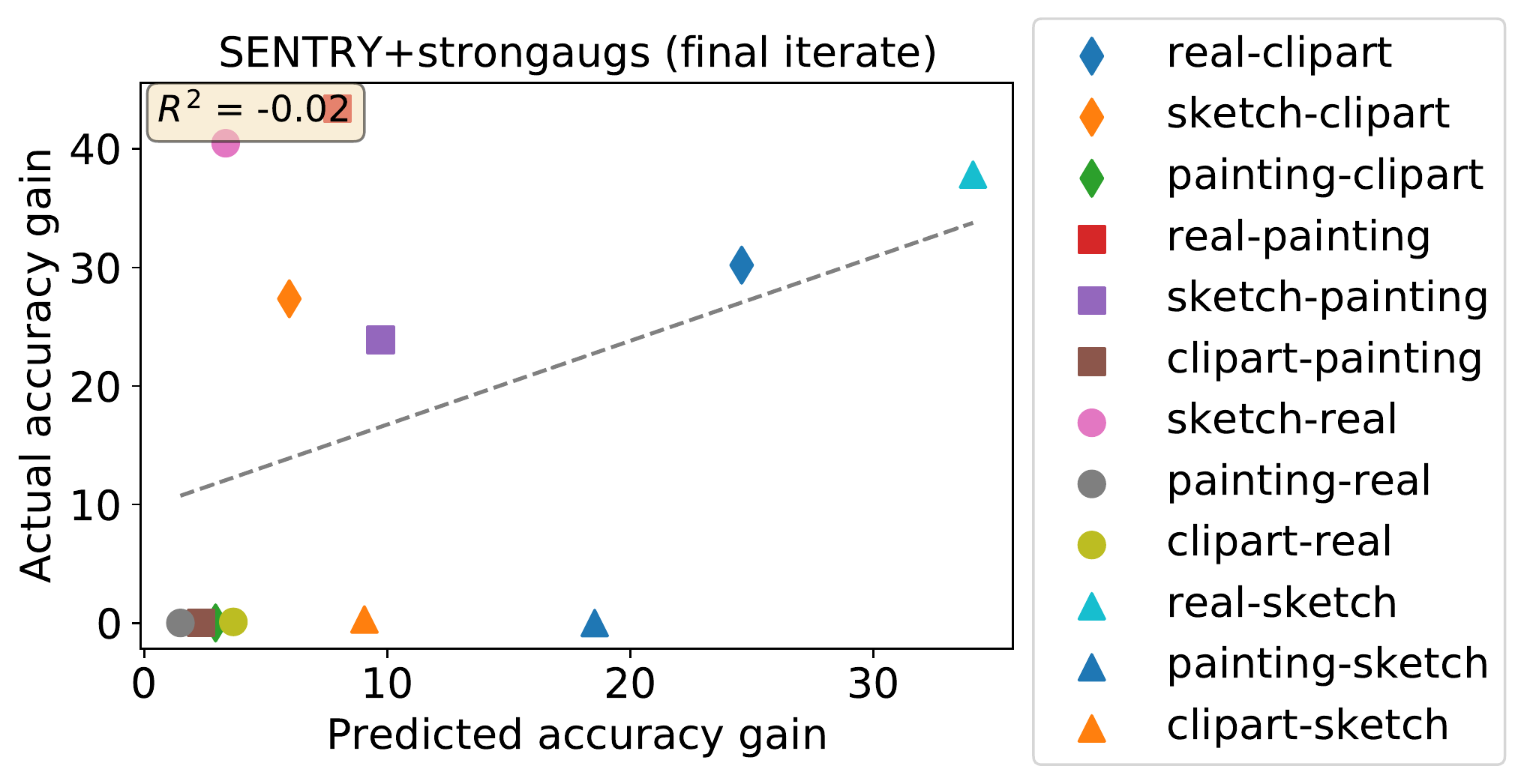}
    \includegraphics[align=c,height=\arxivstyle{2.8}{2.9}cm]{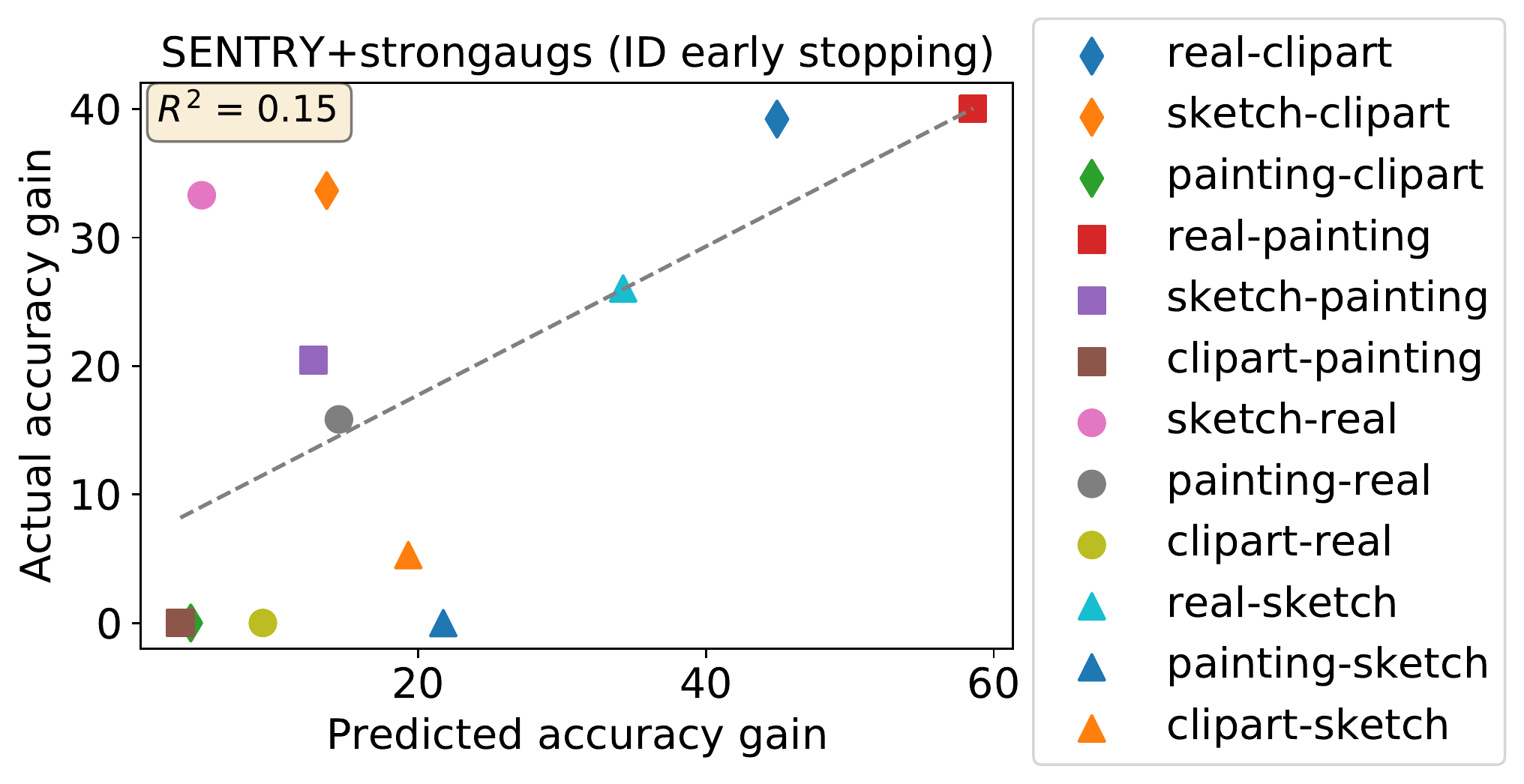}
    \includegraphics[align=c,height=\arxivstyle{2.8}{2.9}cm]{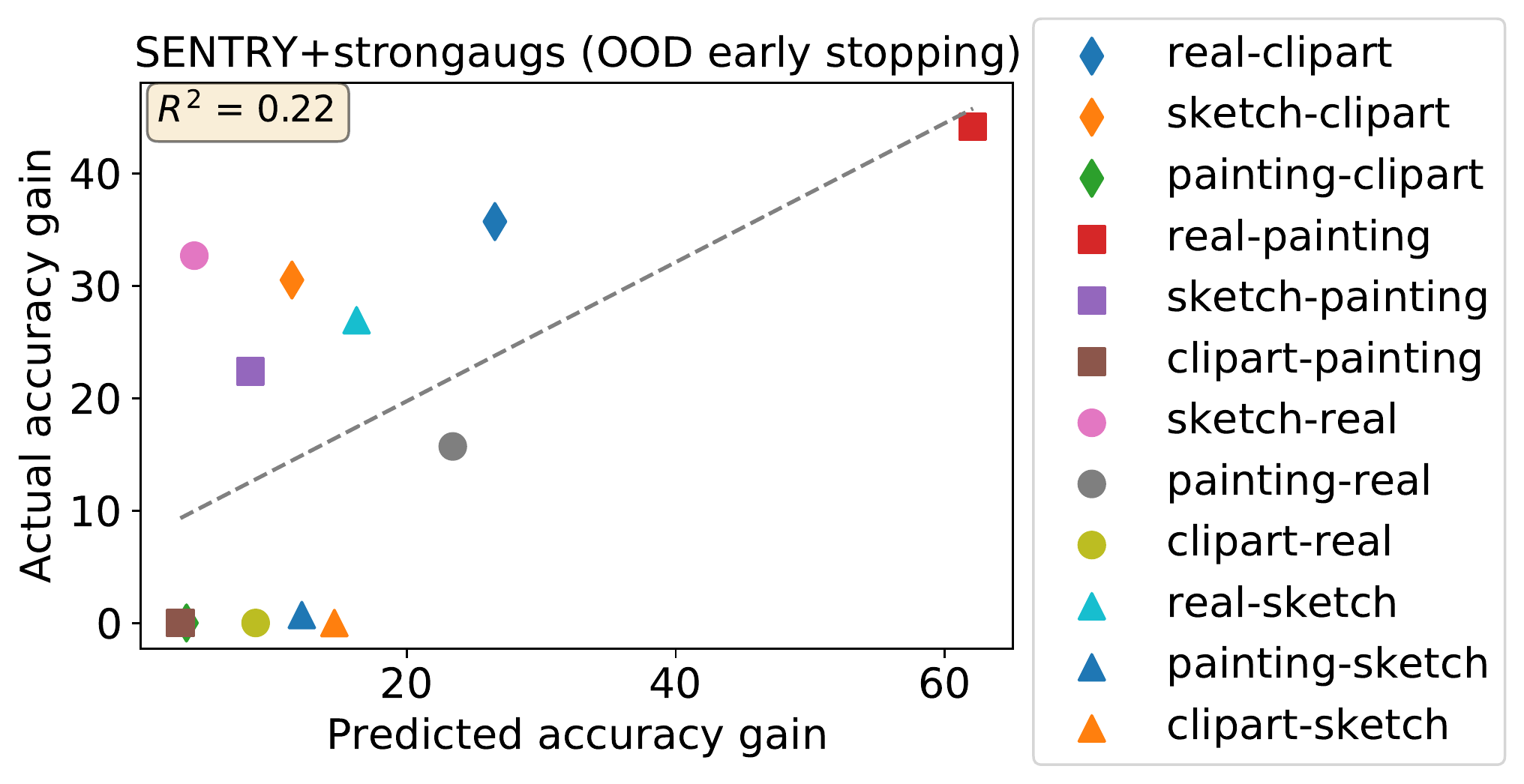}
    \caption{%
    Target accuracies observed vs. accuracies predicted using a function of the connectivity ratios.
    The coefficient of determination ($R^2$) is reported in the upper left corner of each plot.
    The 4 rows contain results of DANN, DANN+\strongaugs, SENTRY, and SENTRY+\strongaugs (in that order).
    The 3 columns contain results from different methods of early stopping: final iterate (i.e., no early stopping), source test, and target test (in that order).
    The connectivity ratio is a somewhat poor predictor of the observed accuracies for DANN and DANN+\strongaugs{} ($R^2 \in [0.23, 0.51]$) and an extremely poor predictor for SENTRY and SENTRY+\strongaugs{} ($R^2 \in [-0.20, 0.22]$).
    Note that because we do not fit an intercept, the computation of the coefficient of determination may produce a negative number.
    The line of best fit between the observed and predicted is also plotted to help visualize the relationship.}
    \label{fig:trend-plots-app-two}
\end{figure*}
\begin{table}[h]
  \centering
\scalebox{0.80}{
\begin{tabular} {l r r r}
    \toprule
    & Across-domain / across-both & Across-class / across-both & Coefficient of \\
    & coefficient ($\wone$) & coefficient ($\wtwo$) & determination ($R^2$) \\
    \midrule
    SwAV & 14.86 & 2.67 & 0.78 \\
    MoCo & 16.07 & 2.65 & 0.79 \\
    MoCo-V3 & 13.51 & 2.79 & 0.60 \\
    \bottomrule
\end{tabular}
}
\caption{%
    Values of $\wone$ and $\wtwo$ fitted to Eq.~\ref{eq:log-connect-ratio}, along with the coefficients of determination ($R^2$) between the observed and predicted target accuracies.
    Accuracies of contrastive pre-training methods (SwAV, MoCo-V2, and MoCo-V3) are well-explained using the connectivity ratios ($R^2$ of 0.78, 0.79, 0.60, respectively).
}
\label{table:app-exponents}
\end{table}

\begin{table}[h]
  \centering
\scalebox{0.80}{
\begin{tabular} {l r r r r r r r r}
\toprule
    & \multicolumn{2}{c}{Source vs. Target} & \multicolumn{2}{c}{Source vs. Domain} & \multicolumn{2}{c}{Target vs. Domain} \\
    \cmidrule(lr){2-3}\cmidrule(lr){4-5}\cmidrule(lr){6-7}
    & Pre-trained & Fine-tuned & Pre-trained & Fine-tuned & Pre-trained & Fine-tuned \\
    \midrule
    Real $\to$ Sketch & 0.200 & 0.286 & 0.016 & 0.011 & 0.019 & 0.020 \\
    Real $\to$ Painting & 0.228 & 0.305 & 0.017 & 0.014 & 0.019 & 0.021 \\
    Real $\to$ Clipart & 0.222 & 0.319 & 0.016 & 0.013 & 0.017 & 0.018 \\
    Sketch $\to$ Real & 0.200 & 0.263 & 0.019 & 0.019 & 0.016 & 0.016 \\
    Sketch $\to$ Painting & 0.176 & 0.254 & 0.015 & 0.016 & 0.021 & 0.021 \\
    Sketch $\to$ Clipart & 0.159 & 0.250 & 0.018 & 0.011 & 0.019 & 0.019 \\
    Painting $\to$ Real & 0.228 & 0.271 & 0.019 & 0.017 & 0.017 & 0.014 \\
    Painting $\to$ Sketch & 0.176 & 0.240 & 0.021 & 0.022 & 0.015 & 0.018 \\
    Painting $\to$ Clipart & 0.138 & 0.233 & 0.019 & 0.016 & 0.024 & 0.020 \\
    Clipart $\to$ Real & 0.222 & 0.277 & 0.017 & 0.018 & 0.016 & 0.010 \\
    Clipart $\to$ Sketch & 0.159 & 0.234 & 0.019 & 0.021 & 0.018 & 0.013 \\
    Clipart $\to$ Painting & 0.138 & 0.237 & 0.024 & 0.022 & 0.019 & 0.016 \\
    \midrule
    Average (DomainNet) & 0.187 & \textbf{0.264} & \textbf{0.018} & 0.017 & \textbf{0.018} & 0.017 \\
    \% change (Fine-tune vs. Pre-train) & & \textbf{+41\%} & & \textbf{-5.5\%} & & \textbf{-5.5\%} \\
    \bottomrule
\end{tabular}
}
    \caption{Cosine similarity of class and domain classifiers trained on SwAV representations (average over all classes) on DomainNet. Class classifiers  trained on the source and target individually learn similar linear weights, evidenced by the average cosine similarity 0.18. However, domain classifiers learn linear weights that are nearly orthogonal to the class classifier weights, suggesting that SwAV pre-training learns features containing domain and class information in somewhat separate directions.}
\label{table:dot-products-app}
\end{table}

\begin{table}[h]
\begin{minipage}{1.0\textwidth}
  \centering
\scalebox{0.72}{
\begin{tabular}{l r r r r}
	\toprule
    \textbf{Input space} & Different class, & Different class, & Same class, & Different class, \\
    & same domain, domain 1 & same domain, domain 2 & different domain & different domain \\
	\midrule
	Real $\leftrightarrow$ Sketch & 24.49 & 38.12 & 21.47 & 20.51 \\
	Real $\leftrightarrow$ Painting & 24.49 & 33.71 & 33.62 & 28.07 \\
	Real $\leftrightarrow$ Clipart & 24.49 & 35.20 & 23.12 & 21.18 \\
	Sketch $\leftrightarrow$ Painting & 38.12 & 33.71 & 24.34 & 23.16 \\
	Sketch $\leftrightarrow$ Clipart & 38.12 & 35.20 & 30.23 & 27.72 \\
    Painting $\leftrightarrow$ Clipart & 33.71 & 35.20 & 31.36 & 30.09 \\
	\midrule
    Avg. & \multicolumn{2}{c}{32.88} & 27.36 & 25.12 \\
	\bottomrule
\end{tabular}
}
\end{minipage}
\vskip.4\baselineskip
\begin{minipage}{1.0\textwidth}
  \centering
\scalebox{0.72}{
\begin{tabular}{l r r r r}
	\toprule
    \textbf{CLIP fine-tuned} & Different class, & Different class, & Same class, & Different class, \\
    \textbf{space} & same domain, domain 1 & same domain, domain 2 & different domain & different domain \\
	\midrule
	Real $\leftrightarrow$ Sketch & 0.72 & 1.85 & 2.71 & 0.52 \\
	Real $\leftrightarrow$ Painting & 0.72 & 1.64 & 5.05 & 0.41 \\
	Real $\leftrightarrow$ Clipart & 0.72 & 1.55 & 6.27 & 0.33 \\
	Sketch $\leftrightarrow$ Painting & 1.85 & 1.64 & 3.58 & 0.54 \\
	Sketch $\leftrightarrow$ Clipart & 1.85 & 1.55 & 5.82 & 0.91 \\
    Painting $\leftrightarrow$ Clipart & 1.64 & 1.55 & 4.05 & 0.51 \\
	\midrule
    Avg. & \multicolumn{2}{c}{1.44} & 4.58 & 0.54 \\
	\bottomrule
\end{tabular}
}
\end{minipage}
\vskip.4\baselineskip
\begin{minipage}{1.0\textwidth}
  \centering
\scalebox{0.72}{
\begin{tabular}{l r r r r}
	\toprule
    \textbf{SwAV pre-trained} & Different class, & Different class, & Same class, & Different class, \\
    \textbf{space} & same domain, domain 1 & same domain, domain 2 & different domain & different domain \\
	\midrule
	Real $\leftrightarrow$ Sketch & 3.08 & 6.92 & 4.73 & 1.74 \\
	Real $\leftrightarrow$ Painting & 3.07 & 7.43 & 11.79 & 2.23 \\
	Real $\leftrightarrow$ Clipart & 3.02 & 6.34 & 8.93 & 1.73 \\
	Sketch $\leftrightarrow$ Painting & 8.50 & 10.65 & 5.55 & 2.66 \\
	Sketch $\leftrightarrow$ Clipart & 8.54 & 8.15 & 8.10 & 3.29 \\
	Painting $\leftrightarrow$ Clipart & 10.22 & 8.40 & 6.13 & 3.12 \\
	\midrule
    Avg. & \multicolumn{2}{c}{7.03} & 7.54 & 2.46 \\
	\bottomrule
\end{tabular}
}
\end{minipage}
\vskip.4\baselineskip
\begin{minipage}{1.0\textwidth}
  \centering
\scalebox{0.75}{
\begin{tabular}{l r r r r}
	\toprule
    \textbf{DANN+\strongaugs{}} & Different class, & Different class, & Same class, & Different class, \\
    \textbf{feature space} & same domain, source & same domain, target & different domain & different domain \\
	\midrule
    Real $\rightarrow$ Sketch & 2.72 & 5.70 & 7.72 & 2.28 \\
    Real $\rightarrow$ Painting & 2.54 & 7.44 & 16.85 & 3.36 \\
    Real $\rightarrow$ Clipart & 2.91 & 5.45 & 10.45 & 1.92 \\
    Sketch $\rightarrow$ Real & 4.54 & 4.72 & 13.32 & 4.30 \\
    Sketch $\rightarrow$ Painting & 4.69 & 9.79 & 14.27 & 5.12 \\
    Sketch $\rightarrow$ Clipart & 4.85 & 5.49 & 17.44 & 4.64 \\
    Painting $\rightarrow$ Real & 7.29 & 3.76 & 20.71 & 4.87 \\
    Painting $\rightarrow$ Sketch & 7.43 & 6.54 & 12.45 & 4.57 \\
    Painting $\rightarrow$ Clipart & 7.90 & 5.90 & 10.76 & 3.63 \\
    Clipart $\rightarrow$ Real & 4.42 & 5.27 & 14.38 & 3.13 \\
    Clipart $\rightarrow$ Sketch & 6.21 & 3.93 & 13.84 & 4.50 \\
    Clipart $\rightarrow$ Painting & 4.42 & 11.69 & 11.52 & 4.32 \\
    \midrule
    Avg. & 4.99 & 6.31 & 13.64 & 3.89 \\
	\bottomrule
\end{tabular}
}
\end{minipage}
\caption{%
    Empirical estimates of the different parameters of connectivity in the input space (top) and feature spaces computed by CLIP (second from top), SwAV (third from top), and DANN+\strongaugs{} (bottom).
    The numbers provided are the error of classifying different types of class-domain pairs.}
\label{table:fullconnectivity}
\end{table}

\end{document}